\documentclass{article}

\usepackage{microtype}
\usepackage{graphicx}
\usepackage{subfigure}
\usepackage{booktabs} 

\usepackage{hyperref}

\usepackage[accepted]{icml2023}

\usepackage{amsmath}
\usepackage{amssymb}
\usepackage{mathtools}
\usepackage{amsthm}
\usepackage{thm-restate}
\usepackage{enumitem}

\usepackage[capitalize,noabbrev]{cleveref}

\theoremstyle{plain}
\newtheorem{theorem}{Theorem}[section]

\newtheorem{lemma}[theorem]{Lemma}

\newtheorem{claim}{Claim}[section]
\theoremstyle{definition}
\newtheorem{definition}[theorem]{Definition}
\newtheorem{assumption}[theorem]{Assumption}
\theoremstyle{remark}

\usepackage{amsmath,amsfonts,bm}

\newcommand{\loss}{\mathcal{L}_S}
\newcommand{\lossmix}{\mathcal{L}_S^\mathrm{mix}}
\newcommand{\lossmask}{\mathcal{L}_S^\mathrm{mask}}
\newcommand{\w}[1][\kappa]{\vw^*(#1)}
\newcommand{\wmix}[1][\kappa]{\vw^*_\mathrm{mix}(n,#1)}
\newcommand{\wmask}[1][\kappa]{\vw^*_\mathrm{mask}(n,#1)}
\newcommand{\cosim}{\mathrm{sim}}
\newcommand{\support}{\mathrm{support}}
\newcommand{\veta}{{\bm{\eta}}}
\newcommand{\vDelta}{{\bm{\Delta}}}

\def\eqref#1{equation~\ref{#1}}
\def\Eqref#1{Equation~\ref{#1}}

\def\1{\bm{1}}

\def\vzero{{\bm{0}}}
\def\vone{{\bm{1}}}
\def\vmu{{\bm{\mu}}}
\def\vtheta{{\bm{\theta}}}

\def\ve{{\bm{e}}}

\def\vu{{\bm{u}}}
\def\vv{{\bm{v}}}
\def\vw{{\bm{w}}}
\def\vx{{\bm{x}}}

\def\vz{{\bm{z}}}

\def\mI{{\bm{I}}}

\def\mM{{\bm{M}}}

\def\mSigma{{\bm{\Sigma}}}

\DeclareMathAlphabet{\mathsfit}{\encodingdefault}{\sfdefault}{m}{sl}
\SetMathAlphabet{\mathsfit}{bold}{\encodingdefault}{\sfdefault}{bx}{n}

\def\rz{{\textnormal{z}}}

\def\rvx{{\mathbf{x}}}

\def\rvz{{\mathbf{z}}}

\newcommand{\R}{\mathbb{R}}

\DeclareMathOperator*{\argmin}{arg\,min}

\DeclareMathOperator{\Tr}{Tr}

\allowdisplaybreaks
\usepackage[textsize=tiny]{todonotes}

\icmltitlerunning{Provable Benefit of Mixup for Finding Optimal Decision Boundaries}

\begin{document}

\twocolumn[
\icmltitle{Provable Benefit of Mixup for Finding Optimal Decision Boundaries}



\icmlsetsymbol{equal}{*}

\begin{icmlauthorlist}
\icmlauthor{Junsoo Oh}{KAISTAI}
\icmlauthor{Chulhee Yun}{KAISTAI}
\end{icmlauthorlist}

\icmlaffiliation{KAISTAI}{Kim Jaechul Graduate School of AI, KAIST}

\icmlcorrespondingauthor{Chulhee Yun}{chulhee.yun@kaist.ac.kr}

\icmlkeywords{Mixup, optimal decision boundary, sample complexity, theory, generalization}

\vskip 0.3in
]



\printAffiliationsAndNotice{}  

\begin{abstract}
We investigate how pair-wise data augmentation techniques like Mixup affect the sample complexity of finding optimal decision boundaries in a binary linear classification problem. For a family of data distributions with a separability constant~$\kappa$, we analyze how well the optimal classifier in terms of training loss aligns with the optimal one in test accuracy (i.e., Bayes optimal classifier). For vanilla training without augmentation, we uncover an interesting phenomenon named the \emph{curse of separability}. As we increase $\kappa$ to make the data distribution more separable, the sample complexity of vanilla training \emph{increases exponentially} in $\kappa$; perhaps surprisingly, the task of finding optimal decision boundaries becomes harder for more separable distributions. For Mixup training, we show that Mixup mitigates this problem by \emph{significantly reducing} the sample complexity. To this end, we develop new concentration results applicable to $n^2$ pair-wise augmented data points constructed from $n$ independent data, by carefully dealing with dependencies between overlapping pairs. Lastly, we study other masking-based Mixup-style techniques and show that they can \emph{distort} the training loss and make its minimizer converge to a suboptimal classifier in terms of test accuracy.
\end{abstract}
\section{Introduction}
Mixup~\citep{zhang2018mixup} is a modern technique that augments training data with a random convex combination of a pair of training points and labels. \citet{zhang2018mixup} empirically show that this simple technique has various benefits, such as better generalization, robustness to label corruption and adversarial attack, and stabilization of generative adversarial network training. Inspired by the success of Mixup, several variants of Mixup have appeared in the literature; e.g., Manifold Mixup~\citep{verma2019manifold}, Cutmix~\citep{yun2019cutmix}, Puzzle Mix~\citep{kim2020puzzle}, and Co-Mixup~\citep{kim2021co}. The success of Mixup-style training schemes is not only limited to improved generalization performance in supervised learning; they are known to be helpful in other aspects including model calibration~\citep{thulasidasan2019mixup}, semi-supervised learning~\citep{berthelot2019mixmatch, sohn2020fixmatch}, contrastive learning~\citep{kalantidis2020hard, verma2021towards}, and natural language processing~\citep{guo2019augmenting, sun2020mixup}.

Although Mixup and its variants demonstrate surprising empirical benefits, a concrete theoretical understanding of such benefits still remains mysterious. As a result, a recent line of research~\citep{carratino2020mixup, zhang2020does, zhang2022and, chidambaram2021towards, chidambaram2022provably, parkunified} trying to theoretically understand Mixup and its variants has appeared in the literature. Most results, including this paper, compare the training procedure with Mixup against solving the vanilla empirical risk minimization (ERM) problem without any augmentation. For the remainder of this paper, we refer to the training loss without augmentation as \emph{ERM loss}, and the training loss with Mixup as \emph{Mixup loss}.

The parallel works of \citet{carratino2020mixup} and \citet{zhang2020does} represent the Mixup loss as the ERM loss equipped with additive data-dependent regularizers which penalize the gradient and Hessian of a model with respect to data. Using these regularizers, \citet{zhang2020does} show that Mixup training can yield smaller Rademacher complexity which allows for a smaller uniform generalization bound. Also, \citet{parkunified} extend these results to several Mixup variants such as Cutmix~\cite{yun2019cutmix}. However, the benefits of Mixup-style schemes from the Rademacher complexity view shown by \citet{zhang2020does} and \citet{parkunified} stand out only if the intrinsic dimension of data is small.

\citet{chidambaram2021towards} try to understand Mixup by investigating when Mixup training minimizes the ERM loss. The authors show failure cases of Mixup and also sufficient conditions for its success. \citet{chidambaram2021towards} also show that ERM loss and Mixup loss can have the same optimal classifier by considering linear model and special types of data distribution. Detailed comparisons to our work are provided in Section~\ref{section:setting}.

Another recent work~\citep{chidambaram2022provably} shows that Midpoint Mixup~\citep{guo2021midpoint} can outperform ERM training in terms of feature learning performance in a multi-view data framework proposed by \citet{allen2020towards}. However, their analysis is limited to an extreme case of Mixup  using only the midpoint of two data points. 

\subsection{Our Contributions}

In this paper, we study the problem of finding optimal decision boundaries in a binary linear classification problem with logistic loss (i.e., logistic regression). We consider a data generating distribution whose positive and negative examples come from two symmetric Gaussian distributions. The two distributions have the same covariance matrices inversely proportional to a \emph{separability constant} $\kappa$, which controls the degree to which the two classes are separable.

Our data generating distribution has the advantage that the \emph{Bayes optimal classifier} (i.e., the classifier with the best test accuracy) is given as a closed-form linear classifier. This motivates us to study how well the optimal classifier in terms of training loss is aligned with the Bayes optimal classifier, and how many data points are required to make sure that the two classifiers are close enough. Therefore, we focus on the sample complexity of vanilla ERM training and Mixup-style training for achieving close-to-one cosine similarity between the two classifiers.

Our results demonstrate that Mixup provably requires a much smaller number of samples to achieve the same cosine similarity compared to vanilla ERM training. Our contributions can be summarized as the following:
\begin{itemize}

    \item In Section~\ref{section:ERM}, we investigate the sample complexity of (vanilla) ERM under our data distribution. 
    Theorem~\ref{thm:expected_loss_ERM} shows that the expected value of ERM training loss has a unique optimal solution aligned with the Bayes optimal classifier. 
    We then prove that the sample complexity for making the ERM loss optimum closely aligned with the Bayes optimal classifier grows \emph{exponentially} with the separability constant $\kappa$; we show that the exponential growth is both sufficient (Theorem~\ref{thm:ERM_sufficient}) and necessary (Theorem~\ref{thm:ERM_necessary}). Interestingly, these results demonstrate that \emph{ERM suffers the curse of separability}, where the sample complexity increases as the data distribution becomes more separable.
    
    \item In Section~\ref{section:mix}, we study a unified class of Mixup-style training scheme that includes Mixup~\citep{zhang2018mixup} with various choices of hyperparameters. Theorem~\ref{thm:expected_loss_mix} shows that the expected value of Mixup loss has a unique optimal solution still aligned with the Bayes optimal classifier, which indicates that Mixup-style augmentations do not ``distort'' the training loss in a misleading way, at least in our setting. 
    In Theorem~\ref{thm:mixup_convergence}, we show that the sample complexity for getting a near-Bayes-optimal Mixup loss solution grows only quadratically in $\kappa$. This result indicates that \emph{Mixup provably mitigates the curse of separability}.
    
    \item  In Section~\ref{section:mask}, we analyze how recent masking-based variants of Mixup such as CutMix~\citep{yun2019cutmix} behave in our setting. Unfortunately, Theorem~\ref{thm:mask_minimizer} shows that masking-based augmentation can \emph{distort} the training loss and the expected value of masking-based Mixup loss can have optimal solutions far away from the Bayes optimal classifier. Ironically, Theorem~\ref{thm:mask_convergence} indicates the sample complexity for approaching such a ``wrong'' minimizer does not grow with separability. These results show that, at least in our setting, masking-based techniques have small sample complexity but may not converge to the Bayes optimal classifier.
\end{itemize}
\vspace*{-10pt}
\section{Problem Setting and Notation}\label{section:setting}
In this section, we introduce our formal problem setting and notation. 
We consider a binary linear classification problem with training dataset $S = \{(\vx_i,y_i)\}_{i=1}^n$ where $\vx_i\in \mathbb{R}^d$ are data points and $y_i\in \{0,1\}$ are labels.
The vanilla empirical risk minimization (ERM) loss under training set $S$ can be formulated as follows:
\begin{equation}
\label{eq:1}
   \loss(\vw) :=  \frac{1}{n} \sum_{i=1}^n y_i l(\vw^\top \vx_i) + (1-y_i) l(-\vw^\top \vx_i), 
\end{equation}
where $l(\cdot)$ is the logistic loss $l(z) := \log(1+\exp(-z))$.

\paragraph{Data Generating Distribution.} We mainly focus on the following family of data generating distributions, where we obtain positive and negative examples from two symmetric Gaussian distributions. More precisely, for a given constant $\kappa\in (0, \infty)$, we define a data generating distribution $\mathcal{D}_\kappa$ as the following: we say that $(\vx,y) \sim \mathcal{D}_\kappa$ when $y$ is uniformly drawn from $\{0, 1\}$ and
\begin{align*}
    \vx \mid {y=1} &\sim N( \vmu , \kappa^{-1} \mSigma ), \\
    \vx \mid {y=0} &\sim N(-\vmu , \kappa^{-1}\mSigma ).
\end{align*}
where $\vmu \in \mathbb{R}^d$ is nonzero, $\mSigma \in \mathbb{R}^{d \times d}$ is positive definite, and $\lVert \vmu \rVert^2 = \lVert \mSigma \rVert$. Here, $\lVert \cdot \rVert$ denotes the $\ell_2$ norm for a vector and the spectral norm for a matrix. We further define a data generating distribution $\mathcal{D}_\infty$ as the limiting behavior of $\mathcal{D}_\kappa$ as $\kappa \to \infty$; i.e., $(\vx,y) \sim \mathcal{D}_\infty$ implies $y$ follows uniform distribution on $\{0,1\}$ and $\vx = (2y-1)\vmu$. 

\paragraph{Separability Constant $\bm{\kappa}$.}
Our data generating distribution $\mathcal{D}_\kappa$ with larger $\kappa$ is more \emph{separable} in the sense that the two Gaussian distributions overlap less and they are more likely to generate well-separated training data. This separability constant $\kappa$ is of great importance in our analysis, as we will demonstrate the curse of separability phenomenon based on the dependency of sample complexity on $\kappa$.

\paragraph{Bayes Optimal Classifier.}
The Bayes optimal classifier is a classifier achieving the lowest possible test error on the data population. In other words, any other classifier cannot outperform the Bayes optimal classifier in terms of test accuracy. 
For the specific form of data generating distribution $\mathcal{D}_\kappa$ that we consider in this paper, it is well-known that for any $\kappa \in (0,\infty)$, decision boundary of the Bayes optimal classifier for $\mathcal{D}_\kappa$ is a hyperplane with normal vector $\mSigma^{-1} \vmu$. For completeness, we provide proof for this in Appendix~\ref{proof:Bayes_optimal}. The fact that we have a closed-form solution of the optimal decision boundary motivates us to study finding a decision boundary close enough to that optimal one. Since we are mainly interested in linear models, for our ``closeness'' metric, we use cosine similarity between the two linear decision boundaries, or equivalently, the two normal vectors.

\paragraph{Comparison to the Setting of \citet{chidambaram2021towards}.}
\citet{chidambaram2021towards} also consider the linear model and a special type of data generating distributions to compare training ERM loss vs.\ Mixup loss. Our setting is different from theirs in several aspects. First, there is a difference in data generating distributions; we draw positive and negative data from two symmetric Gaussian distributions while \citet{chidambaram2021towards} consider the distribution that obtains data from a single spherical Gaussian distribution. \citet{chidambaram2021towards} show that training ERM loss and Mixup loss can lead to the same solution by considering the highly overparametrized regime. On the other hand, we consider the underparametrized regime and show the gap between the two training methods. In addition, \citet{chidambaram2021towards} only consider the optimal classifier in terms of training loss while we investigate the number of required training data to get close to the optimal solution of population loss.

\paragraph{Notation.}
We denote taking expectation on training data $S = \{(\vx_i, y_i)\}_{i=1}^n \overset{\mathrm{i.i.d}}{\sim} \mathcal{D}_\kappa$ and any other randomness, if any, 
as $\mathbb{E}_\kappa$ for each $\kappa \in (0, \infty]$. We use $[k]$ for the index set $\{1, 2, \dots, k\}$ for each $k \in \mathbb{N}$. For two nonzero vectors $\vu,\vv \in \mathbb{R}^d$, let us denote their cosine similarity as $\cosim(\vu,\vv) = \frac{\vu^\top \vv}{\lVert \vu\rVert\lVert \vv\rVert}$ and the angle between them as $\angle(\vu,\vv) = \cos^{-1}(\cosim(\vu,\vv))$. We use $\support(\cdot)$ to denote the support of a probability distribution or a random variable and $\vone_{(\cdot)}$ denotes the $0$-$1$ indicator function. We also use $\odot$ to denote element-wise multiplication between two vectors or two matrices having the same size. We use $\mathcal{O}(\cdot), \Theta(\cdot), \Omega(\cdot)$ to represent asymptotic behavior as $\kappa$ grows and to hide terms related to $d, \vmu, \mSigma$. 
In addition, whenever we express the $\kappa$ dependency in $\mathcal{O}(\cdot), \Theta(\cdot), \Omega(\cdot)$ notation, we only write the most dominant factor; for example, we say $\kappa^m\mathrm{polylog}(\kappa) = \Theta(\kappa^m)$ and $\exp\left(c \kappa^m\right) \mathrm{poly}(\kappa) = \exp(\Theta(\kappa^m))$.
\section{ERM Suffers the Curse of Separability}\label{section:ERM}
In this section, we investigate a solution of ERM loss in \Eqref{eq:1}.
The ERM loss function $\loss(\vw)$ is a stochastic function that depends on the random samples in the training set $S = \{(\vx_i, y_i)\}_{i=1}^n \overset{\mathrm{i.i.d}}{\sim} \mathcal{D}_\kappa$. 
We will first characterize the minimizer of the expected value $\mathbb{E}_\kappa[\loss(\vw)]$ of the ERM loss\footnote{The expected ERM loss is equal to the population loss, which is independent of $n$. However, for Mixup loss, its expected value becomes \emph{dependent} on $n$, as we will see later.} and show that the unique optimum of the expected ERM loss has the same direction as the Bayes optimal classifier. Next, we study the sample complexity for the ERM loss optimum to align closely enough to the Bayes optimal classifier and conclude that ERM without Mixup suffers the \emph{curse of separability}.

Our first theorem below analyzes the optimum of expected ERM loss $\mathbb{E}_\kappa[\loss(\vw)]$.
\begin{theorem}\label{thm:expected_loss_ERM}
For any $\kappa \in (0,\infty)$, the expectation of ERM loss $\mathbb{E}_\kappa[\loss(\vw)]$ has a unique minimizer $\w$. In addition, its direction is the same as the Bayes optimal solution $\mSigma^{-1} \vmu$.
\end{theorem}
\vspace*{-10pt}
\begin{proof}[Proof Sketch]
We will only sketch main proof ideas and full proof can be found in Appendix~\ref{proof:expected_loss_ERM}. We can rewrite the expected ERM loss as
\begin{equation*}
\mathbb{E}_\kappa [\loss (\vw)] = \mathbb{E}\left[ l\left( \kappa^{-1/2} \left(\vw^\top \mSigma \vw\right)^{1/2} Z + \vw^\top \vmu \right)\right],
\end{equation*}
where $Z \sim N(0,1)$.
The following Lemma~\ref{lemma:logistic} implies that for a fixed value of $\vw^\top\vmu$, a smaller value of $\vw^\top \mSigma \vw$ induces a smaller $\mathbb{E}_\kappa[\loss(\vw)]$.
\begin{restatable}{lemma}{logistic}\label{lemma:logistic}
Let $X_1 \sim N(m, \sigma_1^2), X_2 \sim N(m, \sigma_2^2)$, $m \in \mathbb{R}$ and $\sigma_1 > \sigma_2>0$. Then, $\mathbb{E}[l(X_1)] > \mathbb{E}[l(X_2)]$.
\end{restatable}
\vspace*{-10pt}
Lemma~\ref{lemma:opt} below concludes that the unique minimizer should be parallel to $\mSigma^{-1} \vmu$.
\begin{restatable}{lemma}{opt}\label{lemma:opt}
For any constant $C$, a unique solution of $\min_{\vw \in \mathbb{R}^d, \vw^\top \vmu = C} \frac{1}{2} \vw^\top \mSigma \vw$ is a rescaling of $\mSigma^{-1} \vmu$.
\end{restatable}
\vspace*{-10pt}
Any remaining details can be found in Appendix~\ref{proof:expected_loss_ERM}.
\end{proof}
\vspace*{-10pt}
Our Theorem~\ref{thm:expected_loss_ERM} implies that the minimizer $\w$ of expected ERM loss $\mathbb{E}_\kappa [\loss (\vw)]$ induce the Bayes optimal classifier. 
This may look obvious to some readers because $\mathbb{E}_\kappa [\loss (\vw)]$ is in fact equal to $\mathbb{E}_\kappa[y l(\vw^\top \vx) + (1-y) l(-\vw^\top \vx)]$, the population loss. However, notice that the optimal classifier in terms of the population loss is dependent on the loss $l$, and hence is not always equal to the Bayes optimal classifier.
Also, we will see theorems similar to Theorem~\ref{thm:expected_loss_ERM} in later sections; in particular, Theorem~\ref{thm:mask_minimizer} reveals how masking-based augmentation can \emph{distort} the optimal classifier of the corresponding expected training loss; thus, characterizing such optima is of importance.

Notice that the minimizer characterized in Theorem~\ref{thm:expected_loss_ERM} is for \emph{expected} ERM loss $\mathbb{E}_\kappa[\loss(\vw)]$. Since $\mathcal{D}_\kappa$ is not known to us and we only observe training data $ \{ (\vx_i, y_i)\}_{i=1}^n \overset{\mathrm{i.i.d}}{\sim} \mathcal{D}_\kappa $, we can only hope to get close to the minimizer by optimizing the training loss $\loss(\vw)$, 
and sufficiently many data samples are required to obtain a ``close enough'' one.
Thus, a natural question arises:
\begin{center}
    \vspace{-5pt}
    \emph{How many data points are required to find \\
    a near Bayes optimal classifier using the ERM loss?}\\
\end{center}
We present two theorems that answer the question above. The first one (Theorem~\ref{thm:ERM_sufficient}) shows that the number of samples growing exponentially with $\kappa^2$ is \emph{sufficient}, and the next one (Theorem~\ref{thm:ERM_necessary}) proves that this exponential growth with $\kappa$ is \emph{necessary}.
In other words, as the data distribution becomes more separable, the sample complexity for getting a Bayes optimal classifier grows exponentially. 

\begin{theorem}\label{thm:ERM_sufficient}
Let $\epsilon, \delta \in (0,1)$. Suppose the training set $S = \{(\vx_i, y_i)\}_{i=1}^n\overset{\mathrm{i.i.d}}{\sim} \mathcal{D}_\kappa$ with large enough $\kappa \in (0,\infty)$, and
\begin{equation*}
n = \frac{\exp(\Omega(\kappa^2)) }{\epsilon^4} \left (1+ \log \frac{1}{\epsilon} + \log \frac{1}{\delta} \right ).
\end{equation*}
Then, with probability at least $1-\delta$, the unique minimizer $\hat{\vw}_S^*$ of $\loss(\vw)$ exists and  $\cosim (\hat{\vw}_S^*, \mSigma^{-1}\vmu) \geq 1- \epsilon$.
\end{theorem}
\vspace*{-10pt}
\begin{proof}[Proof Sketch]
A useful tool for proving Theorem~\ref{thm:ERM_sufficient} is Lemma~\ref{lemma:minimizer_independent} inspired by the proof techniques used in \citet{dai2000convergence} and \citet{shapiro2008stochastic}.
\begin{lemma}\label{lemma:minimizer_independent}
Let $f(\cdot, \cdot) : \mathbb{R}^k \times \mathbb{R}^m \rightarrow \mathbb{R}$ be a real-valued function. Define functions $F_N :  \mathbb{R}^k \rightarrow \mathbb{R}$ and $\hat{F}_N : \mathbb{R}^k \rightarrow \mathbb{R}$ as
\vspace*{-5pt}
\begin{equation*}
    F (\vtheta) = \mathbb{E}_{\veta \sim \mathcal{P}}[f(\vtheta, \veta)], 
    ~\hat{F}_N  (\vtheta) = \frac{1}{N} \sum\nolimits_{i=1}^N f(\vtheta, \veta_i),
\end{equation*}
where $\mathcal{P}$ is a probability distribution on $\mathbb{R}^m$ and $\{\veta_i\}_{i=1}^N \overset{\mathrm{i.i.d.}}{\sim} \mathcal{P}$. Let $\mathcal{C}$ be a nonempty compact subset of $\mathbb{R}^k$ with diameter $D$. Suppose the following assumptions hold:
\vspace*{-5pt}
\begin{itemize}[leftmargin=3.5mm]
\item The functions $F$ and $\hat{F}_N$ have unique minimizers on $\mathcal{C}$ named $\vtheta^*$ and $\hat{\vtheta}_N^*$, respectively.
\item The function $F$ is $\alpha$-strongly convex on $\mathcal{C}$ ($\alpha>0$).
\item For any $\vtheta \in \mathcal{C}$, $\mathbb{E}_{\veta \sim \mathcal{P}}\left [e^{|f(\vtheta, \veta)- \mathbb{E}_{\veta \sim \mathcal{P}}[f(\vtheta, \veta)] |} \right]<M$.
\item There exists a function $g(\cdot): \R^k \rightarrow \R$ such that for any $\vtheta \in \mathcal{C}$ and $\veta \in \R^k$, it holds that $\lVert \nabla_\vtheta f(\vtheta,\veta) \rVert \leq g(\veta)$ and $\lVert \nabla_\vtheta \mathbb{E}_{\veta \sim \mathcal{P}}[f(\vtheta, \veta)]\rVert \leq \mathbb{E}_{\veta \sim \mathcal{P}}[g(\veta)]$. In addition, $\mathbb{E}_{\veta \sim \mathcal{P}}\left[ e^{g(\veta)}\right]< L$.
\end{itemize}
\vspace*{-5pt}
For each $0<\epsilon< \alpha^{-1/2}$, we have $\lVert \hat{\vtheta}_N^* - \vtheta \rVert < \epsilon$ with probability at least $1- \delta$  if $N$ is greater than
\begin{equation*}
\frac{C_1 M}{ \alpha^2 \epsilon^4} \log\left(\frac{3}{\delta} \max \left \{1, \left(\frac{C_2 k^{1/2} D L}{\alpha \epsilon^2 }\right)^k \right\} \right),
\end{equation*}
where $C_1, C_2>0$ are universal constants.
\end{lemma}

We use Lemma~\ref{lemma:minimizer_independent} by considering $\{ \veta_i\}_{i=1}^n \overset{\mathrm{i.i.d.}}{\sim} \mathcal{P}$ as our training dataset $S = \{(\vx_i,y_i)\}_{i=1}^n \overset{\mathrm{i.i.d.}}{\sim} \mathcal{D}_\kappa$, $\vtheta$ as weight vector $\vw$, and $f(\vtheta, \veta_i)$ as $(\vw, (\vx_i,y_i)) \mapsto y_i l(\vw^\top \vx_i) + (1-y_i) l(-\vw^\top \vx_i)$.
When we apply Lemma~\ref{lemma:minimizer_independent}, several quantities are sensitive to $\lVert \w \rVert$, the $\ell_2$ norm of $\w$ defined in Theorem~\ref{thm:expected_loss_ERM}, which we characterize in the following lemma.
\vspace{-5pt}
\begin{lemma}\label{lemma:ERM_norm} The unique minimizer $\w$ of expected ERM loss $\mathbb{E}_\kappa [\loss (\vw)]$ satisfies
\begin{equation*}
    \lVert \w \rVert = \Theta (\kappa).
\end{equation*}
\end{lemma}
\vspace*{-10pt}
Full proof of Theorem~\ref{thm:ERM_sufficient} appears in Appendix~\ref{proof:ERM_sufficient}.
\end{proof}
\vspace{-10pt}
Our sufficient sample complexity in Theorem~\ref{thm:ERM_sufficient} grows \emph{exponentially} with $\kappa$. The sufficient number of data points for the minimizer $\hat{\vw}_S^*$ of the ERM loss $\loss(\vw)$ to be close to the Bayes optimal classifier becomes exponentially larger for more well-separable data distributions. 

One may think that this exponential dependency may be just an artifact of our analysis and the exponential growth in $\kappa$ is in fact avoidable.
As an answer to this question, we introduce Theorem~\ref{thm:ERM_necessary} indicating that the exponential growth of sample complexity is \emph{inevitable}, dashing hopes for sub-exponential sample complexity bounds.
\begin{theorem}\label{thm:ERM_necessary}
Assume $S = \{(\vx_i, y_i)\}_{i=1}^n\overset{\mathrm{i.i.d}}{\sim} \mathcal{D}_\kappa$ with large enough $\kappa \in (0, \infty)$. If $n = \exp (\mathcal{O}(\kappa))$, then $\{(\vx_i, y_i) \}_{i=1}^n$ is linearly separable and $\cosim(\bar{\vw}_S, \mSigma^{-1}\vmu) < \frac{1+\cosim(\vmu,\mSigma^{-1}\vmu)}{2}$ with probability at least 0.99, where $\bar{\vw}_S$ is the $\ell_2$ max margin solution: 
\begin{equation}
\label{eq:2}
\begin{aligned}
    \bar{\vw}_S = \argmin\nolimits_{\vw \in \mathbb{R}^d} &\quad\lVert \vw \rVert^2 \\
    \mathrm{subject~to }&\quad(2y_i-1)\vw^\top \vx_i \geq 1.
\end{aligned}
\end{equation}
\end{theorem}
\begin{proof}[Proof Sketch]
If $\kappa$ is large, then $(2y-1)\vx$ will be concentrated near $\vmu$ where $(\vx,y) \sim \mathcal{D}_\kappa$. Hence, if $n$ is not large enough, $(2y_i-1)\vx_i$'s will be located inside a small ball centered at $\vmu$, with high probability. Therefore, $\{(\vx_i, y_i)\}_{i=1}^n$ is likely to be linearly separable. From the KKT condition of the $\ell_2$ max margin problem (\Eqref{eq:2}), $\bar{\vw}_S$ is a homogeneous combination of $(2y_i-1)\vx_i$'s. It implies $\bar{\vw}_S$ is directionally close to $\vmu$, not $\mSigma^{-1} \vmu$. The detailed proof is in Appendix~\ref{proof:ERM_necessary}.
We note that the numbers $\frac{1+ \cosim(\vmu, \mSigma^{-1}\vmu)}{2}$ and $0.99$ in the statement are not strictly necessary; they can be replaced by any other feasible constants.
\end{proof}

Theorem~\ref{thm:ERM_necessary} states that if we do not have sufficiently many data points, then the training dataset becomes linearly separable: there exists a direction $\vw$ such that $(2y_i-1) \vw^\top \vx_i > 0$ for all $i \in [n]$. However, this in fact means that there exist infinitely many directions $\vw$ that classify the data perfectly. Then why do we care about the specific $\ell_2$ max margin classifier in Theorem~\ref{thm:ERM_necessary}? \citet{soudry2018implicit, ji2019implicit} study the implicit bias of gradient descent on a linear model with logistic loss and show that this algorithm converges in direction to the $\ell_2$ max margin classifier when training data is linearly separable. In other words, if we run gradient descent on $\loss(\vw)$, then the algorithm will return a linear classifier defined by the direction of $\bar{\vw}_S$. This is why we analyze the $\ell_2$ max margin solution.
\vspace{-5pt}
\paragraph{The Curse of Separability.}Theorem~\ref{thm:ERM_necessary} implies that without the number of samples exponentially growing with $\kappa$, the solution found by gradient descent can be far from the Bayes optimal classifier. Combining Theorems~\ref{thm:ERM_sufficient} and \ref{thm:ERM_necessary}, we have an interesting conclusion that even though it is easier to correctly classify training dataset when $\kappa$ is larger, finding the theoretically optimal model becomes much harder. We refer to this interesting phenomenon as the \emph{curse of separability}; without Mixup, ERM training suffers the curse of separability due to its sample complexity growing exponentially in $\kappa$.
\vspace{-5pt}
\paragraph{Intuitive Explanations.}
What causes this phenomenon? 
When the data distribution is well-separable, limited training data can result in many decision boundaries having high training accuracy. However, among these, there is only one optimal decision boundary in terms of test accuracy, which is difficult to locate due to the scarcity of data points near it; this causes the curse of separability. We believe that this intuition extends beyond our simple setup; see Section~\ref{section:exp}.
\vspace{-5pt}
\section{Mixup Provably Mitigates the Curse of Separability in ERM}\label{section:mix}
In this section, we study a unifying framework of Mixup-style data augmentation techniques and show that Mixup significantly alleviates the curse of separability. 
We will first define the unifying framework along with Mixup loss, study the location of the minimizer of expected Mixup loss, and then study the sample complexity for Mixup training to achieve near Bayes optimal classifier.

We start by defining the Mixup loss in the following framework. From this point on, we assume $n\geq 2$.
\begin{definition}[Mixup Loss] Mixup loss with training set $S = \{(\vx_i,y_i)\}_{i=1}^n$ is defined by
\begin{equation*}
 \lossmix(\vw) \!=\! \frac{1}{n^2}\! \sum_{i,j=1}^n \tilde{y}_{i,j} l(\vw^\top \tilde{\vx}_{i,j}) + (1-\tilde{y}_{i,j}) l(-\vw^\top \tilde{\vx}_{i,j}),
\end{equation*}
where
\begin{align*}
   \lambda_{i,j} &\overset{\mathrm{i.i.d}}{\sim} \Lambda,\\
   \tilde{\vx}_{i,j} &~= g(\lambda_{i,j}) \vx_i  + (1-g(\lambda_{i,j}))\vx_j, \\
   \tilde{y}_{i,j} &~= \lambda_{i,j} y_i  + (1-\lambda_{i,j})y_j.
\end{align*}
The probability distribution $\Lambda$ satisfies $\support(\Lambda) \subset [0,1]$ and $\mathbb{P}_{\lambda \sim \Lambda}\left [\lambda \notin \{ 0,1 \} \land g(\lambda) \neq \frac{1}{2}\right]>0$. Also, the function $g : [0,1] \rightarrow [0,1]$ satisfies $g(z)>\frac{1}{2}$ if and only if $z >\frac{1}{2}$.
\end{definition}
Our definition is a broad framework that covers the original Mixup by choosing $\Lambda$ as the Beta distribution and $g(\cdot)$ as the identity function.
Similar to ERM loss $\loss ( \vw)$, Mixup loss $\lossmix(\vw)$ is also a stochastic function depending on the training set $S$. Unlike ERM loss, the expectation of Mixup loss \emph{depends} on $n$, as can be checked in Appendix~\ref{proof:expected_loss_mix}. As we did in Section~\ref{section:ERM}, we first characterize the minimizer of expected Mixup loss $\mathbb{E}_\kappa [\lossmix (\vw)]$.
\begin{theorem}\label{thm:expected_loss_mix}
For each $\kappa \in (0,\infty)$ and $n \in \mathbb{N}$, the expectation of Mixup loss $\mathbb{E}_\kappa [\lossmix (\vw)]$ has a unique minimizer $\wmix$. In addition, its direction is the same as the Bayes optimal solution $\mSigma^{-1}\vmu$.
\vspace{-10pt}
\end{theorem}
\begin{proof}[Proof Sketch]
We can rewrite $\mathbb{E}_\kappa[ \lossmix (\vw)]$ as the form
\begin{equation*}
\mathbb{E} \left[ \sum_{i=1}^k a_i l \left( b_i\kappa^{-1/2} (\vw^\top \mSigma \vw)^{1/2} Z   + c_i \vw^\top \vmu \right)\right],
\end{equation*}
where $Z \sim N(0,1)$ and $a_i$, $b_i$, $c_i$'s are real-valued random variables depending on $\Lambda$; in particular, $a_i$, $b_i$'s are positive. Then, the same proof idea of Theorem~\ref{thm:expected_loss_ERM} works.
\vspace{-5pt}
\end{proof}

\paragraph{Mixup Does Not Distort Training Loss.}
Theorem~\ref{thm:expected_loss_mix} shows that the expected Mixup loss also has its unique minimizer pointing to the Bayes optimal direction. In other words, this theorem implies that the pair-wise mixing done in Mixup does not introduce any bias or distortion in the training loss, at least in our setting. This is one benefit that Mixup has compared to other masking-based augmentations, as we will see in Section~\ref{section:mask}.

In order to investigate the sample complexity for achieving a near Bayes optimal classifier when we train with Mixup loss, one could speculate that the same approach using Lemma~\ref{lemma:minimizer_independent} should work. However, this is not the case; analysis of the Mixup loss has to overcome a significant barrier because the mixed data points $\{(\tilde{\vx}_{i,j}, \tilde{y}_{i,j})\}_{i,j=1}^n$ are no longer independent of one another. To overcome this difficulty, we prove the following lemma in Appendix~\ref{proof:minimizer_dependent}, which could be of independent interest:
\begin{lemma}\label{lemma:minimizer_dependent}
Let $f(\cdot, \cdot) : \mathbb{R}^k \times \mathbb{R}^m \rightarrow \mathbb{R}$ be a real-valued function. Define functions $F_N :  \mathbb{R}^k \rightarrow \mathbb{R}$ and $\hat{F}_N : \mathbb{R}^k \rightarrow \mathbb{R}$ as
\begin{equation*}
    F_N(\vtheta) = \frac{1}{N^2}\!\! \sum_{i,j\in [N]}\!\! f(\vtheta, \veta_{i,j}), ~F_N(\vtheta) = \mathbb{E}\left[ \hat{F}_N(\vtheta) \right]
\end{equation*}
where  $\mathcal{P}_{i,j}$ are probability distributions on $\mathbb{R}^m$,$\veta_{i,j} \sim \mathcal{P}_{i,j}$, and expectation is taken over all randomness. Let $\mathcal{C}$ be a nonempty compact subset of $\mathbb{R}^k$ with diameter $D$. Suppose the following assumptions hold:
\begin{itemize}[leftmargin=3.5mm]
\vspace{-5pt}
\item  For $i_1,i_2,j_1,j_2\in [N]$, if $\{i_1\}\cup \{ j_1\}$ and $ \{i_2\}\cup \{j_2\}$ are disjoint, then $\veta_{i_1, j_1}$ and  $\veta_{i_2, j_2}$ are independent.
\vspace{-5pt}
\item  The functions $F_N$ and $\hat{F}_N$ have unique minimizers on $\mathcal{C}$ named $\vtheta_N^*$ and $\hat{\vtheta}_N^*$, respectively.
\vspace{-5pt}
\item The function $F_N$ is $\alpha$-strongly convex on $\mathcal{C}$ ($\alpha>0$).
\vspace{-5pt}
\item For any $\vtheta \in \mathcal{C}$ and $i,j\in[N]$, 
\begin{equation*}
    \mathbb{E}_{\veta \sim \mathcal{P}_{i,j}}\left [e^{|f(\vtheta, \veta)- \mathbb{E}_{\veta \sim \mathcal{P}_{i,j}}[f(\vtheta, \veta)] |} \right]<M.
    \vspace{-10pt}
\end{equation*}
\item There exists a function $g(\cdot): \R^k \rightarrow \R$ such that for any $\vtheta \in \mathcal{C}$ and $\veta \in \R^k$, it holds that  $\lVert \nabla_\vtheta f(\vtheta,\veta) \rVert \leq g(\veta)$. In addition,  $\lVert \nabla_\vtheta \mathbb{E}_{\veta \sim \mathcal{P}_{i,j}}[f(\vtheta, \veta)]\rVert \leq \mathbb{E}_{\veta \sim \mathcal{P}_{i,j}}[g(\veta)]$ and $\mathbb{E}_{\veta \sim \mathcal{P}_{i,j}}\left[ e^{g(\veta)}\right]< L$ for each $\vtheta \in \mathcal{C}$ and  $i,j \in [N]$.
\end{itemize}
For each $0<\epsilon< \alpha^{-1/2}$, we have $\lVert \hat{\vtheta}_N^* - \vtheta \rVert < \epsilon$ with probability at least $1- \delta$  if $N$ is greater than
\begin{equation*}
\frac{C_1' M}{ \alpha^2 \epsilon^4} \log\left(\frac{3}{\delta} \max \left \{1, \left(\frac{C_2' k^{1/2} D  L}{\alpha \epsilon^2 }\right)^k \right\} \right),
\end{equation*}
where $C_1', C_2'>0$ are universal constants.
\end{lemma}
\begin{proof}[Proof Sketch]
When we follow the proof of Lemma~\ref{lemma:minimizer_independent}, the challenging part is that we cannot use the fact that the expected value of a product of independent random variables is equal to a product of expectations of individual random variables. We overcome this by partitioning the $N^2$ random variables into batches such that random variables belonging to the same batch are independent (Lemma~\ref{lemma:partition}) and then applying a generalized Cauchy-Schwartz inequality (Lemma~\ref{lemma:cauchy}) to bound an expectation of a product of dependent random variables (each corresponding to a batch) with a product of expectations of the random variables. A formal proof can be found in Appendix~\ref{proof:minimizer_dependent}. 
\end{proof}

Similar to the proof of Theorem~\ref{thm:ERM_sufficient}, considering $\{ \veta_{i,j} \}_{i,j=1}^n$ as the ``mixed'' dataset $\{ (\tilde{\vx}_{i,j}. \tilde{y}_{i,j})\}_{i.j=1}^n$, $\vtheta$ as weight vector $\vw$, and $f(\vtheta, \veta_{i,j})$ as $(\vw, (\tilde{\vx}_{i,j}, \tilde{y}_{i,j})) \mapsto \tilde{y_i}l(\vw^\top \tilde{\vx}_{i,j} + (1-\tilde{y_i})l(-\vw^\top \tilde{\vx}_{i,j})$ induces the following theorem.
\begin{theorem}\label{thm:mixup_convergence}
Let $\epsilon, \delta \in (0,1)$. Suppose the training set $S = \{(\vx_i, y_i)\}_{i=1}^n\overset{\mathrm{i.i.d}}{\sim} \mathcal{D}_\kappa$ with large enough $\kappa \in (0,\infty)$ and
\begin{equation*}
    n = \frac{\Omega(\kappa^2)}{\epsilon^4}\left (1 + \log \frac{1}{\epsilon} + \log \frac{1}{\delta} \right).
\end{equation*}
Then, with probability at least $1-\delta$, the unique minimizer $\hat{\vw}_{\mathrm{mix}, S}^*$ of $\lossmix(\vw)$ exists and $\cosim(\hat{\vw}_{\mathrm{mix}, S}^*, \mSigma^{-1}\vmu) \geq 1-\epsilon$.
\end{theorem}
Theorem~\ref{thm:mixup_convergence} indicates that the sample complexity for finding a near Bayes optimal classifier with Mixup training grows only quadratically in $\kappa$. Compared to the necessity of exponential growth demonstrated in Theorem~\ref{thm:ERM_necessary}, Theorem~\ref{thm:mixup_convergence} shows that there is a provable exponential gap between ERM training and Mixup training.
\paragraph{Intuition on the Smaller Sample Complexity of Mixup.}We would like to provide some intuition on our result before we introduce technical aspects. Unlike ERM training, Mixup training uses mixed training points and these can be located near the optimal decision boundary when we mix two data points having distinct labels. This closeness of mixed points to the Bayes optimal decision boundary makes it easier to correctly locate the boundary.
\vspace*{-5pt}
\paragraph{Different Scaling of $\w$ and $\wmix$.}
On the technical front, the difference in sample complexity from Theorem~\ref{thm:ERM_sufficient} stems from the difference between the norm of expected loss minimizers $\w$ and $\wmix$, which determine several meaningful terms when we apply Lemma~\ref{lemma:minimizer_dependent}.
The following lemma characterizes the $\ell_2$ norm of $\wmix$, defined in Theorem~\ref{thm:expected_loss_mix}:
\begin{lemma}\label{lemma:mixup_norm}
The unique minimizer $\wmix$ of expected Mixup loss $\mathbb{E}_\kappa [\lossmix(\vw)]$ satisfies
\begin{equation*}
    \lVert \wmix \rVert  = \Theta (1).\footnote{We stress that the upper/lower bounds are independent of $n$.}
\end{equation*}
\end{lemma}
\vspace*{-10pt}
Comparison with Lemma~\ref{lemma:ERM_norm} reveals that the minimizer $\wmix$ of expected Mixup loss is much closer to zero compared to that of the expected ERM loss. 
For ERM loss, large scaling of weight leads to smaller loss for correctly classified data points. Also, for larger $\kappa$, larger portion of population data will be correctly classified by $\mSigma^{-1} \vmu$. Hence, $\lVert \w \rVert$ increases as $\kappa$ increases. However, in case of Mixup, mixed labels prevent $\wmix$ from growing with $\kappa$. To illustrate why, consider the case $\tilde y_{i,j} = 0.5$, which leads to a Mixup loss
\begin{equation*}
0.5 l(\vw^\top \tilde \vx_{i,j}) + 0.5 l(-\vw^\top \tilde \vx_{i,j}).
\end{equation*}
Notice here that the loss becomes large whenever $\vw^\top \tilde \vx_{i,j}$ is large in magnitude, no matter the sign is. For this reason, $\wmix$ should not increase with $\kappa$ and this leads to the smaller sample complexity of Mixup.

In this section, we showed that Mixup training does not distort the training loss (Theorem~\ref{thm:expected_loss_mix}) and also that Mixup provides a great remedy to the curse of separability phenomenon (Theorem~\ref{thm:mixup_convergence}), because the sample complexity only grows in $\kappa^2$ while ERM suffers at least exponential growth in $\kappa$. Thus, Mixup provably mitigates the curse of separability and helps us find a model with the best generalization performance.
\vspace*{-5pt}
\section{Masking-based Mixup Can Distort Training}\label{section:mask}
Recent Mixup variants for image data \citep{yun2019cutmix, kim2020puzzle, kim2021co, liu2022automix} use masking on input data. In this section, we investigate how masking-based augmentation techniques work in our data distribution setting. We consider the class of masking-based Mixup variants formulated as follows.
\begin{definition}(Masking-based Mixup Loss) Masking-based Mixup loss with training set $S = \{ (\vx_i, y_i)\}_{i=1}^n$ is defined as
\vspace{-10pt}
\begin{align*}
   \lossmask(\vw):= \frac{1}{n^2} &\sum_{i,j=1}^n \tilde{y}_{i,j}^\mathrm{mask} l(\vw^\top \tilde{\vx}_{i,j}^\mathrm{mask})\\
   &\qquad+ (1-\tilde{y}_{i,j}^\mathrm{mask}) l(-\vw^\top \tilde{\vx}_{i,j}^\mathrm{mask}) ,
   \vspace{-10pt}
\end{align*}
where
\vspace{-10pt}
\begin{align*}
   (\mM_{i,j},\lambda_{i,j}) &\overset{\mathrm{i.i.d}}{\sim} \mathcal{M},\\
   \tilde{\vx}_{i,j}^\mathrm{mask} &~= \mM_{i,j} \odot  \vx_i  + (\vone -\mM_{i,j}) \odot \vx_j, \\
   \tilde{y}_{i,j}^\mathrm{mask} &~= \lambda_{i,j} y_i  + (1-\lambda_{i,j})y_j.
\end{align*}
Here, $\support (\mathcal{M}) \subset \{0,1\}^d \times[0,1]$. 
\end{definition}
\vspace{-10pt}
In our definition of masking-based mixup loss, we formulate the masking operation on data points by element-wise multiplication with vectors having entries from only $0$ and $1$. This formulation includes CutMix \citep{yun2019cutmix} which is simplest type of masking-based Mixup. State-of-the-art Mixup variants having more complex masking strategies \citep{kim2020puzzle, kim2021co, liu2022automix} are out of the scope of this paper. We also introduce the following assumption on masking.
\begin{assumption} \label{Assumption:mask}
The set $\support(\vmu \odot (2\mM-\vone))$ spans $\mathbb{R}^d$ and $\mathbb{P}[\lambda \vone_{\mM = \mM_0} \not \in \{0,1\}]>0$ for each $\mM_0 \in \support(\mM)$ where $(\mM, \lambda) \sim \mathcal{M}$.
\end{assumption}
\vspace{-5pt}
Before we move on to our main results, we demonstrate why our formulation and assumption hold for CutMix. CutMix samples mixing ratio $\lambda_{i,j}$ from beta distribution and masking vector $\mM_{i,j}$ is uniformly sampled from vectors in which the number of $1$'s is proportional to $\lambda_{i,j}$. 
Since the support of beta distribution is $[0,1]$, support of $\mM_{i,j}$ contains the standard basis of $\mathbb{R}^d$. Hence if all the components of $\vmu$ are nonzero, Assumption~\ref{Assumption:mask} holds.

Recall that we defined $\mathcal{D}_\infty$ as a limit behavior of $\mathcal{D}_\kappa$ as $\kappa \rightarrow \infty$ and it is independent of $\mSigma$. Hence, $\mathbb{E}_\infty [\lossmask(\vw)]$ is independent of $\mSigma$. The following theorem investigates the minimizer of the expected masking-based Mixup loss $\mathbb{E}_\kappa [\lossmask (\vw)]$, focusing on large enough $\kappa$.

\begin{theorem}\label{thm:mask_minimizer}
Suppose Assumption~\ref{Assumption:mask} holds and the training set $S = \{(\vx_i, y_i)\}_{i=1}^n\overset{\mathrm{i.i.d}}{\sim} \mathcal{D}_\kappa$ with $\kappa \in (0, \infty]$. Then, the expected loss $\mathbb{E}_\kappa [\lossmask(\vw)]$ has a unique minimizer $\wmask$. In addition, $\lim_{\kappa \rightarrow \infty} \wmask[\kappa] = \wmask[\infty]$.
\end{theorem}
\vspace*{-10pt}
\begin{proof}[Proof Sketch]
For the uniqueness of the minimizer, we use almost the same strategy as the uniqueness parts of Theorem~\ref{thm:expected_loss_ERM} and Theorem~\ref{thm:expected_loss_mix}. The only part that requires a different strategy is the uniqueness for $\kappa = \infty$; in this case, we exploit Assumption~\ref{Assumption:mask}. Also, from Assumption~\ref{Assumption:mask}, we can show that $\mathbb{E}_\infty \left[ \lossmask (\vw) \right]$ is $\alpha$-strongly convex with some $\alpha>0$.
Using strong convexity constant $\alpha$, we establish upper bound on $\lVert \wmask - \wmask[\infty] \rVert$ represented by $\mathbb{E}_\kappa \left[ \lossmask (\vw) \right]  - \mathbb{E}_\infty \left[ \lossmask (\vw) \right]$ for several values of $\vw$ contained in a bounded set.
We finish up by showing $\mathbb{E}_\kappa \left[ \lossmask (\vw) \right]  - \mathbb{E}_\infty \left[ \lossmask (\vw) \right] \rightarrow 0$ uniformly on the bounded set as $\kappa \rightarrow \infty$.
\vspace{-5pt}
\end{proof} 
\paragraph{Masking-based Mixup Can Distort Training Loss.}
Unlike ERM loss and Mixup loss, characterizing the exact direction of $\wmask$ is challenging since $\mathbb{E}_\kappa \left[ \lossmask (\vw) \right]$ has a more complicated form because of masking. However,
our Theorem~\ref{thm:mask_minimizer} implies that $\wmask$ leads to a solution only depending on $\vmu$ and deviates from the Bayes optimal solution for sufficiently large $\kappa$. Even though Theorem~\ref{thm:mask_minimizer} guarantees deviation of $\wmask$ from the Bayes optimal direction only for large $\kappa$, our experimental results in Section~\ref{exp:gaussian} suggest that our result holds even for moderately sized $\kappa$.

One might be wondering whether the same thing can be said for the minimizer of the expected Mixup loss $\mathbb{E}_\infty[\lossmix (\vw)]$; we illustrate why the proof idea of Theorem~\ref{thm:mask_minimizer} does not work for Mixup. While $\mathbb{E}_\infty[\lossmask (\vw)]$ is strongly convex and has a unique minimizer, minimizers of $\mathbb{E}_\infty[\lossmix (\vw)]$ are not unique since if $\vu \in \mathbb{R}^d$ is a minimizer of  $\mathbb{E}_\infty[\lossmix (\vw)]$, $\vu+\vv$ is also a minimizer of $\mathbb{E}_\infty[\lossmix (\vw)]$ for any $\vv \in \mathbb{R}^d$ orthogonal to $\vmu$. 
Therefore, $\wmix$ maintains its direction even though $\mathbb{E}_\infty[\lossmix (\vw)]$ is independent of $\mSigma$; since there are many minimizers in $\mathbb{E}_\infty[\lossmix (\vw)]$, $\wmix$ converges to the Bayes optimal minimizer (among the many) as $\kappa \to \infty$.

While masking-based Mixup training does not necessarily lead to Bayes optimal classifiers, the sample complexity proof still works. In fact, we can find the solution near the minimizer $\wmask$ of the expected loss $\mathbb{E}_{\kappa}[\lossmask(\vw)]$ with fewer samples.
\begin{theorem}\label{thm:mask_convergence}
Let $\epsilon, \delta \in (0,1)$. Suppose Assumption~\ref{Assumption:mask} holds and the training set $S = \{(\vx_i, y_i)\}_{i=1}^n\overset{\mathrm{i.i.d}}{\sim} \mathcal{D}_\kappa$ with large enough $\kappa \in (0,\infty)$.
If 
\begin{equation*}
    n = \frac{\Omega(1)}{\epsilon^4} \left (1+ \log \frac{1}{\epsilon} + \log \frac{1}{\delta} \right ),
    \vspace{-5pt}
\end{equation*}
then with probability at least $1-\delta$, the unique minimizer $\hat{\vw}_{\mathrm{mask}, S}^*$ of $\lossmask(\vw)$ exists and $\cosim(\hat{\vw}_{\mathrm{mask}, S}^*, \wmask) \geq 1-\epsilon$. 
\end{theorem} 
Theorem~\ref{thm:mask_convergence} indicates masking-based Mixup also mitigates the curse of separability (even better than Mixup).
\paragraph{Why Even Smaller Sample Complexity?} We consider a simple case $\vmu = [1,1]^\top$, which is sufficient to convey our intuition. When $\kappa$ is large, most of the data points are likely to be concentrated around $\vmu$ and $-\vmu$. Since Mixup uses linearly interpolated data points, all the mixed training data will be close to a line that passes through the origin and has direction $\vmu$. It means that all raw data points or mixed points are almost orthogonal to $[1,-1]^\top$. Consequently, both ERM and Mixup training loss will be less sensitive to perturbations orthogonal to $\vmu$: i.e., $\loss(\vw ) \approx \loss(\vw + t[1,-1]^\top)$ and $\lossmix(\vw ) \approx \lossmix(\vw + t[1,-1]^\top)$ for any small $t$, which makes it difficult to locate the exact minimizer of the objective loss.
In contrast, masking-based Mixup uses cut-and-pasted data points such as $\vz_1\approx [1,-1]^\top$, constructed by pasting the first coordinate of positive data and the second coordinate of negative data. Masking-based Mixup also uses $\vz_2 \approx [1,1]^\top$, obtained from mixing two positive data points. We observe that ${\vz_1, \vz_2 }$ span the whole space $\mathbb{R}^2$. Therefore, for any perturbation $\vv \in \mathbb{R}^2$, $\vv^\top \vz_1 \not \approx 0$ and/or $\vv^\top \vz_2 \not \approx 0$ should hold, which implies that $\lossmask(\vw ) \not \approx \lossmask(\vw + \vv)$ for any perturbation $\vv$. In other words, the masking-based Mixup loss is sensitive to perturbations in any direction. This makes it easier to locate the exact minimizer of the objective loss, even when $\kappa$ is large.

In this section, we showed that masking-based Mixup mitigates the curse of separability even better than Mixup, but unfortunately, making-based Mixup can find a classifier that is far from being Bayes optimal due to Theorem~\ref{thm:mask_minimizer}. One may think that these results are contradictory to the empirical success of masking-based Mixup such as CutMix~\citep{yun2019cutmix} on image data. However, the regularization effect of Mixup variants is highly dependent on the data, as also noted by \citet{parkunified}. Therefore, our conclusion in Section 5 does not necessarily contradict the success of masking-based Mixup on practical image data. We speculate that the distortion effect of masking-based Mixup on complex image data may be small or even beneficial (e.g., by increasing the chance of co-occurrence of some useful features). In this case, the small sample complexity of masking-based Mixup would be helpful. However, a rigorous theoretical analysis is beyond the scope of our current work and is an essential direction for future research.
\section{Experiments} \label{section:exp}
In this section, we present several experimental results to support our findings. 
\subsection{Experiments on Our Setting $\mathcal{D}_\kappa$}\label{exp:gaussian}
First, we provide empirical results on our setting. We properly choose $\vmu$ and $\mSigma$ with $\lVert \vmu \rVert^2 = \lVert \mSigma \rVert$ so that $\vmu$ and the Bayes optimal direction $\mSigma^{-1}\vmu$ have different directions, i.e., $\vmu$ is not an eigenevector of $\mSigma$. 
We provide exact values of our choice of $\vmu$ and $\mSigma$ in Appendix~\ref{setting:gaussian}.  
We compare three training methods ERM, Mixup, and Mask Mixup.
Mask Mixup we considered is a kind of masking-based Mixup with $\mathcal{M}$ defined by the following:
we say $(\mM, \lambda) \sim \mathcal{M}$ when $\lambda$ is drawn from beta distribution $\mathrm{Beta} (\alpha,\alpha)$ and each component of $\mM$ follows $\mathrm{Bernoulli}(\lambda)$.

\begin{figure}[t]
    \centering
\subfigure[$\kappa = 0.5$]{\includegraphics[width= 0.238\textwidth]{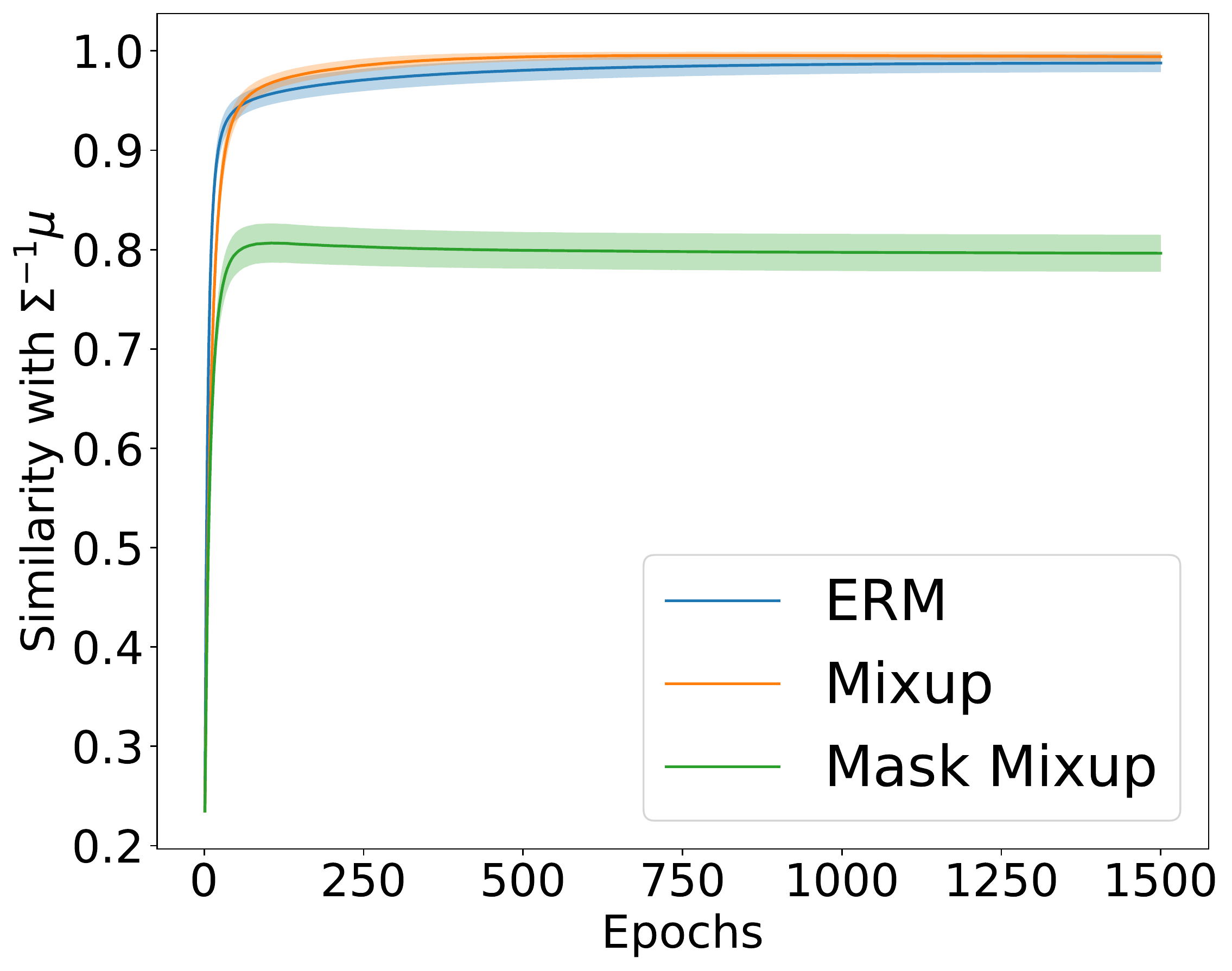}}
\subfigure[$\kappa = 2.0$]{\includegraphics[width= 0.238\textwidth]{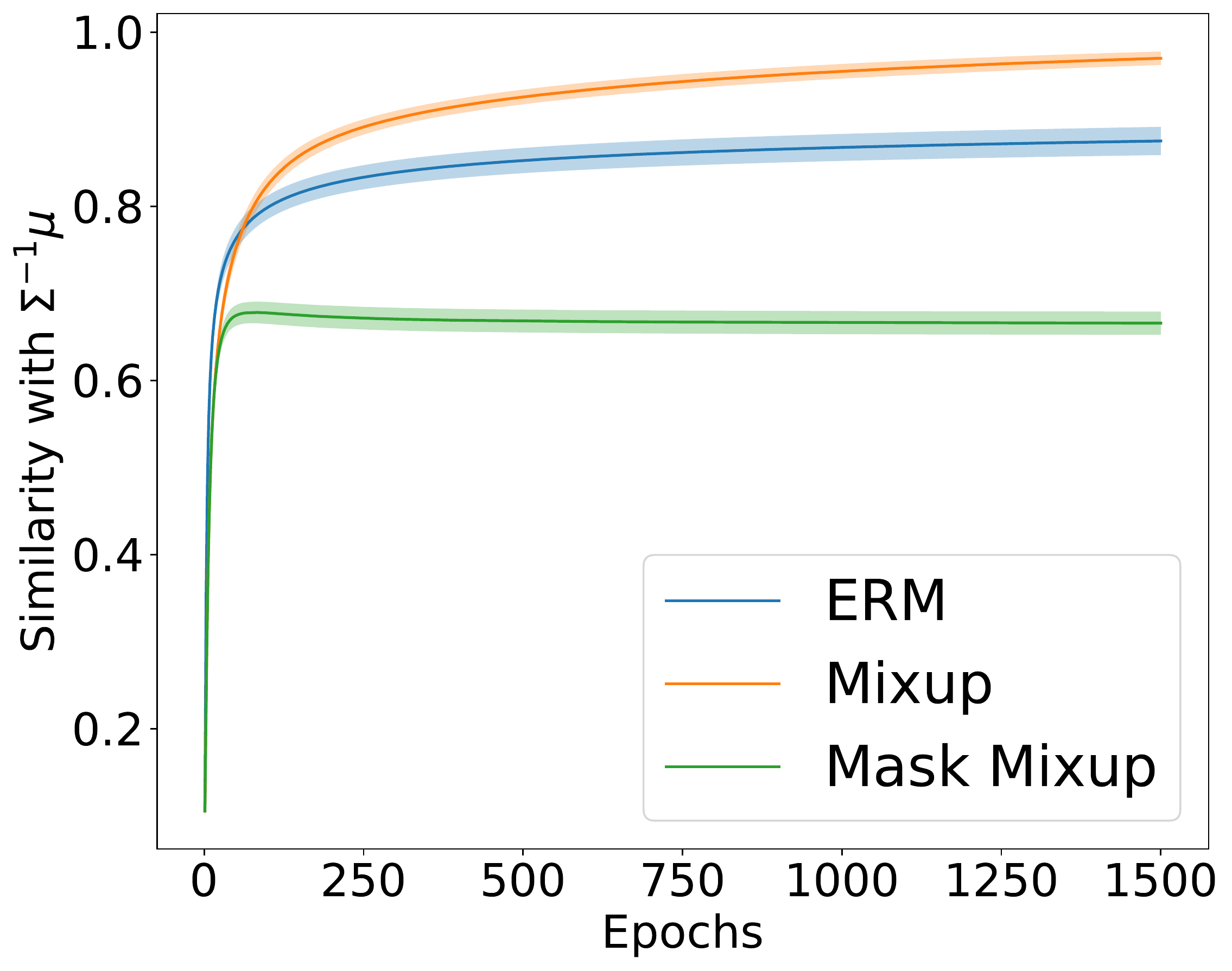}}
\vspace{-15pt}
\caption{Cosine similarity between learned weight and the Bayes optimal direction $\mSigma^{-1} \vmu$. ERM successfully finds the Bayes optimal classifier when $\kappa$ is small ($\kappa=0.5$) and fails when $\kappa$ is larger ($\kappa = 2.0$) because of the curse of separability. Meanwhile, Mixup succeeds in both cases since Mixup mitigates the curse of separability. Mask Mixup fails in both cases due to distortion.}
\label{figure:gaussian}
\vspace{-20pt}
\end{figure}

We compare two different values of $\kappa$; $\kappa = 0.5$ and $\kappa = 2.0$. We train for 1500 epochs using randomly sampled 500 training samples from each $\mathcal{D}_\kappa$ and full gradient descent with learning rate $1$ and we choose $\alpha = 1$ for the hyperparameter of Mixup and Mask Mixup. We run 500 times with fixed initial weight but different samples of training sets and we plot cosine similarity between the trained weight and the Bayes optimal direction during training in Figure~\ref{figure:gaussian}. 

For the case $\kappa = 0.5$, ERM and Mixup lead to the Bayes optimal classifier. However, for the case $\kappa = 2.0$, ERM finds a solution deviating from the Bayes optimal solution, while Mixup still finds almost accurate solutions. This result is predicted by our theoretical findings; ERM suffers the curse of separability and Mixup mitigates it. Also, we can check the minimizers of the Mask Mixup loss deviate significantly from the Bayes optimal direction for both cases $\kappa = 0.5$ and $\kappa = 2.0$, even though our theoretical result (Theorem~\ref{thm:mask_minimizer}) focus on large $\kappa$. We provide additional results on more various values of $\kappa$, the number of samples $n$, and the dimension of data $d$ in Appendix~\ref{result:gaussian}.
\vspace{-5pt}
\subsection{2D Classification on Synthetic Data} \label{exp:2d}
\vspace{-5pt}
\begin{figure*}[h]
    \vspace{-10pt}
    \centering
    \subfigure[Large noise, less separable setting]{\includegraphics[width=0.45\textwidth]{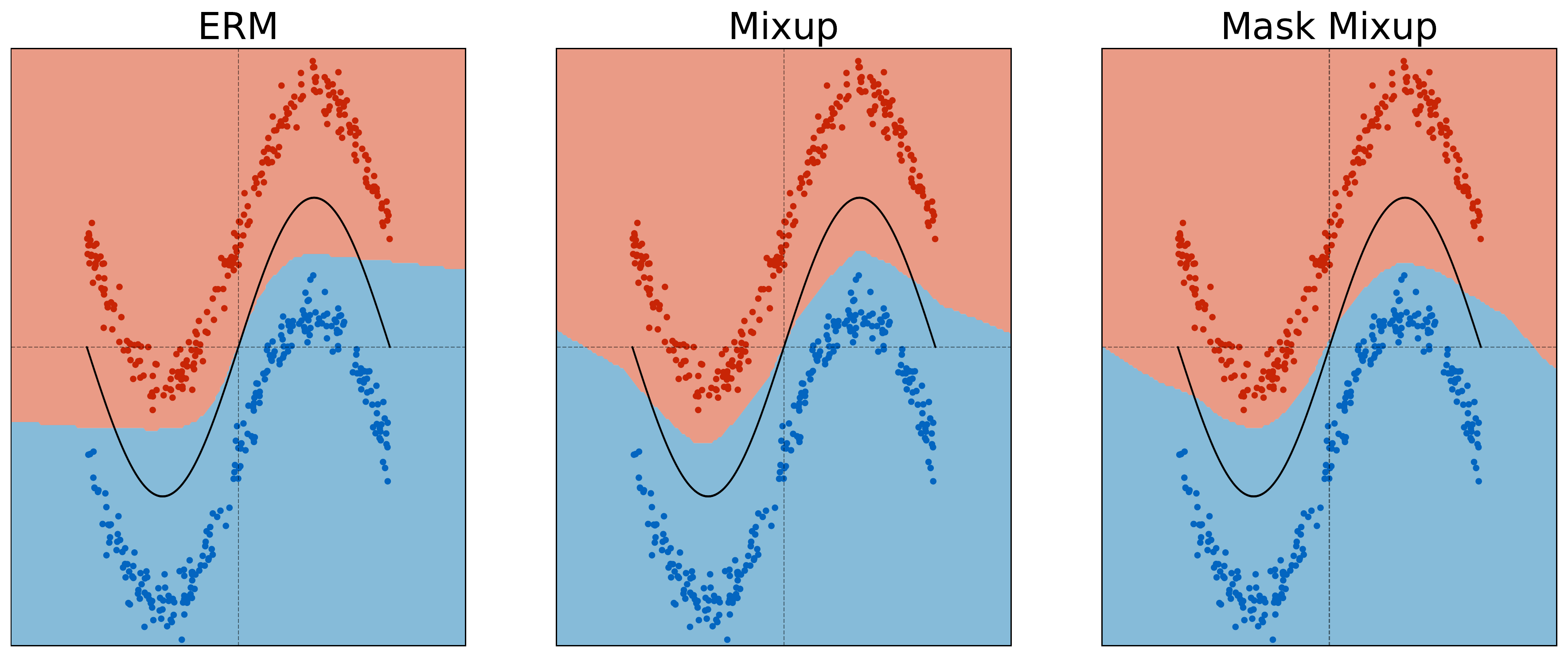}\label{figure:2d(a)}}~
    \subfigure[Small noise, well separable setting]{\includegraphics[width=0.45\textwidth]{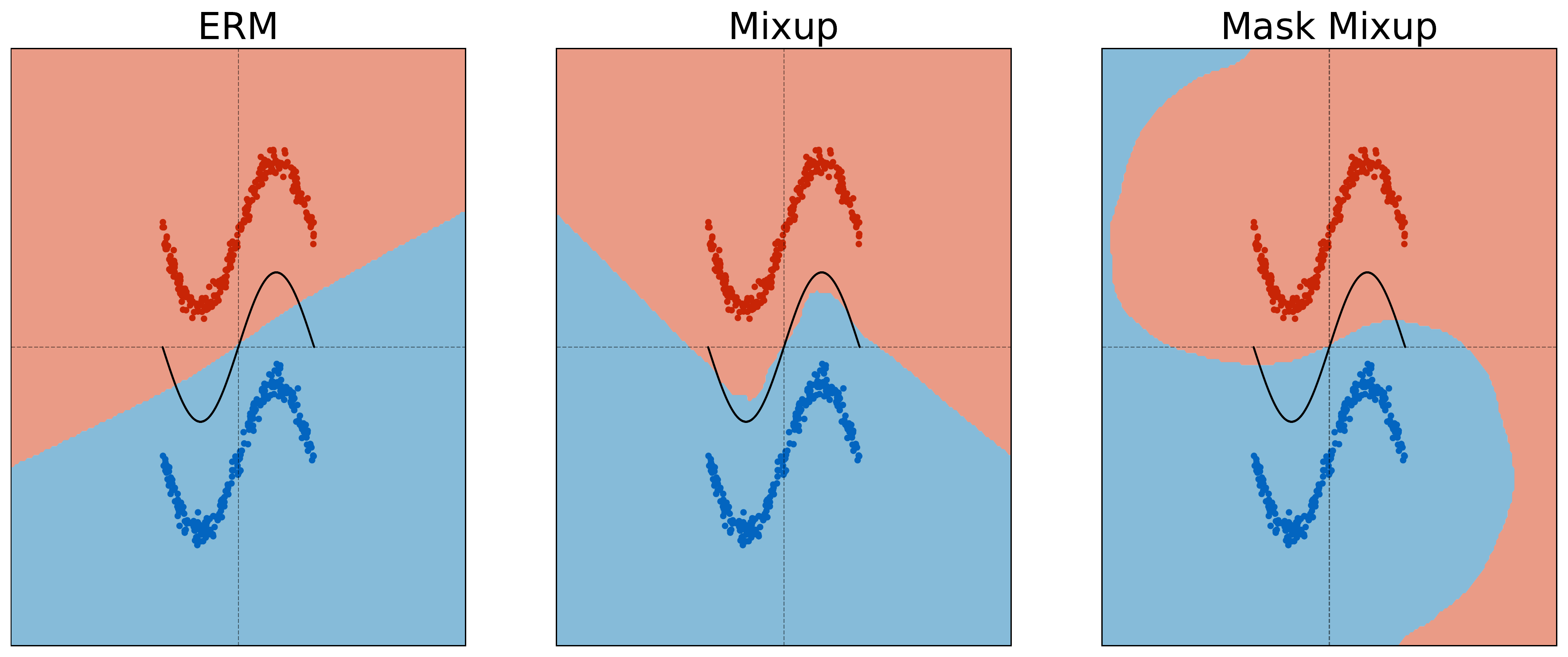} \label{figure:2d(b)}}
    \vspace{-10pt}
    \caption{Boundary decision of trained models with ERM loss, Mixup loss and Mask Mixup loss}
    \vspace{-10pt}
    \label{figure:2d}
    \vspace{-5pt}
\end{figure*}
We also provide empirical results supporting that the intuitions gained from our analysis extend beyond our settings. We consider training a two-layer ReLU network with 500 hidden nodes on 2D synthetic data with binary labels having sine-shaped noise\footnote{The noise consists of uniformly sampled $x$-coordinate and $y$-coordinate having sine value with additional Gaussian noise.} from its mean for each class. As a result of the noise, the optimal decision boundary is also sine-shaped.
We consider two settings with different magnitudes of noise while keeping the means the same. 
Using three methods ERM, Mixup, and Mask Mixup (which we introduced in the previous subsection), we train for 1500 epochs using 500 samples of data points and Adam \citep{kingma2014adam} with full batch, learning rate 0.001 and using default hyperparameters of $\beta_1 = 0.9, \beta_2 = 0.999$. We also use $\alpha=1$ for the hyperparameter of Mixup and Mask Mixup. 

Figure~\ref{figure:2d(a)} plots the decision boundaries (red vs.\ blue) of trained models in the setting with larger noise, which corresponds to a less separable setting. We also draw the Bayes optimal boundaries with black solid lines.  All ERM, Mixup, and Mask Mixup find decision boundaries that reflect the sine-shaped optimal decision boundary. Figure~\ref{figure:2d(b)} shows the results with smaller noise, i.e., a more separable setting. The decision boundary of ERM degenerates to a linear boundary, ignoring the sine-shaped noise.
However, even though Mixup slightly distorts training,\footnote{
One may check that the optimal decision boundary becomes similar to a sine-shaped curve with a slightly smaller amplitude.} Mixup finds a nonlinear boundary that captures sine shape even when data is highly separated. 
This result is consistent with our findings that \emph{Mixup mitigates the curse of separability}, even outside our simple settings. 
Also, the decision boundary of models trained by using Mask Mixup is nonlinear, which may come from smaller sample complexity, but it  seems to suffer more distortion compared to Mixup.
\vspace{-5pt}
\subsection{Classification on CIFAR-10}
\vspace{-5pt}
We also conduct experiments on the real-world data CIFAR-10~\citep{krizhevsky2009learning}. 
To compare three methods ERM, Mixup, and CutMix, we train VGG19~\citep{simonyan2014very} and ResNet18~\citep{he2016deep} for 300 epochs on the training set with batch size 256 using SGD with weigh decay $10^{-4}$ and we choose $\alpha = 1$ for the hyperparameter of Mixup and CutMix. Also, we use a learning rate $0.1$ at the beginning and divide it by 10 after 100 and 150 epochs. Unlike linear models and 2D classification tasks, the decision boundaries of deep neural networks trained with complex data are intractable. 
Hence, following the method considered in \citet{nar2019cross,pezeshki2021gradient}, we use the norm of input perturbation\footnote{We apply the projected gradient descent attack implemented by \citet{rauber2017foolbox} to compute the perturbation on input. } required to cross the decision boundary to investigate the complex decision boundary. 

Figure~\ref{fig:CIFAR10} indicates that Mixup tends to find decision boundaries farther from overall data points than decision boundaries obtained by ERM. This is consistent with our intuition on the curse of separability and how Mixup mitigates it. In addition, the plots of CutMix are placed between the plots of ERM and Mixup. As also observed in Figure~\ref{figure:2d(b)}, we believe that the combination of distortion and smaller sample complexity results in such a trend.

\begin{figure}[t]
    \vspace{-10pt}
    \centering
    \subfigure[VGG19]{\includegraphics[width= 0.23\textwidth]{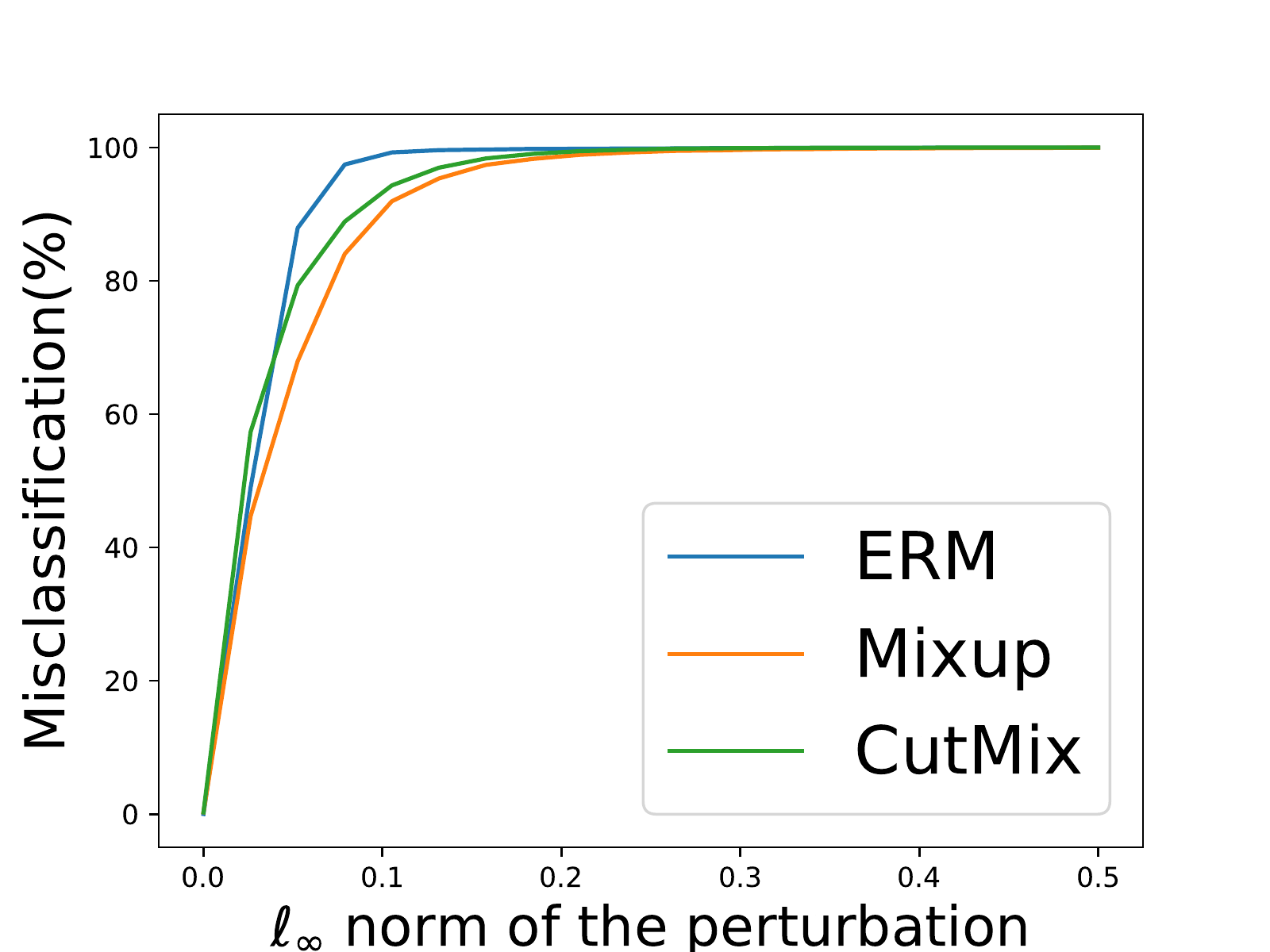}}
    \subfigure[ResNet18]{\includegraphics[width= 0.23\textwidth]{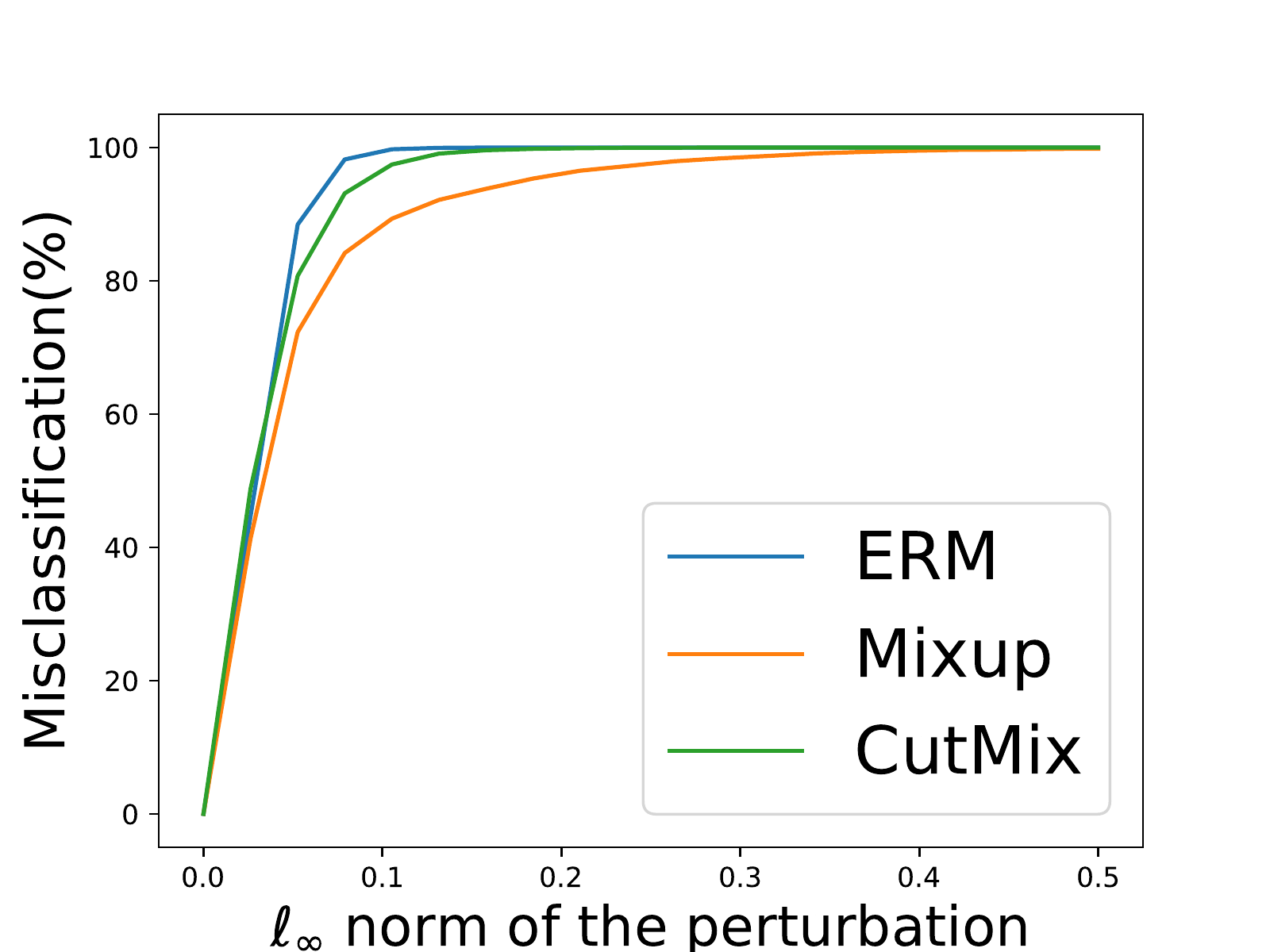}}
    \vspace{-10pt}
    \caption{Estimation of decision boundaries of (a) VGG19 and (b) ResNet18 trained with CIFAR-10}
    \label{fig:CIFAR10}
    \vspace{-20pt}
\end{figure}
\section{Conclusion}
We analyzed how Mixup-style training influences the sample complexity for getting optimal decision boundaries in logistic regression on data drawn from two symmetric Gaussian distributions with separability constant $\kappa$. Interestingly, we proved that vanilla training suffers the curse of separability. More precisely, ERM requires an exponentially increasing number of data with $\kappa$ for finding a near Bayes optimal classifier. We proved that Mixup mitigates the curse of separability and the sample complexity for finding optimal classifier grows only quadratically with $\kappa$. We further investigated masking-based Mixup methods and showed that they can cause training loss distortion and find a suboptimal decision boundary while having a small sample complexity. One interesting future direction is analyzing when and how state-of-the-art masking-based Mixup works, by considering more sophisticated data distributions capturing image data and complicated models such as neural networks.
\vspace{-5pt}
\section*{Acknowledgements}
\vspace{-5pt}
This paper was supported by Institute of Information \& communications Technology Planning \& Evaluation (IITP) grant (No.2019-0-00075, Artificial Intelligence Graduate School Program (KAIST)) funded by the Korea government (MSIT), two National Research Foundation of Korea (NRF) grants (No. NRF-2019R1A5A1028324, RS-2023-00211352) funded by the Korea government (MSIT), and a grant funded by Samsung Electronics Co., Ltd.

\bibliography{Reference}
\bibliographystyle{icml2023}

\newpage
\onecolumn

\appendix

\section{Proofs for Section \ref{section:ERM}}
\subsection{Proof of Theorem~\ref{thm:expected_loss_ERM}}\label{proof:expected_loss_ERM}
We prove the existence and uniqueness of a minimizer of $\mathbb{E}_\kappa [\loss (\vw)]$ first. Next, we characterize a direction of a unique minimizer of $\mathbb{E}_\kappa [\loss (\vw)]$.
\paragraph{Step 1: Existence and Uniqueness of a Minimizer of $\mathbb{E}_\kappa[\loss(\vw)]$}\quad

From convexity of $l(\cdot)$, for any $z \in \mathbb{R}$, $l(z) \geq -\frac{1}{2}z + \log 2 \geq -\frac{1}{2}z$ since $z \mapsto -\frac{1}{2} z + \log 2$ is tangent line of the graph of $l(z)$ at $z=0$. Thus, we have 
\begin{equation*}
    \mathbb{E}_\kappa[\loss(\vw)] = \mathbb{E}_{\rvx\sim N(\vmu,\kappa^{-1}\mSigma)}[l(\vw^\top \rvx)] \geq \mathbb{E}_{\rvx\sim N(\vmu,\kappa^{-1}\mSigma)}[l(\vw^\top \rvx) \cdot \vone_{\vw^\top \rvx < 0} ]\geq \mathbb{E}_{\rvx \sim N(\vmu, \kappa^{-1}\mSigma)}\left[-\frac{1}{2}\vw^\top \rvx \cdot \vone_{\vw^\top \rvx < 0}\right].
\end{equation*}
\begin{claim}
\label{claim:1}
For any $\kappa \in (0, \infty)$, a mapping $\vw \mapsto \mathbb{E}_{\rvx \sim N(\vmu, \kappa^{-1}\mSigma)} [\vw^\top \rvx \cdot \vone_{\vw^\top \rvx<0}]$ is continuous.
\end{claim}
\begin{proof}[Proof of Claim A.1]
For any $\epsilon >0$, let $\delta_\epsilon:= \left(\kappa^{-1} \lVert \mSigma \rVert + \lVert \vmu \rVert^2 \right)^{-1/2} \cdot \epsilon /2 >0$. Then, for any $\vDelta \in \mathbb{R}^d$ with $\lVert \vDelta \rVert \leq \delta_\epsilon$ and $\vw \in \mathbb{R}^d$, we are going to show that 
\begin{equation*}
\left | \mathbb{E}_{\rvx \sim N(\vmu, \kappa^{-1} \mSigma)} [(\vw+\vDelta)^\top \rvx \cdot \vone_{(\vw+\vDelta)^\top \rvx<0}]-\mathbb{E}_{\rvx \sim N(\vmu, \kappa^{-1} \mSigma)} [\vw^\top \rvx \cdot \vone_{\vw^\top \rvx<0}] \right | \leq \epsilon,
\end{equation*}
to conclude that the mapping $\vw \mapsto \mathbb{E}_{\rvx \sim N(\vmu, \kappa^{-1} \mSigma)} [\vw^\top \rvx \cdot \vone_{\vw^\top \rvx<0}]$ is continuous.

To this end, we start by
\begin{align*}
    &\quad \left | \mathbb{E}_{\rvx \sim N(\vmu, \kappa^{-1} \mSigma)} [(\vw+\vDelta)^\top \rvx \cdot \vone_{(\vw+\vDelta)^\top \rvx<0}]-\mathbb{E}_{\rvx \sim N(\vmu, \kappa^{-1}\mSigma)} [\vw^\top \rvx \cdot \vone_{\vw^\top \rvx<0}]\right|\\
    &\leq \mathbb{E}_{\rvx \sim N(\vmu, \kappa^{-1} \mSigma)}\left[ \left| (\vw+\vDelta)^\top \rvx \cdot \vone_{(\vw+\vDelta)^\top \rvx<0}- \vw^\top \rvx \cdot \vone_{\vw^\top \rvx<0}\right| \right]\\
    &=  \mathbb{E}_{\rvx \sim N(\vmu, \kappa^{-1} \mSigma)}\left[ \left| \vDelta ^\top \rvx \cdot \vone_{(\vw+\vDelta)^\top \rvx <0} + \vw^\top \rvx \cdot (\vone_{(\vw+\vDelta)^\top \rvx <0} - \vone_{\vw^\top \rvx <0})\right| \right]\\
    &\leq \mathbb{E}_{\rvx \sim N(\vmu, \kappa^{-1} \mSigma)}\left[ \left| \vDelta ^\top \rvx \cdot \vone_{(\vw+\vDelta)^\top \rvx <0} \right| \right]+\mathbb{E}_{\rvx \sim N(\vmu, \kappa^{-1} \mSigma)}\left[ \left| \vw^\top \rvx \cdot (\vone_{(\vw+\vDelta)^\top \rvx <0} - \vone_{\vw^\top \rvx <0})\right| \right].
\end{align*}
It is clear that $\mathbb{E}_{\rvx \sim N(\vmu, \kappa^{-1} \mSigma)}\left[ \left| \vDelta ^\top \rvx \cdot \vone_{(\vw+\vDelta)^\top \rvx <0} \right| \right] \leq \mathbb{E}_{\rvx \sim N(\vmu, \kappa^{-1} \mSigma)}\left[ \left| \vDelta ^\top \rvx\right| \right]$. Also, for each $\vx \in \mathbb{R}^d$,
\begin{align*}
    \vw^\top \vx \cdot (\vone_{(\vw+\vDelta)^\top \vx <0} - \vone_{\vw^\top \vx <0}) = 
    \begin{cases}
    \vw^\top \vx \leq -\vDelta^\top \vx = |\vDelta ^\top \vx| &\text{if} (\vw+\vDelta)^\top \vx <0, \vw^\top \vx \geq 0,\\
    -\vw^\top \vx \leq \vDelta^\top \vx = |\vDelta ^\top \vx| &\text{if} (\vw+\vDelta)^\top \vx \geq0, \vw^\top \vx< 0,\\
    0 \leq |\vDelta ^\top \vx| &\text{otherwise}.
    \end{cases}
\end{align*}
Therefore, $\mathbb{E}_{\rvx \sim N(\vmu, \kappa^{-1} \mSigma)}\left[ \left| \vw^\top \rvx \cdot (\vone_{(\vw+\vDelta)^\top \rvx <0} - \vone_{\vw^\top \rvx <0})\right| \right] \leq \mathbb{E}_{\rvx \sim N(\vmu, \kappa^{-1} \mSigma)}\left[ \left| \vDelta ^\top \rvx\right| \right]$. Also, by Jensen's inequality,
\begin{align*}
    \mathbb{E}_{\rvx \sim N(\vmu, \kappa^{-1} \mSigma)}\left[ \left| \vDelta ^\top \rvx \right| \right] &\leq \mathbb{E}_{\rvx \sim N(\vmu, \kappa^{-1} \mSigma)}\left[ (\vDelta ^\top \rvx)^2 \right]^{1/2}\\
    &= \left( \kappa^{-1} \vDelta ^\top \mSigma \vDelta +(\vDelta^\top \vmu)^2 \right)^{1/2}\\
    &\leq \left(\kappa^{-1} \lVert \mSigma \rVert \lVert \vDelta \rVert^2 + \lVert \vmu\rVert^2\lVert \vDelta \rVert^2 \right)^{1/2}\\
    &\leq \left(\kappa^{-1} \lVert \mSigma \rVert + \lVert \vmu\rVert^2\right)^{1/2} \cdot \delta_\epsilon\\
    &= \epsilon/2.
\end{align*}
Hence, we have $\left | \mathbb{E}_{\rvx \sim N(\vmu, \kappa^{-1} \mSigma)} [(\vw+\vDelta)^\top \rvx \cdot \vone_{(\vw+\vDelta)^\top \rvx<0}]-\mathbb{E}_{\rvx \sim N(\vmu, \kappa^{-1} \mSigma)} [\vw^\top \rvx \cdot \vone_{\vw^\top \rvx<0}] \right| \leq \epsilon$, as desired. \hfill $\square$

From Claim~\ref{claim:1} and compactness of the unit sphere $\{ \vw\in \mathbb{R}^d : \lVert \vw \rVert = 1 \} \subset \mathbb{R}^d$, it follows that for any given $\kappa \in (0,\infty)$, a mapping $\vw \mapsto \mathbb{E}_{\rvx \sim N(\vmu, \kappa^{-1} \mSigma)} [\vw^\top \rvx \cdot \mathbf{1}_{\vw^\top \rvx<0}]$ has the maximum value (over the unit sphere) $-m_\kappa$ with $m_\kappa >0$. For any $\vw$ satisfying $\lVert \vw \rVert > 2 m_\kappa^{-1}$, we have
\begin{align}\label{eqn:ball}
    \mathbb{E}_\kappa[\loss (\vw)] &\geq \mathbb{E}_{\rvx \sim N(\vmu, \kappa^{-1} \mSigma)}\left[- \frac{1}{2} \vw^\top \rvx \cdot \vone_{\vw^\top \rvx<0} \right] \nonumber \\
    &= \lVert \vw \rVert \cdot \mathbb{E}_{\rvx \sim N(\vmu, \kappa^{-1}\mSigma)}\left[- \frac{1}{2} \cdot \frac{\vw^\top \rvx}{\lVert \vw \rVert}\cdot \vone_{\frac{\vw^\top \rvx}{\lVert \vw \rVert}<0} \right]\\
    &\geq \frac{m_\kappa \lVert \vw \rVert}{2} 
    > 1 > \log 2 \nonumber\\
    &= \mathbb{E}_\kappa[\loss(\vzero)] \nonumber.
\end{align}
Therefore, a minimizer of $\mathbb{E}_\kappa [ \loss (\vw) ]$ has to be necessarily contained in a compact set $\{ \vw \in \mathbb{R}^d :  \lVert \vw \rVert \leq 2  m_\kappa^{-1} \}$. Since $\mathbb{E}_\kappa [ \loss (\vw) ]$ is a continuous function of $\vw$, there must exist a minimizer. The existence part is hence proved.

To show uniqueness, we will prove strict convexity of $\mathbb{E}_\kappa[\loss (\vw)]$. From strict convexity of $l(\cdot)$, for any $t \in [0,1]$, $\vw_1, \vw_2 \in \mathbb{R}^d$ with $\vw_1 \neq \vw_2$ and $y \in \{0,1\}$, we have
\begin{align*}
    &\quad t[yl(\vw_1^\top \vx) + (1-y) l(-\vw_1^\top \vx)] + (1-t) [yl(\vw_2^\top \vx) + (1-y) l(-\vw_2^\top \vx)]\\
    &> y l((t\vw_1+(1-t)\vw_2)^\top \vx) + (1-y) l(-(t\vw_1+(1-t)\vw_2)^\top \vx),
\end{align*}
and any $\vx \in \mathbb{R}^d$ except for a Lebesgue measure zero set (i.e., the set of points $\vx \in \R^d$ satisfying  $\vw_1^\top \vx = \vw_2^\top \vx$).

By taking expectation, we have for any $t \in [0,1]$, and $\vw_1, \vw_2 \in \mathbb{R}^d$,
\begin{equation*}
    t \mathbb{E}_\kappa[\loss(\vw_1)] + (1-t) \mathbb{E}_\kappa[\loss(\vw_2)] >  \mathbb{E}_\kappa[\loss(t\vw_1+(1-t)\vw_2)].
\end{equation*}
Therefore, $\mathbb{E}_\kappa[\loss(\vw)]$ is strictly convex.
Since a strictly convex function has at most one minimizer, we conclude that $\mathbb{E}_\kappa[\loss(\vw)]$ has a unique minimizer $\w$ for any given $\kappa \in (0,\infty)$. 
\end{proof}

\paragraph{Step 2: Direction of a Unique Minimizer of $\mathbb{E}_\kappa[\loss(\vw)]$} \quad

We rewrite $\mathbb{E}_\kappa[\loss (\vw) ] $ as
\begin{align}\label{eqn:ERM_rewrite}
    \mathbb{E}_\kappa[\loss (\vw) ] &= \mathbb{E}_{\rvx \sim N(\vmu, \kappa^{-1} \mSigma)}\left[l\left( \vw^\top \rvx \right)\right]\nonumber \\
    &= \mathbb{E}_{X \sim N(\vw^\top \vmu, \kappa^{-1} \vw^\top \mSigma \vw )}\left[l(X)\right] \nonumber\\
    &= \mathbb{E}_{Z\sim N(0, 1 )}\left[l\left( \kappa^{-1/2} \left(\vw ^\top \mSigma \vw \right)^{1/2} Z + \vw^\top \vmu \right)\right].
\end{align}
We recall two lemmas we described in our proof sketch.
\logistic*
\opt*
Let $C_\kappa := \w^\top \vmu$. By \Eqref{eqn:ERM_rewrite} and Lemma~\ref{lemma:logistic}, $\w$ is a solution for the problem $\min_{\vw^\top \vmu = C_\kappa}\frac{1}{2} \vw^\top \mSigma \vw$.
 Hence, Lemma~\ref{lemma:opt} implies that there exists $c_\kappa \in \mathbb{R}$ such that $\w:= c_\kappa \mSigma^{-1}\vmu$. The only remaining part is showing $c_\kappa>0$. If $c_\kappa <0$, we have 
\begin{align*}
    \mathbb{E}_\kappa[\loss(\w)] &= \mathbb{E}_{Z \sim N(0,1)} \left [l\left (\kappa^{-1/2}(\w^\top \mSigma \w)^{1/2} Z + c_\kappa \vmu^\top \mSigma^{-1} \vmu \right ) \right ]\\
    &> \mathbb{E}_{Z \sim N(0,1)} \left [l \left (\kappa^{-1/2}(-\w)^\top  \mSigma (-\w))^{1/2} Z - c_\kappa \vmu^\top \mSigma^{-1} \vmu) \right ) \right ]\\
    &= \mathbb{E}\left [\loss \left ( -\w  \right ) \right ],
\end{align*}
where the inequality holds because $l(\cdot)$ is strictly decreasing and $c_\kappa \vmu^\top \mSigma^{-1} \vmu < 0$.
It is contradictory to $\w$ being a unique minimizer of $\mathbb{E}_\kappa[\loss(\vw)]$, so we conclude $c_\kappa \geq 0$.  Showing $c_\kappa$ is strictly positive will be handled in the proof of Lemma~\ref{lemma:ERM_norm} which can be found in Appendix~\ref{proof:lemma:ERM_norm}. \hfill $\square$

\subsection{Proof of Lemma~\ref{lemma:logistic}}
Define a function $f:(0,\infty) \rightarrow \mathbb{R}$ as $f(\sigma) := \mathbb{E}_{Z \sim N(0,1)}[l(m+\sigma Z)]$. It suffices to show that $f$ is strictly increasing. For each $\sigma \in (0,\infty)$ and $z \in \mathbb{R}$, $\left | \frac{\partial}{\partial \sigma} l(m+\sigma z) \right| = \left| l'(m+\sigma z) z \right|  = \left| \frac{z}{1+e^{m+\sigma z}} \right|\leq |z|$ and $\mathbb{E}_{Z \sim N(0,1)}[|Z|]< \infty$. Thus, by Lemma~\ref{lemma:leibniz}, 
\begin{align*}
    \frac{d}{d\sigma}f(\sigma) &= \mathbb{E}_{Z \sim N(0,1)}\left[\frac{\partial}{\partial \sigma} l(m+\sigma Z)\right] = \mathbb{E}_{Z \sim N(0,1)}[l'(m+\sigma Z) Z]\\
    &= - \frac{1}{\sqrt{2 \pi}} \int_{-\infty}^\infty \frac{z}{1+e^{m+\sigma z}}e^{-z^2/2}dz\\
     &= -\frac{1}{\sqrt{2 \pi}} \left( \int_0^\infty \frac{z}{1+e^{m+\sigma z}}e^{-z^2/2}dz + \int_{-\infty}^0 \frac{z}{1+e^{m-\sigma z}}e^{-z^2/2}dz \right)\\
    &= -\frac{1}{\sqrt{2 \pi}} \left( \int_0^\infty \frac{z}{1+e^{m+\sigma z}}e^{-z^2/2}dz - \int_0^\infty \frac{z}{1+e^{m-\sigma z}}e^{-z^2/2}dz \right)\\
    &= -\frac{1}{\sqrt{2 \pi}}  \int_0^\infty \left( \frac{z}{1+e^{m+\sigma z}}-\frac{z}{1+e^{m-\sigma z}} \right)e^{-z^2/2}dz\\
    &>0.
\end{align*}
The last inequality holds since $\frac{z}{1+e^{m+\sigma z}} < \frac{z}{1+e^{m-\sigma z}}$ for each $z>0$. Therefore, $f$ is strictly increasing as we desired. \hfill $\square$

\subsection{Proof of Lemma~\ref{lemma:opt}}
Consider a function $f : \mathbb{R}^d \rightarrow \mathbb{R}$ defined as $f(\vw) = \frac{1}{2} \vw^\top \mSigma \vw$. Since $\nabla_\vw^2 f(\vw) = \mSigma$ is positive definite, $f$ is strictly convex. The strict convexity continues to hold even when we restrict the domain to $\{\vw \mid \vw^\top \vmu = C\}$, so $\min_{\vw \in \mathbb{R}^d, \vw^\top \vmu = C} f(\vw)$ has at most one minimizer. Let $\bar{\vw} = \frac{C}{ \vmu^\top \mSigma^{-1} \vmu} \mSigma^{-1} \vmu$. Then, for any $\vw\in \mathbb{R}^d$ such that $\vw^\top \vmu = C$, we have 
\begin{equation*}
    f(\vw) - f(\bar{\vw}) \geq \nabla f(\bar{\vw})^\top(\vw - \bar{\vw})= \left(\mSigma \bar{\vw}\right)^\top (\vw - \bar{\vw}) = \frac{C}{\vmu^\top \mSigma^{-1} \vmu}\vmu^\top \left (\vw - \frac{C}{\vmu^\top \mSigma^{-1} \vmu}\mSigma^{-1} \vmu \right) = 0.
\end{equation*}
Therefore, $\bar{\vw} = \frac{C}{\vmu^\top \mSigma^{-1} \vmu} \mSigma^{-1} \vmu$, a rescaling of $\mSigma^{-1} \vmu$, is the unique minimizer.\hfill $\square$
\subsection{Proof of Theorem~\ref{thm:ERM_sufficient}} \label{proof:ERM_sufficient}
Since we consider sufficiently large $\kappa$, we may assume $n \geq d$
and let $R_\kappa:=2 \lVert \w \rVert$. By Lemma~\ref{lemma:ERM_norm}, we know that $R_\kappa = \Theta(\kappa)$. Next, define a compact set $\mathcal{C}_\kappa := \{ \vw \in \mathbb{R}^d : \lVert \vw \rVert \leq R_\kappa\}$, which trivially contains $\w$. For any $\vw \in \mathbb{R}^d$ and nonzero $\vv \in \mathbb{R}^d$, we have
\begin{equation*}
    \vv^\top \nabla^2_\vw \loss (\vw) \vv = \frac{1}{n} \sum_{i=1}^n (y_i l''(\vw^\top \vx_i) (\vv^\top \vx_i)^2 + (1-y_i) l''(-\vw^\top \vx_i) (\vv^\top \vx_i)^2 ) >0,
\end{equation*}
almost surely, since $\{\vx_i\}_{i \in [n] }$ spans $\mathbb{R}^d$ almost surely.  Therefore, $\loss(\vw)$ is strictly convex on $\mathbb{R}^d$ almost surely and we conclude that  $\loss (\vw)$ has a unique minimizer $\hat{\vw}_{S}^*$ on $\mathcal{C}_\kappa$ almost surely. Note that, if $\hat{\vw}_S^*$ belongs to interior of $\mathcal{C}$, it is a unique minimizer of $\loss (\vw)$ over the entire $\mathbb{R}^d$. We prove high probability convergence of $\hat{\vw}_{S}^*$ to $\w$ using Lemma~\ref{lemma:minimizer_independent} and convert $\ell_2$ convergence into directional convergence. 
For simplicity, we define
\begin{equation*}
f_i(\vw) := y_il(\vw^\top \vx_i) + (1-y_i) l(\vw^\top \vx_i),
\end{equation*}
for each $i \in [n]$.
We start with the following claim which is useful for estimating quantities described in assumptions of Lemma~\ref{lemma:minimizer_independent} for our setting.
\begin{claim}\label{claim:ERM}
For any $t>0$, we have
\begin{equation*}
    \mathbb{E}_\kappa \left[ e^{t\lVert \vx _i \rVert}\right] \leq  \left( 2^{d/2} + e^{4\kappa^{-1}t^2 \lVert \mSigma \rVert}\right)e^{t\lVert \vmu \rVert},
\end{equation*}
for all $i \in [n] $.
\end{claim}
\begin{proof}[Proof of Claim~\ref{claim:ERM}]
By applying triangular inequality and Lemma~\ref{lemma:ineq4}, we have
\begin{align*}
\mathbb{E}_\kappa \left[ e^{t \left \lVert \vx _i\right \rVert} \right] &= \mathbb{E}_{\rvx \sim N(\vmu, \kappa^{-1} \mSigma) } \left[ e^{t \left \lVert \rvx\right \rVert} \right]\\
&\leq \mathbb{E}_{\rvx \sim N(\vmu, \kappa^{-1} \mSigma) } \left[ e^{t \left \lVert \rvx - \vmu \right \rVert} \right] e^{t\lVert \vmu \rVert}\\
&= \mathbb{E}_{\rvz \sim N(\vzero, \kappa^{-1}t^2 \mSigma )} \left[e^{\lVert \rvz \rVert} \right] e^{t\lVert \vmu \rVert} \\
&\leq \left( 2^{d/2} + e^{4\kappa^{-1}t^2 \lVert \mSigma \rVert}\right)e^{t\lVert \vmu \rVert},
\end{align*}
for each $i \in [n]$.
\end{proof}

\paragraph{Step 1: Estimate Upper Bound of $\mathbb{E}_\kappa \left[ e^{\left|f_i(\vw) - \mathbb{E}_\kappa[f_i(\vw)]\right|}\right]$ on $\mathcal{C}_\kappa$}\quad

For any $\vw \in \mathcal{C}_\kappa$ and $i \in [n]$,
\begin{equation*}
    |f_i(\vw)| = |y_il(\vw^\top \vx_i) + (1-y_i)l(-\vw^\top \vx_i)| \leq l(-|\vw^\top \vx_i|) \leq l(-R_\kappa \lVert \vx_i \rVert).
\end{equation*}
Hence, we have $\mathbb{E}_\kappa \left[ e^{|f_i(\vw)|}\right] \leq  \mathbb{E}_\kappa \left[ 1+ e^{R_\kappa\lVert \vx_i \rVert} \right]$. By applying Claim~\ref{claim:ERM} for $t= R_\kappa$, there exists $M_\kappa'$ such that  $\mathbb{E}_\kappa \left[ e^{|f_i(\vw)|}\right] \leq  M_\kappa'$ and $M_\kappa' = \exp(\Theta(\kappa))$ since $R_\kappa = \Theta(\kappa)$ by Lemma~\ref{lemma:ERM_norm}.
By triangular inequality and Jensen's inequality, we have
\begin{align*}
\mathbb{E}_\kappa \left[ e^{\left|f_i(\vw) - \mathbb{E}_\kappa \left[f_i(\vw)\right]\right|}\right] 
&\leq \mathbb{E}_\kappa \left[ e^{\left|f_i(\vw)\right| + \left|\mathbb{E}_\kappa \left [f_i(\vw)\right]\right|}\right]\\
&\leq \mathbb{E}_\kappa \left[ e^{\left|f_i(\vw)\right| }\right]^2\\
&\leq {M_\kappa'}^2.
\end{align*}
Defining $M_\kappa := {M_\kappa'}^2$, it follows that $M_\kappa = \exp(\Theta(\kappa))$ and $\mathbb{E}_\kappa \left[ e^{\left|f_i(\vw) - \mathbb{E}_\kappa[f_i(\vw)]\right| }\right] \leq M_\kappa$ for any $\vw \in \mathcal{C}_\kappa$.

\paragraph{Step 2: Estimate Upper Bound of $\lVert \nabla_\vw f_i(\vw) \rVert$ and $\lVert \nabla_\vw \mathbb{E}_\kappa[ f_i(\vw) ]\rVert$} \quad

For each $\vw \in \mathcal{C}_\kappa$ and $i \in [n]$, 
\begin{equation*}
\lVert \nabla_\vw f_i(\vw) \rVert = \left \lVert \nabla_\vw (y_il(\vw^\top \vx_i) + (1-y_i) l(-\vw^\top \vx_i)) \right \rVert 
= \left \lVert y_il'(\vw^\top \vx_i)\vx_i - (1-y_i) l'(-\vw^\top \vx_i)\vx_i \right \rVert 
\leq  \lVert \vx_i \rVert.
\end{equation*}
The last inequality holds since $0<l'(z) <1$ for any $z\in \R$. In addition, by Lemma~\ref{lemma:gradient&hessian},
\begin{align*}
\left\lVert \nabla_\vw \mathbb{E}_\kappa [f_i(\vw)]\right \rVert &=  \left \lVert \nabla_\vw \mathbb{E}_\kappa [(y_il(\vw^\top \vx_i) + (1-y_i) l(-\vw^\top \vx_i)) ]\right \rVert\\
&= \left \lVert  \mathbb{E}_\kappa [\nabla_\vw(y_il(\vw^\top \vx_i) + (1-y_i) l(-\vw^\top \vx_i)) ]\right \rVert\\
&= \lVert \mathbb{E}_\kappa [(y_il'(\vw^\top \vx_i) + (1-y_i) l'(-\vw^\top \vx_i))\vx_i ]\rVert\\
&\leq \mathbb{E}_\kappa [\lVert \vx_i \rVert].
\end{align*}
Also, applying Claim~\ref{claim:ERM} with $t = 1$, there exists $L_\kappa$ such that $\mathbb{E}_\kappa \left[ e^{\lVert \vx_i \rVert}\right]< L_\kappa$ and $L_\kappa = \Theta(1)$.

\paragraph{Step 3: Estimate Strong Convexity Constant of $\mathbb{E}_\kappa[\loss (\vw)]$ on $\mathcal{C}_\kappa$} \quad

By Lemma~\ref{lemma:gradient&hessian} and Lemma~\ref{lemma:ineq1}, for any $\vw \in \mathcal{C}_\kappa$ and unit vector $\vv \in \mathbb{R}^{d}$, we have
\begin{equation*}
    \vv^\top \nabla^2_\vw \mathbb{E}_\kappa[\loss(\vw)] \vv = \mathbb{E}_{\rvx \sim N(\vmu, \kappa^{-1} \mSigma)}[l''(\vw^\top \rvx) (\vv^\top \rvx)^2]
    \geq \frac{1}{4} \mathbb{E}_{\rvx \sim N(\vmu, \kappa^{-1}\mSigma)}[e^{-(\vw^\top \rvx)^2/2}(\vv^\top \rvx)^2].
\end{equation*}
By Lemma~\ref{lemma:ineq3}, 
\begin{align*}
    &\quad \mathbb{E}_{\rvx \sim N(\vmu, \kappa^{-1}\mSigma)}[e^{-(\vw^\top \rvx)^2/2}(\vv^\top \rvx)^2]\\
    &\geq \kappa^{d/2}\lVert \mSigma \rVert ^{-d/2} \left (\kappa \lVert\mSigma^{-1}\rVert + \lVert \vw\rVert^2 \right )^{-(d+2)/2} \exp \left (-\lVert \vw\rVert^2 \lVert \mSigma \rVert \lVert \mSigma^{-1} \rVert \lVert \vmu\rVert^2 \right )\\
    &\geq \kappa^{d/2}\lVert \mSigma \rVert ^{-d/2} \left (\kappa \lVert \mSigma^{-1} \rVert + R_\kappa^2 \right )^{-(d+2)/2} \exp \left (-R_\kappa^2 \lVert \mSigma \rVert \lVert \mSigma^{-1} \rVert \lVert \vmu\rVert^2 \right ),
\end{align*}
and substituting $R_\kappa = \Theta (\kappa)$ to the RHS of the inequality above gives
\begin{equation*}
\kappa^{d/2}\lVert \mSigma \rVert ^{-d/2} \left (\kappa \lVert \mSigma^{-1} \rVert + R_\kappa^2 \right )^{-(d+2)/2} \exp \left (-R_\kappa^2 \lVert \mSigma \rVert \lVert \mSigma^{-1} \rVert \lVert \vmu\rVert^2 \right ) 
= \frac{1}{\exp(\Theta(\kappa^2))}.
\end{equation*}
Since this hold for any unit vector $\vv \in \R^d$, $\mathbb{E}_\kappa [\loss(\vw)]$ is $\alpha_\kappa$-strongly convex with $\alpha_\kappa = \frac{1}{\exp(\Theta(\kappa^2))}$ on $\mathcal{C}_\kappa$.

\paragraph{Step 4: Sample Complexity for Directional Convergence}\quad

Since $\lVert \w \rVert = \Theta(\kappa)$ and $\alpha_\kappa = \frac{1}{\exp ( \Theta(\kappa^2))}$, we assume $\kappa$ is large enough so that $\alpha_\kappa \lVert \w \rVert^2 <1$, which is quite easy to satisfy given the rate of decay in $\alpha_\kappa$. Assume the unique existence of $\hat{\vw}_S^*$ which occurs almost surely. By Lemma~\ref{lemma:minimizer_independent}, for each $0<\epsilon<1$, if 
\begin{equation*}
n \geq \frac{C_1M_\kappa}{\alpha_\kappa ^2 \lVert \w \rVert^4 \epsilon^4} \log \left(\frac{3}{\delta} \max \left\{1, \left(\frac{2 C_2d^{1/2}R_\kappa L_\kappa}{\alpha_\kappa \lVert \w\rVert^2 \epsilon^2}\right)^d \right\}   \right) = \frac{\exp(\Theta(\kappa^2)) }{\epsilon^4} \left ( 1+ \log\frac{1}{\epsilon} + \log\frac{1}{\delta} \right ),
\end{equation*}
then we have $\lVert \w- \hat{\vw}_S^* \rVert \leq  \lVert \w\rVert \epsilon$ with probability at least $1-\delta$. Also, if $\lVert \w- \hat{\vw}_S^* \rVert \leq  \lVert \w\rVert \epsilon$, then $\hat{\vw}_S^*$ belongs to interior of $\mathcal{C}_\kappa$. Hence, $\hat{\vw}_S^*$ is a minimizer of $\loss(\vw)$ over the entire $\mathbb{R}^d$. Also, we have
\begin{align*}
    \cosim \left ( \hat{\vw}_S^*, \mSigma^{-1}\vmu \right) 
    &=  \cosim \left( \hat{\vw}_S^*,\w \right)
    = \bigg (1- \sin^2  \Big(\angle\big(\hat{\vw}_S^*, \w\big) \Big ) \bigg )^{1/2}\\
    &\geq 1- \sin \Big (\angle\big(\hat{\vw}_S^*, \w\big) \Big ) = 1-\frac{\lVert \w - \hat{\vw}_S^* \rVert}{\lVert \w \rVert}\\
    &\geq 1- \epsilon.
\end{align*}
Hence, we conclude that if $n = \frac{\exp(\Omega(\kappa^2)) }{\epsilon^4}\left(1+ \log\frac{1}{\epsilon} + \log\frac{1}{\delta}\right)$, then
with probability at least $1-\delta$, the ERM loss $\loss(\vw)$ has a unique minimizer $\hat{\vw}_S^*$ and $\cosim ( \hat{\vw}_S^*, \mSigma^{-1}\vmu)\geq 1- \epsilon$. \hfill $\square$

\subsection{Proof of Lemma~\ref{lemma:minimizer_independent}}
By Lemma~\ref{lemma:independent_concentration}, there exists a universal constant $C_1>0$ such that for any fixed $\vtheta \in \mathcal{C}$ independent of the draws of $\veta_i$'s,
\begin{align}
\label{eq:lemmaminimizer_independent-2}
    &\quad \mathbb{P}\left[ \left| \hat{F}_N(\vtheta) - F(\vtheta)\right| > \frac{\alpha \epsilon^2}{8} \right]\nonumber \\
    &= \mathbb{P}\left[ \frac{1}{N} \sum_{i=1}^N f
(\vtheta, \veta_i)- \mathbb{E}_{\veta \sim \mathcal{P}}[f(\vtheta,\veta)] > \frac{\alpha \epsilon^2}{8} \right] +  \mathbb{P}\left[ \frac{1}{N} \sum_{i=1}^N f
(\vtheta, \veta_i)- \mathbb{E}_{\veta \sim \mathcal{P}}[f(\vtheta,\veta)] < -\frac{\alpha \epsilon^2}{8} \right]\nonumber \\
    &\leq 2\exp \left(-\frac{\alpha^2 \epsilon^4}{C_1 M}N\right).
\end{align}

Notice that from the given condition and Jensen's inequality, 
\begin{equation*}
    \mathbb{E}_{\veta \sim \mathcal{P}}[g(\veta)] \leq \log \mathbb{E}_{\veta \sim \mathcal{P}}\left[e^{g(\veta)}\right] < \log L <L,
\end{equation*}
and from triangular inequality, we have
\begin{align}
\label{eq:lemmaminimizer_independent-1}
 \mathbb{P}\left[ \sup_{\vtheta \in \mathcal{C}}\left\lVert \nabla_\vtheta \left( \hat{F}_N(\vtheta) - F(\vtheta) \right)\right\rVert > 3 L \right] 
 &\leq \mathbb{P}\left[ \frac{1}{N} \sum_{i=1}^N \sup_{\vtheta \in \mathcal{C}}\left\lVert \nabla_\vtheta \left( f(\vtheta, \veta_i) - \mathbb{E}[f(\vtheta, \veta_i)] \right)\right\rVert > 3 L\right]\nonumber \\
 &\leq \mathbb{P} \left[ \frac{1}{N} \sum_{i=1}^N \left( g(\veta_i) + \mathbb{E}[g(\veta_i)] \right)> 3 L \right] \nonumber \\
 &\leq \mathbb{P} \left[ \frac{1}{N} \sum_{i=1}^N \left( g(\veta_i) - \mathbb{E}[g(\veta_i)] \right)> L \right].
\end{align}
Since we have $\mathbb{E}_{\veta \sim \mathcal{P}}\left [e^{ |g(\veta) - \mathbb{E}_{\veta\sim \mathcal{
P}}[g(\veta)]|} \right] \leq \mathbb{E}_{\veta \sim \mathcal{P}}\left [e^{ g(\veta)} \right] \cdot e^{\mathbb{E}_{\veta\sim \mathcal{
P}}[g(\veta)] }<L^2$ by our assumption, we can apply Lemma~\ref{lemma:independent_concentration} to the RHS of \Eqref{eq:lemmaminimizer_independent-1}.

Therefore, we have
\begin{equation}
\label{eq:lemmaminimizer_independent-3}
 \mathbb{P}\left[ \sup_{\vtheta \in \mathcal{C}}\left\lVert \nabla_\vtheta \left( \hat{F}_N(\vtheta) - F(\vtheta) \right)\right\rVert > 3 L \right] 
 \leq \exp \left(-\frac{1}{C_1'}N\right),
\end{equation}
where $C_1'>0$ is a universal constant. Without loss of generality, we can choose $C_1 = C_1'$ that works for both \Eqref{eq:lemmaminimizer_independent-2} and \Eqref{eq:lemmaminimizer_independent-3}.

We choose $\bar{\vtheta}_1, \dots, \bar{\vtheta}_m \in \mathcal{C}$ with $m \leq \max \left\{ 1, \left(\frac{C_2 k^{1/2} D  L}{\alpha \epsilon^2 }\right)^k \right\}$ where $C_2$ is a universal constant and satisfies following: 
For any $\vtheta \in \mathcal{C}$, there exists $i_\vtheta \in [m]$ such that $\left \lVert \vtheta - \bar{\vtheta}_{i_\vtheta} \right \rVert<  \frac{\alpha \epsilon^2}{24  L}$. In other words, $\bar{\vtheta}_1, \dots, \bar{\vtheta}_m$ form an $\frac{\alpha \epsilon^2}{24L}$-cover of $\mathcal {C}$.

Suppose $| \hat{F}_N(\bar{\vtheta}_k) - F(\bar{\vtheta}_k) | < \frac{\alpha \epsilon^2}{8}$  for each $k \in [m]$ and $\sup_{\theta \in \mathcal{C}}\left\lVert \nabla_\vtheta \left( \hat{F}_N(\vtheta) - F(\vtheta) \right)\right\rVert < 3 L$ which implies $\hat{F}_N(\vtheta)-F(\vtheta)$ is $ 3 L$-Lipschitz. Then, for any $\vtheta \in \mathcal{C}$, we have
\begin{align*}
    &\quad \left | \hat{F}_N(\vtheta) - F(\vtheta) \right|\\
    &\leq \left | \hat{F}_N(\bar{\vtheta}_{i_\vtheta}) - F(\bar{\vtheta}_{i_\vtheta}) \right| +  \left |\left( \hat{F}_N(\vtheta) - F(\vtheta) \right) -  \left  (\hat{F}_N(\bar{\vtheta}_{i_\vtheta}) - F(\bar{\vtheta}_{i_\vtheta}) \right) \right|\\
    &\leq \frac{\alpha \epsilon^2}{8} + 3 L \lVert \vtheta - \bar{\vtheta}_{i_\vtheta} \rVert < \frac{\alpha \epsilon^2}{8} + \frac{\alpha \epsilon^2}{8} \\&= \frac{\alpha \epsilon^2}{4}.
\end{align*}
By applying union bound, we conclude
\begin{align*}
    &\quad \mathbb{P}\left[ \sup_{\vtheta \in \mathcal{C}}|\hat{F}_N(\vtheta) - F(\vtheta)| >  \frac{\alpha \epsilon^2}{4} \right] \\&\leq 2 \max \left \{1, \left(\frac{C_2k^{1/2} D  L}{\alpha \epsilon^2 }\right)^k \right\} \exp\left(- \frac{\alpha^2 \epsilon^4}{C_1 M}N \right) + \exp\left(-\frac{1}{C_1}N \right)\\ &\leq 3 \max \left \{1, \left(\frac{C_2 k^{1/2} D L}{\alpha \epsilon^2 }\right)^k \right\} \exp \left( - \min \left\{ \frac{\alpha^2 \epsilon^4}{C_1 M}, \frac{1}{C_1}\right \}N \right)\\
    & =  3\max \left \{1, \left(\frac{C_2 k^{1/2} D  L}{\alpha \epsilon^2 }\right)^k \right\} \exp \left( - \frac{\alpha^2 \epsilon^4}{C_1 M}N \right),
\end{align*}
where the last equality is due to $\alpha \epsilon^2 \leq 1$ and $M\geq 1$ which are implied by given conditions.

Suppose $ \sup_{\vtheta \in \mathcal{C}}|\hat{F}_N(\vtheta) - F(\vtheta)| <  \frac{\alpha \epsilon^2}{4}$ and $\left \lVert \hat{\vtheta}_N^* - \vtheta^* \right \rVert > \epsilon$. Then, from the strong convexity of $F(\vtheta)$, we have
\begin{align*}
    &\quad \hat{F}_N(\hat{\vtheta}_N^*) - \hat{F}_N(\vtheta^*)\\
    &= \left( \hat{F}_N(\hat{\vtheta}_N^*) - F(\hat{\vtheta}_N^*) \right) + \left( F(\vtheta^*) - \hat{F}_N(\vtheta^*)\right) + \left(  F(\hat{\vtheta}_N^*) - F(\vtheta^*) \right)\\
    &\geq -\frac{\alpha \epsilon^2 }{2} + \frac{\alpha}{2} \left \lVert \hat{\vtheta}_N^* - \vtheta^* \right \rVert^2 \\
    &>0.
\end{align*}
This is a contradiction to the fact that $\hat{\vtheta}_N$ is a minimizer of $\hat{F}_N(\vtheta)$. Hence, we have
\begin{align*}
    \mathbb{P} \left[ \left \lVert \hat{\vtheta}_N^* - \vtheta^* \right \rVert >\epsilon \right] &\leq 3 \max \left \{1, \left(\frac{C_2 k^{1/2} D  L}{\alpha \epsilon^2 }\right)^k \right\} \exp \left( - \frac{\alpha^2 \epsilon^4}{C_1 M}N \right),
\end{align*}
and equivalently, if 
\begin{equation*}
    N \geq \frac{C_1 M}{ \alpha^2 \epsilon^4} \log\left(\frac{3}{\delta} \max \left \{1, \left(\frac{C_2 k^{1/2} D  L}{\alpha \epsilon^2 }\right)^k \right\} \right),
\end{equation*}
then $\left \lVert \hat{\vtheta}_N^* - \vtheta^* \right \rVert <\epsilon$ with probability at least $1-\delta$.
\hfill $\square$

\subsection{Proof of Lemma~\ref{lemma:ERM_norm}}\label{proof:lemma:ERM_norm}
By Theorem~\ref{thm:expected_loss_ERM}, there exists $c_\kappa^*\geq 0$ (strict positivity of $c_\kappa^*$ will be proved here) such that $\w= c_\kappa^* \mSigma^{-1}\vmu$. For any $c\geq 0$, we have
\begin{equation*}
\mathbb{E}_\kappa[\loss(c\mSigma^{-1}\vmu)] = \mathbb{E}_{X \sim N \left(\vmu^\top \mSigma^{-1} \vmu, \kappa^{-1} \vmu^\top \mSigma^{-1} \vmu \right)}[l(c X)] = \mathbb{E}_{X \sim N \left(s, \kappa^{-1} s\right)}[l(c X)],
\end{equation*}
where we define $s:= \vmu^\top \mSigma^{-1} \vmu$ for simplicity. By Lemma~\ref{lemma:gradient&hessian}, we have
\begin{align*}
    -\frac{\partial}{\partial c} \mathbb{E}_\kappa[\loss(c\mSigma^{-1}\vmu)] 
    &= - \mathbb{E}_{X \sim N \left(s, \kappa^{-1} s \right)} \left[l'(cX)X \right]\\
    &= \mathbb{E}_{X \sim N \left(s, \kappa^{-1} s \right)} \left[\frac{X}{1+e^{cX}} \right]\\
    &= \mathbb{E}_{X \sim N \left(s, \kappa^{-1} s \right)} \left[\frac{X}{1+e^{cX}}\cdot \vone_{X\geq0} \right]+\mathbb{E}_{X \sim N \left(s, \kappa^{-1}s \right)} \left[\frac{X}{1+e^{cX}} \cdot \vone_{X<0}\right].
\end{align*}
For any $c\geq0$, $1+e^{ct} \geq 2e^{ct/2}$ for all $t \in \mathbb{R}$ by the AM-GM inequality and $1+e^{ct} \leq 2$ if $t<0$. Then, we have
\begin{align*}
    &\quad -\frac{\partial}{\partial c} \mathbb{E}[\loss(c\mSigma^{-1}\vmu)]\\
    &\leq \frac{1}{2} \left (\mathbb{E}_{X \sim N \left(s, \kappa^{-1} s \right)}[X e^{-cX/2} \cdot \vone_{X\geq 0}]+ \mathbb{E}_{X \sim N \left(s, \kappa^{-1} s\right)}[X \cdot \vone_{X<0}]\right)\\
    &= \frac{1}{2} \int_0^\infty \frac{1}{\sqrt{2\pi}(\kappa^{-1} s)^{1/2}}x \exp\left(-\frac{\left(x-s \right)^2}{2\kappa^{-1} s}-\frac{cx}{2}\right)dx + \frac{1}{2} \int_{-\infty}^0  \frac{1}{\sqrt{2\pi}(\kappa^{-1} s)^{1/2}}x \exp \left(-\frac{\left(x-s\right)^2}{2\kappa^{-1} s}\right)dx\\
    &=\frac{1}{2} \int_0^\infty \frac{1}{\sqrt{2\pi}(\kappa^{-1} s)^{1/2}}x \exp\left(-\frac{\left(x-s \right)^2}{2\kappa^{-1} s}-\frac{cx}{2}\right)dx - \frac{1}{2} \int^\infty_0  \frac{1}{\sqrt{2\pi}(\kappa^{-1} s)^{1/2}}x \exp \left(-\frac{\left(x+s\right)^2}{2\kappa^{-1} s}\right)dx\\
    &= \frac{1}{2} \int_0^\infty \frac{1}{\sqrt{2\pi}(\kappa^{-1} s)^{1/2}}x (e^{-cx/2}-e^{-2\kappa x})\exp\left(-\frac{\left(x-s\right)^2}{2\kappa^{-1}s}\right) dx.
\end{align*}
If $c > 4\kappa$, then $e^{-cx/2}-e^{-2\kappa x}<0$ for each $x>0$ and $-\frac{\partial}{\partial c} \mathbb{E}_\kappa \left[ \loss(c\Sigma^{-1}\mu) \right]<0$. Thus, $c_\kappa^* \leq 4\kappa$.

For $c \geq 0$, $1+e^{ct} \leq 2e^{ct}$ if $t\geq 0$ and $1+e^{ct} \geq 2 e^{ct/2}$ for all $t\in \mathbb{R}$ by AM-GM inequality. Therefore, we have
\begin{align*}
    &\quad -\frac{\partial}{\partial c} \mathbb{E}_\kappa[\loss(c\mSigma^{-1}\vmu)] \\
    &\geq \frac{1}{2} \left (\mathbb{E}_{X \sim N \left(s, \kappa^{-1} s\right)}[X e^{-cX} \cdot \vone_{X\geq 0}]+ \mathbb{E}_{X \sim N \left(s, \kappa^{-1} s\right)}[Xe^{-cX/2} \cdot \vone_{X<0}] \right)\\
    &= \frac{1}{2} \int_0^\infty \frac{1}{\sqrt{2\pi}\kappa^{-1} s}x \exp\left(-\frac{\left(x-s\right)^2}{2\kappa^{-1} s}-cx\right)dx + \frac{1}{2} \int_{-\infty}^0  \frac{1}{\sqrt{2\pi}\kappa^{-1}s}x \exp\left(-\frac{\left(x-s\right)^2}{2\kappa^{-1} s}-\frac{cx}{2}\right)dx\\
    &= \frac{1}{2} \int_0^\infty \frac{1}{\sqrt{2\pi}\kappa^{-1} s}x \exp\left(-\frac{\left(x-s\right)^2}{2\kappa^{-1} s}-cx\right)dx - \frac{1}{2} \int_0^\infty  \frac{1}{\sqrt{2\pi}\kappa^{-1}s}x \exp\left(-\frac{\left(x+s\right)^2}{2\kappa^{-1} s}+\frac{cx}{2}\right)dx\\
    &= \frac{1}{2} \int_0^\infty \frac{1}{\sqrt{2\pi}\kappa^{-1} s}x (e^{-cx}-e^{-(2\kappa-c/2)x})\exp\left(-\frac{\left(x-s\right)^2}{2\kappa^{-1} s}\right) dx.
\end{align*}
If $c < \frac{4}{3}\kappa$, $e^{-cx}-e^{-(2\kappa-c/2) x}<0$ for each $x>0$ and $\frac{\partial}{\partial c} \mathbb{E}_\kappa[\loss(c\mSigma^{-1}\vmu)]<0$. Hence, $c_\kappa^* \geq \frac{4}{3}\kappa$ and we conclude $\lVert \w\rVert = \Theta(\kappa)$.
\hfill $\square$
\subsection{Proof of Theorem~\ref{thm:ERM_necessary}}\label{proof:ERM_necessary}
We first show that for sufficiently large $\kappa \in (0, \infty)$ and if $n$ is not sufficiently large, 
$S$ can be linearly separable because data points usually concentrated around $\vmu$ and $-\vmu$. Next, we characterize cosine similarity between $\ell_2$ max-margin vector and the Bayes optimal solution. In our analysis, we use the following well-known lemma (See \citet{hanson1971bound,lugosi2019sub}).

\begin{lemma}\label{lemma:gaussian norm}
For positive definite matrix $\mM$, let $\rvx \sim N(\vzero, \mM)$. For any $\delta \in (0,1)$ 
\begin{equation*}
\lVert \rvx \rVert \leq  \Tr(\mM)^{1/2} + (2 \lVert \mM \rVert \log(1/\delta))^{1/2},
\end{equation*}
with probability at least $1-\delta$.
\end{lemma}
First, we introduce some technical quantities. Let $\delta = 0.99$ and $t = \left (\Tr( \mSigma)^{1/2} + (2 \lVert \mSigma \rVert \log(n/\delta))^{1/2} \right )/\lVert \vmu \rVert$ with
\begin{equation}\label{eqn:ERM_nec}
    n \leq \delta \exp\left( \frac{1}{2} \lVert \mSigma \rVert^{-1} \left( \kappa^{1/2}\lVert \vmu \rVert \cdot \frac{1- \cosim(\vmu, \mSigma^{-1} \vmu)}{2} - \Tr(\mSigma)^{1/2}\right)^2 \right) = \exp(\Theta(\kappa)),
\end{equation}
and we assume $\kappa$ is large enough so that $\kappa^{1/2}\lVert \vmu \rVert \cdot \frac{1- \cosim(\vmu, \mSigma^{-1} \vmu)}{2} - \Tr(\mSigma)^{1/2} >0$.
Then, by substituting $n$ in the definition of $t$ by RHS of \Eqref{eqn:ERM_nec} we have
\begin{align*}
\lVert \vmu \rVert t 
&= \Tr( \mSigma)^{1/2} + (2 \lVert \mSigma \rVert \log(n/\delta))^{1/2}\\
&\leq \Tr( \mSigma)^{1/2} + \left[ 2 \lVert \mSigma \rVert \log\left (\exp\left( \frac{1}{2} \lVert \mSigma \rVert^{-1} \left( \kappa^{1/2}\lVert \vmu \rVert \cdot \frac{1- \cosim(\vmu, \mSigma^{-1} \vmu)}{2} - \Tr(\mSigma)^{1/2}\right)^2 \right) \right) \right]^{1/2}\\
&= \Tr( \mSigma)^{1/2} + \left[ 2 \lVert \mSigma \rVert \left\{ \frac{1}{2} \lVert \mSigma \rVert^{-1} \left( \kappa^{1/2}\lVert \vmu \rVert \cdot \frac{1- \cosim(\vmu, \mSigma^{-1} \vmu)}{2} - \Tr(\mSigma)^{1/2}\right)^2 \right\}  \right]^{1/2}\\
&= \Tr( \mSigma)^{1/2} + \left[ \left( \kappa^{1/2}\lVert \vmu \rVert \cdot \frac{1- \cosim(\vmu, \mSigma^{-1} \vmu)}{2} - \Tr(\mSigma)^{1/2}\right)^2   \right]^{1/2}\\
&= \kappa^{1/2}\lVert \vmu \rVert \cdot\frac{1- \cosim(\vmu, \mSigma^{-1} \vmu)}{2}.
\end{align*}
Thus, 
\begin{equation}
\kappa^{-1/2} t \leq \frac{1-\cosim(\vmu, \mSigma^{-1}\vmu)}{2} < \sqrt{1-\cosim(\vmu, \mSigma^{-1}\vmu)^2}.
\label{eqn:ERM_Necessary-3}
\end{equation}
Next, we investigate how much positive and negative data points are concentrated near their means $\vmu$ and $-\vmu$. By applying Lemma~\ref{lemma:gaussian norm} with $\mM = \kappa^{-1}\mSigma$, for each $i\in [n]$, $\lVert (2y_i-1) \vx_i -  \vmu \rVert \leq \kappa^{-1/2} t \lVert \vmu \rVert$ with probability at least $1-\delta/n$; to see why, recall the definition of $t$. Hence, by union bound, we have $\lVert (2y_i-1) \vx_i -  \vmu \rVert  \leq \kappa^{-1/2}t \lVert \vmu \rVert$ for all $i\in [n]$, with probability at least $1-\delta$. We now condition that this event occurred and we prove that our conclusion holds. First, $S$ is strictly linearly separable by $\vmu$ since 
\begin{align*}
     (2y_i-1) \vmu^\top \vx_i &= \vmu^\top \big((2y_i-1)\vx_i - \vmu \big) + \lVert \vmu \rVert^2\\
     &\geq  \left(- \lVert(2y_i-1)\vx_i - \vmu \rVert +\lVert \vmu \rVert \right)\lVert \vmu \rVert \\
     &\geq (1-\kappa^{-1/2}t) \lVert \vmu\rVert^2 \\
     &\geq  \frac{1+\cosim(\vmu, \mSigma^{-1}\vmu)}{2} \cdot \lVert \vmu\rVert^2 \\
     &>0.
\end{align*}
Hence, there exists $\ell_2$ max-margin vector 
\begin{equation*}
    \bar{\vw}_S = \argmin_{\vw\in \mathbb{R}^d} \lVert \vw \rVert^2 \quad \text{subject to} \quad (2y_i-1) \vw^\top \vx_i \geq 1 \quad \forall i\in [n].
\end{equation*}
From the KKT condition of problem above, we have $\bar{\vw}_S = \sum_{i=1}^n \alpha_i (2y_i-1)\vx_i$ where $\alpha_i \geq 0$ for all $i \in [n]$. By triangular inequality, we have
\begin{equation*}
    \cosim(\bar{\vw}_S, \vmu) = \frac{\vmu^\top \bar{\vw}_S} {\lVert \vmu \rVert \lVert \bar{\vw}_S \rVert}
    = \frac{\vmu^\top (\sum_{i=1}^n \alpha_i (2y_i-1)\vx_i)} {\lVert \vmu \rVert \lVert \sum_{i=1}^n \alpha_i (2y_i-1)\vx_i \rVert}
    \geq \frac{\sum_{i=1}^n \alpha_i \vmu^\top (2y_i-1)\vx_i} { \sum_{i=1}^n \alpha_i \lVert \vmu\rVert \lVert (2y_i-1)\vx_i \rVert}.
\end{equation*}
Also, for all $i \in [n]$,
\begin{align*}
    \frac{\vmu^\top (2y_i-1)\vx_i}{\lVert \vmu\rVert \lVert (2y_i-1)\vx_i \rVert} &= \cosim(\vmu, (2y_i-1)\vx_i )
    = \left (1-\sin^2 \angle(\vmu, (2y_i-1)\vx_i) \right )^{1/2}\\
    &\geq \left ( 1-\left ( \frac{\lVert (2y_i-1)\vx_i - \vmu \rVert}{\lVert \vmu \rVert} \right)^2 \right)^{1/2}
    \geq \sqrt{1-\kappa^{-1}t^2}\\
    & \geq \cosim(\vmu, \mSigma^{-1}\vmu),
\end{align*}
where the last inequality used \Eqref{eqn:ERM_Necessary-3}.
Hence, we have $\cosim(\bar{\vw}_S, \vmu) \geq \sqrt{1-\kappa^{-1} t^2} \geq \cosim(\vmu, \mSigma^{-1}\vmu)$. By triangular inequality for angle, 
\begin{equation*}
    \angle(\bar{\vw}_S, \mSigma^{-1}\vmu) \geq \angle(\vmu, \mSigma^{-1}\vmu) - \angle(\bar{\vw}_S, \vmu)\geq \cos^{-1}(\cosim(\vmu, \mSigma^{-1}\vmu)) - \cos^{-1}\left(\sqrt{1-\kappa^{-1}t^2}\right)\geq 0,
\end{equation*}
and we have
\begin{align}\label{eqn:ERM_Necessary-1}
    &\quad \cosim (\bar{\vw}_S, \mSigma^{-1}\vmu) \nonumber \\
    & \leq \cos \left( \cos^{-1}(\cosim(\vmu, \mSigma^{-1}\vmu)) - \cos^{-1}\left(\sqrt{1-\kappa^{-1}t^2}\right) \right)\nonumber \\
    &= \cos \left(\cos^{-1}(\cosim(\vmu, \mSigma^{-1}\vmu))\right)\cdot  \cos \left(\cos^{-1}\left(\sqrt{1-\kappa^{-1}t^2}\right)\right)\nonumber \\
    &\quad + \sin \left(\cos^{-1}(\cosim(\vmu, \mSigma^{-1}\vmu))\right) \cdot \sin \left(\cos^{-1}\left(\sqrt{1-\kappa^{-1}t^2}\right)\right) \nonumber \\
    &= \cosim(\vmu, \mSigma^{-1}\vmu) \cdot \sqrt{1-\kappa^{-1}t^2} + \sqrt{1-\cosim(\vmu, \mSigma^{-1}\vmu)^2} \cdot \kappa^{-1/2}t\nonumber \\
    &= \cosim(\vmu, \mSigma^{-1}\vmu)\cdot \left (\sqrt{1-\kappa^{-1}t^2}-1\right) + \sqrt{1-\cosim(\vmu, \mSigma^{-1}\vmu)^2} \cdot \kappa^{-1/2} t + \cosim(\vmu, \mSigma^{-1}\vmu).
\end{align}
It is clear that $s \mapsto s\left (\sqrt{1-\kappa^{-1}t^2}-1\right) + \sqrt{1-s^2}\cdot \kappa^{-1/2}t$ is a decreasing function on $[0,1]$ for each fixed $t\in[0,1]$. Therefore, by changing $s = \cosim(\vmu, \mSigma^{-1}\vmu)$ to $s = 0$, we have
\begin{align}\label{eqn:ERM_Necessary-2}
&\quad \cosim(\vmu, \mSigma^{-1}\vmu) \cdot \left(\sqrt{1-\kappa^{-1}t^2}-1\right) + \sqrt{1-\cosim(\vmu, \mSigma^{-1}\vmu)^2}\cdot \kappa^{-1/2} t + \cosim(\vmu, \mSigma^{-1}\vmu)\nonumber \\
&\leq \kappa^{-1/2}t + \cosim(\vmu, \mSigma^{-1}\vmu)\nonumber \\
&\leq \frac{1+\cosim(\vmu, \mSigma^{-1}\vmu)}{2},
\end{align}
where the last inequality used \Eqref{eqn:ERM_Necessary-3}. 
Combining \Eqref{eqn:ERM_Necessary-1} and \Eqref{eqn:ERM_Necessary-2}, we have our conclusion. \hfill $\square$

\section{Proofs for Section \ref{section:mix}}
\subsection{Proof of Theorem~\ref{thm:expected_loss_mix}}\label{proof:expected_loss_mix}
We first prove the existence and uniqueness of a minimizer of $\mathbb{E}_\kappa[\lossmix(\vw)]$ and characterize its direction in the next part.
\paragraph{Step 1: Existence and Uniqueness of a Minimizer of $\mathbb{E}_\kappa[\lossmix(\vw)]$} \quad

Since $(\tilde{\vx}_{i,i}, \tilde{y}_{i,i}) = (\vx_i, y_i)$ for each $i \in [n]$ and $l(\cdot)$ is non-negative, for each $\vw \in \mathbb{R}^d$, we have
\begin{equation*}
    \mathbb{E}_\kappa[\lossmix(\vw)]  \geq \mathbb{E}_\kappa \left[ \frac{1}{n^2} \sum_{i=1}^n y_i l(\vw^\top \vx_i) + (1-y_i) l(-\vw^\top \vx_i)\right] = \frac{1}{n} \mathbb{E}_\kappa[\loss(\vw)].
\end{equation*}
As we discussed in \Eqref{eqn:ball}, if $\lVert \vw \rVert \geq 2 n m_\kappa^{-1}$, we have
\begin{align*}
\mathbb{E}_\kappa[\lossmix (\vw)] &\geq \frac{1}{n} \mathbb{E}_\kappa[\loss (\vw)] \geq \frac{1}{n} \mathbb{E}_{\rvx \sim N(\vmu, \kappa^{-1} \mSigma)} \left[- \frac{1}{2} \vw^\top \rvx \cdot \vone_{\vw^\top \rvx <0} \right]\\
&= \frac{1}{n} \lVert \vw \rVert \mathbb{E}_{\rvx \sim N(\vmu, \kappa^{-1}\mSigma)}\left[- \frac{1}{2} \cdot \frac{\vw^\top \rvx}{\lVert \vw \rVert} \cdot \vone_{\frac{\vw^\top \rvx}{\lVert \vw \rVert}<0} \right]\\
&\geq \frac{m_\kappa}{2n} \lVert \vw \rVert > 1 \geq \log 2 \\
&= \mathbb{E}_\kappa [\lossmix (\vzero)].
\end{align*}
Therefore, a minimizer of $\mathbb{E}_\kappa [ \lossmix (\vw) ]$ necessarily contained in a compact set $\{ \vw \in \mathbb{R}^d :  \lVert \vw \rVert \leq 2 n m_\kappa^{-1} \}$. Since $\mathbb{E}_\kappa [ \lossmix (\vw) ]$ is a continuous function of $\vw$, there must exist a minimizer. The existence part is hence proved.

To show uniqueness, we prove strict convexity of $\mathbb{E}_\kappa[\loss (\vw)]$. From strict convexity of $l(\cdot)$, for any $t \in [0,1]$, $\vw_1, \vw_2 \in \mathbb{R}^d$ with $\vw_1 \neq \vw_2$, and $y \in [0,1]$, we have
\begin{align*}
    &\quad t[yl(\vw_1^\top \vx) + (1-y) l(-\vw_1^\top \vx)] + (1-t) [yl(\vw_2^\top \vx) + (1-y) l(-\vw_2^\top \vx)]\\
    &> y l((t\vw_1+(1-t)\vw_2)^\top \vx) + (1-y) l(-(t\vw_1+(1-t)\vw_2)^\top \vx),
\end{align*}
and any $\vx \in \mathbb{R}^d$ except for a Lebesgue measure zero set (i.e., the set of points $\vx \in \R^d$ satisfying  $\vw_1^\top \vx = \vw_2^\top \vx$).

By taking expectations, we have 
\begin{equation*}
    t \mathbb{E}_\kappa[\lossmix(\vw_1)] + (1-t) \mathbb{E}_\kappa[\lossmix(\vw_2)] >  \mathbb{E}_\kappa[\lossmix(t\vw_1+(1-t)\vw_2)].
\end{equation*}
We conclude $\mathbb{E}_\kappa[\lossmix(\vw)]$ is strictly convex.
Since a strictly convex function has at most one minimizer, we conclude that $\mathbb{E}_\kappa[\lossmix(\vw)]$ has a unique minimizer for any given $\kappa \in (0,\infty)$. 

\paragraph{Step 2: Direction of the Unique Minimizer of $\mathbb{E}_\kappa [\lossmix(\vw)]$} \mbox{}

We express expected losses $\mathbb{E}_\kappa[\lossmix(\vw)]$ as the form
\begin{equation}\label{eqn:mix_form}
 \mathbb{E}_{Z \sim N(0,1)}\left[ \sum_{i=1}^k a_i l\left( b_i \kappa^{-1/2} \left(\vw^\top \mSigma \vw \right)^{1/2} Z + c_i \vw^\top \vmu \right) \right],
\end{equation}
where $a_i,b_i,c_i$'s are real valued random variables depending on $\Lambda$ and $a_i, b_i$'s are positive.
Note that $(\tilde{\vx}_{i,i}, \tilde{y}_{i,i}) = (\vx_i,y_i)$ for each $i\in [n]$ and for each $i,j \in [n]$ with $i \neq j$,
\begin{align}\label{eqn:mix_1}
    \tilde{y}_{i,j} \mid \lambda_{i,j} = 
    \begin{cases}
    1 &\text{with probability } \frac{1}{4},\\
    \lambda_{i,j} &\text{with probability } \frac{1}{4},\\
    1-\lambda_{i,j} &\text{with probability }\frac{1}{4},\\
    0 &\text{with probability } \frac{1}{4}.
    \end{cases}
\end{align}
and the conditional probability distribution of the random variable $\tilde{\vx}_{i,j}$ given $\tilde{y}_{i,j}$ and $\lambda_{i,j}$ can be formulated as the following four cases, depending on the outcome of $\tilde{y}_{i,j}$:
\begin{align}\label{eqn:mix_2}
    \begin{cases}
    \tilde{\vx}_{i,j} \mid \tilde{y}_{i,j}=1, \lambda_{i,j} &\sim  N \left ( \vmu, (g(\lambda_{i,j})^2 +(1-g(\lambda_{i,j}))^2) \kappa^{-1}\mSigma \right ),\\
    \tilde{\vx}_{i,j} \mid \tilde{y}_{i,j}=\lambda_{i,j}, \lambda_{i,j} &\sim  N \left ( (2g(\lambda_{i,j})-1)\vmu, (g(\lambda_{i,j})^2 + (1-g(\lambda_{i,j}))^2) \kappa^{-1} \mSigma \right ),\\
    \tilde{\vx}_{i,j} \mid \tilde{y}_{i,j}=1-\lambda_{i,j}, \lambda_{i,j} &\sim  N \left ( -(2g(\lambda_{i,j})-1)\vmu, (g(\lambda_{i,j})^2 + (1-g(\lambda_{i,j}))^2) \kappa^{-1} \mSigma \right ),\\
    \tilde{\vx}_{i,j} \mid \tilde{y}_{i,j}=0, \lambda_{i,j} &\sim  N \left ( -\vmu, (g(\lambda_{i,j})^2 +(1-g(\lambda_{i,j}))^2) \kappa^{-1} \mSigma \right ).
    \end{cases}
\end{align}
For simplicity, we denote $\mSigma_\lambda = (g(\lambda)^2 + (1-g(\lambda))^2) \mSigma$ for each $\lambda \in [0,1]$. Then, we have 
\begin{align}\label{eqn:mix}
    &\quad \mathbb{E}_\kappa[\lossmix(\vw)] \nonumber \\
    &= \frac{1}{n^2} \sum_{i,j=1}^n\mathbb{E}_\kappa[\tilde{y}_{i,j}l(\vw^\top \tilde{\vx}_{i,j}) + (1-\tilde{y}_{i,j}) l(-\vw^\top \tilde{\vx}_{i,j})]\nonumber \\
    &= \frac{1}{n^2} \left ( \sum_{\overset{i,j=1}{i \neq j}}^n\mathbb{E}_\kappa \left [\tilde{y}_{i,j} l(\vw^\top \tilde{\vx}_{i,j}) + (1-\tilde{y}_{i,j}) l(-\vw^\top \tilde{\vx}_{i,j}) \right ] + \sum_{i=1}^n \mathbb{E}_\kappa \left [y_i l(\vw^\top \vx_i) + (1-y_i) l(-\vw^\top \vx_i)\right ]\right)\nonumber\\
    &= \frac{n-1}{2n} \mathbb{E}_{Z \sim N(0,1), \lambda \sim \Lambda } \left [l \left ( (\kappa^{-1} \vw^\top \mSigma_\lambda \vw)^{1/2} Z + \vw^\top \vmu \right ) \right ]\nonumber \\
    &\quad+ \frac{n-1}{2n} \mathbb{E}_{Z \sim N(0,1), \lambda \sim \Lambda } \left [\lambda l \left ( (\kappa^{-1} \vw^\top \mSigma_\lambda \vw)^{1/2} Z + (2g(\lambda)-1) \vw^\top \vmu \right ) \right ]\nonumber \\
    &\quad+ \frac{n-1}{2n} \mathbb{E}_{Z \sim N(0,1), \lambda \sim \Lambda } \left [(1-\lambda) l \left ( (\kappa^{-1} \vw^\top \mSigma_\lambda \vw)^{1/2} Z -(2g(\lambda)-1) \vw^\top \vmu \right) \right ]\nonumber \\
    &\quad+ \frac{1}{n} \mathbb{E}_{Z \sim N(0,1) }\left[ l \left ( (\kappa^{-1} \vw^\top \mSigma \vw)^{1/2} Z + \vw^\top \vmu \right )\right].
\end{align}
 This is the form in \Eqref{eqn:mix_form}. Let $C_\kappa = \wmix^\top \vmu$, then by Lemma~\ref{lemma:logistic}, $\wmix$ have to be a solution of the problem $\min_{\vw^\top \vmu = C_\kappa} \frac{1}{2}\vw^\top \mSigma \vw$, and by Lemma~\ref{lemma:opt}, there exists $c_{\mathrm{mix}, n, \kappa}$ such that  $\wmix:= c_{\mathrm{mix}, n, \kappa} \mSigma^{-1} \vmu$ be the unique minimizer of $\mathbb{E}_\kappa[\lossmix(\vw)]$. 
 
 The only remaining part is showing $c_{\mathrm{mix},n,\kappa}>0$. For simplicity, we will omit $\kappa$ and $n$ in $\wmix$. If $c_{\mathrm{mix}, n, \kappa}<0$, 
\begin{align*}
    &\quad \mathbb{E}_\kappa \left[\lossmix\left(\vw_{\mathrm{mix}}^*\right)\right]\\
    &= \frac{n-1}{2n} \underbrace{\mathbb{E}_{\overset{Z \sim N(0,1)}{\lambda \sim \Lambda} } \left[l\left( \left (\kappa^{-1} {\vw_{\mathrm{mix}}^*}^\top \mSigma_\lambda \vw_{\mathrm{mix}}^* \right)^{1/2}  Z + c_{\mathrm{mix},n,\kappa} \vmu^\top \mSigma^{-1} \mu \right)\right] }_\text{(a)}\\
    &\quad+ \frac{n-1}{2n} \underbrace{\mathbb{E}_{\overset{Z \sim N(0,1)}{\lambda \sim \Lambda} }\left[\lambda l\left( \left(\kappa^{-1} {\vw_{\mathrm{mix}}^*}^\top \mSigma_\lambda \vw_{\mathrm{mix}}^*\right)^{1/2} Z + (2g(\lambda)-1) c_{\mathrm{mix},n,\kappa} \vmu^\top \mSigma^{-1} \vmu\right)\right]}_\text{(b)} \\
    &\quad+ \frac{n-1}{2n} \underbrace{\mathbb{E}_{\overset{Z \sim N(0,1)}{\lambda \sim \Lambda} }\left[(1-\lambda) l\left( \left(\kappa^{-1} {\vw_{\mathrm{mix}}^*}^\top \mSigma_\lambda \vw_{\mathrm{mix}}^*\right)^{1/2} Z -(2g(\lambda)-1) c_{\mathrm{mix}, n, \kappa} \vmu^\top \mSigma^{-1} \vmu\right)\right]}_\text{(c)} \\
    &\quad+ \frac{1}{n} \underbrace{\mathbb{E}_{Z \sim N(0,1) }\left[ l\left( \left(\kappa^{-1} {\vw_{\mathrm{mix}}^*}^\top \mSigma \vw_{\mathrm{mix}}^*\right)^{1/2} Z + c_{\mathrm{mix}, n, \kappa} \vmu^\top \mSigma^{-1} \vmu \right) \right]}_\text{(d)}\\
    &> \frac{n-1}{2n} \underbrace{\mathbb{E}_{\overset{Z \sim N(0,1)}{\lambda \sim \Lambda} } \left[l\left( \left (\kappa^{-1} {\vw_{\mathrm{mix}}^*}^\top \mSigma_\lambda \vw_{\mathrm{mix}}^* \right)^{1/2}  Z - c_{\mathrm{mix},n,\kappa} \vmu^\top \mSigma^{-1} \vmu \right)\right]}_\text{(a)'} \\
    &\quad+ \frac{n-1}{2n} \underbrace{\mathbb{E}_{\overset{Z \sim N(0,1)}{\lambda \sim \Lambda} }\left[\lambda l\left( \left(\kappa^{-1} {\vw_{\mathrm{mix}}^*}^\top \mSigma_\lambda \vw_{\mathrm{mix}}^*\right)^{1/2} Z - (2g(\lambda)-1) c_{\mathrm{mix},n,\kappa} \vmu^\top \mSigma^{-1} \vmu\right)\right]}_\text{(b)'} \\
    &\quad+ \frac{n-1}{2n} \underbrace{\mathbb{E}_{\overset{Z \sim N(0,1)}{\lambda \sim \Lambda} }\left[(1-\lambda) l\left(\left(\kappa^{-1} {\vw_{\mathrm{mix}}^*}^\top \mSigma_\lambda \vw_{\mathrm{mix}}^* \right)^{1/2} Z +(2g(\lambda)-1) c_{\mathrm{mix}, n, \kappa} \vmu^\top \mSigma^{-1} \vmu\right)\right]}_\text{(c)'} \\
    &\quad+ \frac{1}{n} \underbrace{\mathbb{E}_{Z \sim N(0,1) }\left[ l\left( \left(\kappa^{-1} {\vw_{\mathrm{mix}}^*}^\top \mSigma \vw_{\mathrm{mix}}^*\right)^{1/2} Z - c_{\mathrm{mix}, n, \kappa} \vmu^\top \mSigma^{-1} \vmu \right) \right]}_\text{(d)'}\\
    &= \mathbb{E}_\kappa[\lossmix(-\vw_{\mathrm{mix}}^*)].
\end{align*}
The inequality holds by comparing (a), (b), (c), and (d) with (a)', (b)', (c)', and (d)' respectively. This contradicts the assumption that $\vw_{\mathrm{mix}}^*$ is a unique minimizer of $\mathbb{E}_\kappa[\lossmix(\vw)]$, letting us to conclude $c_{\mathrm{mix},n, \kappa} \geq 0$. Showing $c_{\mathrm{mix},n, \kappa}$ is strictly positive will be handled in the proof of Lemma~\ref{lemma:mixup_norm} which can be found in Appendix~\ref{proof:lemma:mixup_norm}. \hfill $\square$
\subsection{Proof of Theorem~\ref{thm:mixup_convergence}}
Since we consider sufficiently large $\kappa$, we may assume $n \geq d$ and $n\geq 2$.
By Lemma~\ref{lemma:mixup_norm}, we can choose $R_\kappa = \Theta(1)$ such that $\lVert \wmix \rVert < R_\kappa$ for \emph{any} $n\in \mathbb{N}$. Let us define a compact set $\mathcal{C}_\kappa:= \{\vw \in \mathbb{R}^d \mid \lVert \vw \rVert \leq R_\kappa\}$. For any vector $\vw\in \mathbb{R}^d$ and nonzero vector $\vv \in \mathbb{R}^d$, we have
\begin{align*}
    \vv^\top \nabla^2_\vw \lossmix (\vw) \vv &= \frac{1}{n^2} \sum_{i,j=1}^n (\tilde{y}_{i,j} l''(\vw^\top \tilde{\vx}_{i,j})(\vv^\top \tilde{\vx}_{i,j})^2 +  (1-\tilde{y}_{i,j})l''(-\vw^\top \tilde{\vx}_{i,j})(\vv^\top \tilde{\vx}_{i,j})^2)\\
    &\geq \frac{1}{n^2} \sum_{i=1}^n (y_i l''(\vw^\top \vx_i)(\vv^\top \vx_i)^2 +  (1-y_i)l''(-\vw^\top \vx_i)(\vv^\top \vx_i)^2)\\
    &>0,
\end{align*}
almost surely since $\{ \vx_i\}_{i=1}^n $ spans $\mathbb{R}^d$ almost surely. Therefore, $\lossmix(\vw)$ is strictly convex almost surely and we conclude that  $\lossmix(\vw)$ has a unique minimizer $\hat{\vw}_{\mathrm{mix}, S}^*$ on $\mathcal{C}$ almost surely. Also, if $\hat{\vw}_{\mathrm{mix},S}^*$ belongs to the interior of $\mathcal{C}$, then it is a unique minimizer of $\lossmix(\vw)$ on $\mathbb{R}^d$. We prove high probability convergence of $\hat{\vw}_{\mathrm{mix}, S}^*$ to $\wmix$ using Lemma~\ref{lemma:minimizer_dependent} and convert $\ell_2$ convergence into directional convergence.
For simplicity, we define
\begin{equation*}
f_{i,j}(\vw) := \tilde{y}_{i,j}l(\vw^\top \tilde{\vx}_{i,j}) + (1-\tilde{y}_{i,j}) l(\vw^\top \tilde{\vx}_{i,j}),
\end{equation*}
for each $i,j\in[n]$. We start with the following claim which is useful for estimating quantities described in assumptions of Lemma~\ref{lemma:minimizer_dependent} for our setting.
\begin{claim}\label{claim:mix}
For any $t>0$, we have
\begin{equation*}
\mathbb{E}_\kappa \left[ e^{t \lVert \tilde{\vx}_{i,j} \rVert}\right] \leq \left( 2^{d/2} + e^{4\kappa^{-1}t^2 \lVert \mSigma \rVert}\right)e^{t\lVert \vmu \rVert},
\end{equation*}
for all $i,j \in [n]$.
\end{claim}
\begin{proof}[Proof of Claim~\ref{claim:mix}]
We first consider the case $i = j$. By applying triangular inequality and Lemma~\ref{lemma:ineq4}, we have
\begin{align*}
\mathbb{E}_\kappa \left[ e^{t \left \lVert \tilde{\vx}_{i,i}\right \rVert} \right] 
&= \mathbb{E}_{\rvx \sim N(\vmu, \kappa^{-1} \mSigma) } \left[ e^{t \left \lVert \rvx\right \rVert} \right] \\
&\leq \mathbb{E}_{\rvx \sim N(\vmu, \kappa^{-1} \mSigma) } \left[ e^{t \left \lVert \rvx - \vmu \right \rVert} \right] e^{t\lVert \vmu \rVert}\\
&= \mathbb{E}_{\rvz \sim N(\vzero, \kappa^{-1}t^2 \mSigma )} \left[e^{\lVert \rvz \rVert} \right] e^{t\lVert \vmu \rVert} \\
&\leq \left( 2^{d/2} + e^{4\kappa^{-1}t^2 \lVert \mSigma \rVert}\right)e^{t\lVert \vmu \rVert},
\end{align*}
for each $i \in [n]$.

Next, we handle the case $i \neq j$. For simplicity, we denote $\mSigma_\lambda = (g(\lambda)^2 + (1-g(\lambda))^2) \mSigma$ for each $\lambda \in [0,1]$. \Eqref{eqn:mix_1} and \Eqref{eqn:mix_2} imply that for each $i,j\in[n]$ with $i \neq j$,
\begin{align*}
&\quad \mathbb{E}_\kappa\left[ e^{t\lVert \tilde{\vx}_{i,j} \rVert}\right]\\
&= \mathbb{E}_{\lambda \sim \Lambda} \left[ \frac{1}{2} \mathbb{E}_{\rvx \sim N(\vmu, \kappa^{-1} \mSigma_\lambda)}\left[ e^{t\lVert \rvx\rVert}\right] +\frac{1}{2} \mathbb{E}_{\rvx \sim N((2g(\lambda)-1)\vmu, \kappa^{-1} \mSigma_\lambda)}\left[ e^{t\lVert \rvx\rVert}\right]\right]\\
&\leq \frac{1}{2} \mathbb{E}_{\lambda \sim \Lambda}\left[\mathbb{E}_{\rvx \sim N(\vmu, \kappa^{-1} \mSigma_\lambda)}\left[ e^{t\lVert \rvx -\vmu \rVert}\right]e^{t\lVert \vmu \rVert} + \mathbb{E}_{\rvx \sim N((2g(\lambda)-1)\vmu, \kappa^{-1} \mSigma_\lambda)}\left[ e^{t\lVert \rvx - (2g(\lambda)-1)\vmu\rVert}\right] e^{t\lVert (2g(\lambda)-1)\vmu \rVert}\right]\\
&\leq \mathbb{E}_{\lambda \sim \Lambda}\left[ \mathbb{E}_{\rvz \sim N(\vzero, \kappa^{-1} \mSigma_\lambda) }\left[ e^{t\lVert \rvz \rVert}\right]\right] e^{t\lVert \vmu\rVert}.
\end{align*}
By Lemma~\ref{lemma:ineq4} for each $\lambda \in [0,1]$ 
\begin{equation*}
\mathbb{E}_{\rvz \sim N(\vzero, \kappa^{-1} \mSigma_\lambda) }\left[ e^{t\lVert \rvz \rVert}\right] \leq 2^{d/2} + e^{4\kappa^{-1}t^2 \lVert \mSigma_\lambda\rVert} \leq 2^{d/2} + e^{4\kappa
^{-1}t^2 \lVert \mSigma\rVert}.
\end{equation*}
and we have our conclusion.
\end{proof}

\paragraph{Step 1: Estimate Upper Bound of $\mathbb{E}_\kappa \left[ e^{\left|f_{i,j}(\vw) - \mathbb{E}_\kappa[f_{i,j}(\vw)]\right|}\right]$ on $\mathcal{C}_\kappa$ }\quad

For any $\vw \in \mathcal{C}$ and $i,j \in [n]$, we have
\begin{equation*}
    |f_{i,j}(\vw)| = |\tilde{y}_{i,j}l(\vw^\top \tilde{\vx}_{i,j}) + (1-\tilde{y}_{i,j})l(-\vw^\top \tilde{\vx}_{i,j})| \leq l(-|\vw^\top \tilde{\vx}_{i,j}|) \leq l(-R_\kappa \lVert \tilde{\vx}_{i,j} \rVert) = \log \left( 1+e^{R_\kappa \lVert \tilde{\vx}_{i,j}\rVert} \right) .
\end{equation*}
By Lemma~\ref{lemma:mixup_norm}, $R_\kappa = \Theta(1)$ and by applying Claim~\ref{claim:mix} for $t = R_\kappa$,  we obtain $M_\kappa'$ such that $\mathbb{E}_\kappa\left[ e^{|f_{i,j}(\vw)|}\right] < M_\kappa'$ and $M_\kappa' = \Theta(1)$.
Also, by triangular inequality and Jensen's inequality, we have
\begin{align*}
\mathbb{E}_\kappa \left[ e^{\left|f_{i,j}(\vw) - \mathbb{E}_\kappa[f_{i,j}(\vw)]\right|}\right]
&\leq \mathbb{E}_\kappa \left[ e^{\left|f_{i,j}(\vw)\right| + \left|\mathbb{E}_\kappa[f_{i,j}(\vw)]\right|}\right]\\
&\leq \mathbb{E}_\kappa \left[ e^{\left|f_{i,j}(\vw)\right| }\right]^2\\
&\leq {M_\kappa'}^2.
\end{align*}
Letting $M_\kappa  = {M_\kappa'^2}$, we have
\begin{equation*}
\mathbb{E}_\kappa \left[ e^{\left|f_{i,j}(\vw) - \mathbb{E}_\kappa[f_{i,j}(\vw)]\right|}\right] \leq M_\kappa,
\end{equation*}
for any $\vw \in \mathcal{C}_\kappa$ and $M_\kappa = \Theta(1)$.

\paragraph{Step 2: Estimate Upper Bound of $\lVert \nabla_{\vw} f_{i,j}(\vw) \rVert$ and $\lVert \nabla_\vw \mathbb{E}_\kappa[f_{i,j}(\vw)]\rVert$ } \quad

For each $\vw \in \mathcal{C}_\kappa$ and $i,j \in [n]$, 
\begin{align*}
\lVert \nabla_\vw f_{i,j}(\vw) \rVert &= \left \lVert \nabla_\vw (\tilde{y}_{i,j}l(\vw^\top \tilde{\vx}_{i,j}) + (1-\tilde{y}_{i,j}) l(-\vw^\top \tilde{\vx}_{i,j})) \right \rVert\\
&= \left \lVert \tilde{y}_{i,j}l'(\vw^\top \tilde{\vx}_{i,j})\tilde{\vx}_{i,j} - (1-\tilde{y}_{i,j}) l'(-\vw^\top \tilde{\vx}_{i,j})\tilde{\vx}_{i,j} \right \rVert\\
&\leq  \lVert \tilde{\vx}_{i,j} \rVert.
\end{align*}
In addition, by Lemma~\ref{lemma:gradient&hessian},
\begin{align*}
\left\lVert \nabla_\vw \mathbb{E}_\kappa [f_{i,j}(\vw)]\right \rVert &=  \left \lVert \nabla_\vw \mathbb{E}_\kappa[ (\tilde{y}_{i,j}l(\vw^\top \tilde{\vx}_{i,j}) + (1-\tilde{y}_{i,j}) l(-\vw^\top \tilde{\vx}_{i,j}))]\right \rVert\\
&= \left   \lVert\mathbb{E}_\kappa[ \nabla_\vw(\tilde{y}_{i,j}l(\vw^\top \tilde{\vx}_{i,j}) + (1-\tilde{y}_{i,j}) l(-\vw^\top \tilde{\vx}_{i,j}))]\right \rVert\\
&= \left \lVert \mathbb{E}_\kappa[ \tilde{y}_{i,j}l'(\vw^\top \tilde{\vx}_{i,j})\tilde{\vx}_{i,j} - (1-\tilde{y}_{i,j}) l'(-\vw^\top \tilde{\vx}_{i,j})\tilde{\vx}_{i,j} \right] \rVert \\
&\leq \mathbb{E}_\kappa [\lVert \tilde{\vx}_{i,j}\rVert].
\end{align*}
Also, applying Claim~\ref{claim:mix} with $t = 1$, there exists $L_\kappa$ such that $\mathbb{E}_\kappa \left[ e^{\lVert \tilde{\vx}_{i,j} \rVert}\right]< L_\kappa$ and $L_\kappa = \Theta(1)$.

\paragraph{Step 3: Estimate Strong convexity Constant of $\mathbb{E}_\kappa[\lossmix (\vw)]$ on $\mathcal{C}_\kappa$} \quad

For each $\lambda \in [0,1]$, denote $\mSigma_\lambda : = (g(\lambda)^2 + (1-g(\lambda))^2) \mSigma$, for simplicity. By Lemma~\ref{lemma:gradient&hessian}, Lemma \ref{lemma:ineq1} and \Eqref{eqn:mix}, we have 
\begin{align*}
    \vv^\top \nabla^2_\vw \mathbb{E}_\kappa[\lossmix (\vw)] \vv &\geq \frac{n-1}{2n} \mathbb{E}_{\lambda \sim \Lambda}[\mathbb{E}_{\rvx \sim N(\vmu, \kappa^{-1} \mSigma_\lambda)}[l''(\vw^\top \rvx)(\vv^\top \rvx)^2]]\\
    &\geq \frac{1}{16} \mathbb{E}_{\lambda \sim \Lambda}[\mathbb{E}_{\rvx \sim N(\vmu, \kappa^{-1} \mSigma_\lambda)}[e^{-(\vw^\top \rvx)^2/2}(\vv^\top \rvx)^2]],
\end{align*}
for any vector $\vw\in\mathcal{C}$ and unit vector $\vv \in \mathbb{R}^d$. By Lemma~\ref{lemma:ineq3} and Lemma~\ref{lemma:mixup_norm}, for each $\lambda \in [0,1]$, 
\begin{align*}
&\quad \mathbb{E}_{\rvx \sim N(\vmu, \kappa^{-1} \mSigma_\lambda)}[e^{-(\vw^\top \rvx)^2/2}(\vv^\top \rvx)^2]\\
&\geq \kappa^{d/2} \lVert \mSigma_\lambda \rVert^{-d/2} (\kappa\lVert \mSigma_\lambda^{-1} \rVert + R_\kappa^2)^{-(d+2)/2} \exp(-R_\kappa^2 \lVert \mSigma_\lambda \rVert \lVert \mSigma_\lambda^{-1} \rVert \lVert \vmu\rVert^2)\\
&= \Theta (\kappa^{-1}).
\end{align*}
Hence, $\mathbb{E}_\kappa[\lossmix(\vw)]$ is $\alpha_\kappa $-strongly convex with $\alpha_\kappa = \Theta(\kappa^{-1})$.

\paragraph{Step 4: Sample Complexity for Directional Convergence} \quad

By Lemma~\ref{lemma:mixup_norm}, we can choose $r_\kappa = \Theta(1)$ such that 
 $r_\kappa \leq \lVert \wmix \rVert$ for any $n$. Since $r_\kappa = \Theta(1)$, we assume $\kappa$ is large enough so that 
$r_\kappa \epsilon < r_\kappa <\alpha_\kappa^{-1/2}$.
Assume the unique existence of $\hat{\vw}_{\mathrm{mix},S}^*$ which occurs almost surely. 
By Lemma~\ref{lemma:minimizer_dependent}, if
\begin{equation*}
    n \geq \frac{C_1'M_\kappa}{\alpha_\kappa^2 r_\kappa^4 \epsilon^4} \log \left( \frac{3}{\delta} \max \left\{ 1, \left( \frac{2C_2'd^{1/2}R_\kappa L_\kappa}{\alpha_\kappa r_\kappa ^2 \epsilon^2} \right)^d  \right\} \right) = \frac{\Theta(\kappa^2)}{\epsilon^4} \left (1+ \log \frac{1}{\epsilon} + \log \frac{1}{\delta} \right ),
\end{equation*}
then we have $\lVert\wmix - \hat{\vw}_{\mathrm{mix},S}^* \rVert \leq  r_\kappa \epsilon \leq \lVert\wmix \rVert \epsilon$ with probability at least $1-\delta$.
Furthermore, if $\lVert \wmix - \hat{\vw}_{\mathrm{mix},S}^* \rVert \leq  \lVert \wmix \rVert \epsilon$, then $\hat{\vw}_{\mathrm{mix},S}^*$ belongs to interior of $\mathcal{C}_\kappa$. Hence, $\hat{\vw}_{\mathrm{mix},S}^*$ is a minimizer of $\lossmix(\vw)$ over the entire $\mathbb{R}^d$. Also, we have
\begin{align*}
    \cosim ( \hat{\vw}_{\mathrm{mix},S}^*, \mSigma^{-1}\vmu) &=  \cosim ( \hat{\vw}_{\mathrm{mix},S}^*,\wmix)\\
    &= \bigg (1- \sin^2\Big (\angle \big (\hat{\vw}_{\mathrm{mix},S}^*, \wmix \big ) \Big ) \bigg)^{1/2}\\
    &\geq 1- \sin\Big (\angle\big (\hat{\vw}_{\mathrm{mix},S}^*, \wmix \big) \Big)\\
    &\geq 1- \epsilon.
\end{align*}
Hence, we conclude that if $n = \frac{\Omega(\kappa^2)}{\epsilon^4} \left (1+ \log \frac{1}{\epsilon} + \log \frac{1}{\delta} \right)$, then
with probability at least $1-\delta$, the Mixup loss $\lossmix(w)$ has a unique minimizer $\hat{\vw}_{\mathrm{mix},S}^*$ and $\cosim ( \hat{\vw}_{\mathrm{mix},S}^*, \mSigma^{-1}\vmu)\geq 1- \epsilon$. \hfill $\square$

\subsection{Proof of Lemma~\ref{lemma:mixup_norm}}\label{proof:lemma:mixup_norm}
By Theorem~\ref{thm:expected_loss_mix}, let $\wmix = c_{\mathrm{mix}, n, \kappa}^* \mSigma^{-1} \vmu$ where $c_{\mathrm{mix}, n, \kappa }^* \geq 0 $. For any $c \in \R$, from \Eqref{eqn:mix}, we have
\begin{align*}
    \mathbb{E}_\kappa[\lossmix(c \mSigma^{-1} \vmu)]
    &= \frac{n-1}{2n} \mathbb{E}_{\lambda \sim \Lambda}[\mathbb{E}_{X \sim N(s, \kappa^{-1} s_\lambda) } [l(c X)]] \\
    &\quad+ \frac{n-1}{2n} \mathbb{E}_{\lambda \sim \Lambda}[\lambda \cdot\mathbb{E}_{X \sim N((2g(\lambda)-1)s, \kappa^{-1} s_\lambda) }[ l(cX)]] \\
    &\quad+ \frac{n-1}{2n} \mathbb{E}_{\lambda \sim \Lambda}[(1-\lambda) \cdot\mathbb{E}_{\rvx \sim N(-(2g(\lambda)-1)s, \kappa^{-1} s_\lambda) }[ l(cX)]] \\
    &\quad+ \frac{1}{n} \mathbb{E}_{X \sim N(s,\kappa^{-1}s) }\left[ l(c X)\right].
\end{align*}
where we denote $s_\lambda = (g(\lambda)^2 + (1-g(\lambda))^2) \vmu^\top \mSigma \vmu$ for each $\lambda \in [0,1]$ and $s = \vmu^\top \mSigma \vmu$, for simplicity.

By Lemma~\ref{lemma:gradient&hessian}, we have
\begin{align*}
    &\quad - \frac{\partial}{\partial c}\mathbb{E}_\kappa[\lossmix(c \mSigma^{-1} \vmu)]\\
    &= -\frac{n-1}{2n} \mathbb{E}_{\lambda \sim \Lambda}[\mathbb{E}_{X \sim N(s, \kappa^{-1} s_\lambda) } [l'(c X)X]] \\
    &\quad- \frac{n-1}{2n} \mathbb{E}_{\lambda \sim \Lambda}[\lambda \cdot \mathbb{E}_{X \sim N((2g(\lambda)-1)s, \kappa^{-1} s_\lambda) }[ l'(cX)X]] \\
    &\quad- \frac{n-1}{2n} \mathbb{E}_{\lambda \sim \Lambda}[(1-\lambda)\cdot\mathbb{E}_{X \sim N(-(2g(\lambda)-1)s, \kappa^{-1} s_\lambda) }[ l'(cX)X]] \\
    &\quad- \frac{1}{n} \mathbb{E}_{X \sim N(s, \kappa^{-1}s) }\left[  l'(cX)X\right]\\
    &= \frac{n-1}{2n} \mathbb{E}_{\lambda \sim \Lambda}\left[\mathbb{E}_{X \sim N(s, \kappa^{-1} s_\lambda) } \left[\frac{X}{1+e^{cX}}\right]\right] \\
    &\quad+ \frac{n-1}{2n} \mathbb{E}_{\lambda \sim \Lambda}\left[\lambda \cdot \mathbb{E}_{X \sim N((2g(\lambda)-1)s, \kappa^{-1} s_\lambda) }\left[\frac{X}{1+e^{cX}}\right]\right] \\
    &\quad+ \frac{n-1}{2n} \mathbb{E}_{\lambda \sim \Lambda}\left[(1-\lambda) \cdot\mathbb{E}_{X \sim N(-(2g(\lambda)-1)s, \kappa^{-1} s_\lambda) }\left[\frac{X}{1+e^{cX} }\right]\right] \\
    &\quad+ \frac{1}{n} \mathbb{E}_{X \sim N(s,\kappa^{-1}s) }\left[\frac{X}{1+e^{cX} }\right],
\end{align*}
and if $c\geq 0$,  by applying Lemma~\ref{lemma:ineq2} with $z= cX$, we have
\begin{align*}
    &\quad -\frac{\partial}{\partial c } \mathbb{E}[\lossmix(c \mSigma^{-1}\vmu)] \\
    &\geq \frac{n-1}{2n} \mathbb{E}_{\lambda \sim \Lambda}\left[\mathbb{E}_{X \sim N(s, \kappa^{-1} s_\lambda) } \left[\frac{1}{2}X - \frac{1}{4} c X^2 \right]\right] \\
    &\quad+ \frac{n-1}{2n} \mathbb{E}_{\lambda \sim \Lambda}\left[\lambda \cdot \mathbb{E}_{X \sim N((2g(\lambda)-1)s, \kappa^{-1} s_\lambda) }\left[\frac{1}{2}X - \frac{1}{4} c X^2\right]\right] \\
    &\quad+ \frac{n-1}{2n} \mathbb{E}_{\lambda \sim \Lambda}\left[(1-\lambda) \cdot\mathbb{E}_{X \sim N(-(2g(\lambda)-1)s, \kappa^{-1} s_\lambda) }\left[\frac{1}{2}X - \frac{1}{4} c X^2\right]\right] \\
    &\quad+ \frac{1}{n} \mathbb{E}_{X \sim N(s,\kappa^{-1}s) }\left[\frac{1}{2}X - \frac{1}{4} c X^2\right]\\
    &= \frac{s}{4n} \cdot \Big((n-1)\mathbb{E}_{\lambda \sim \Lambda}[1+(2\lambda-1)(2g(\lambda)-1)]+2 \Big) \\
    &\quad - \frac{c}{8n} \cdot \left( \kappa^{-1} \left( 2 (n-1) \mathbb{E}_{\lambda \sim \Lambda}[ s_\lambda]+ 2s \right) +s^2 \left((n-1)\mathbb{E}_{\lambda \sim \Lambda}[(2g(\lambda)-1)^2 +1] +2 \right)  \right) \\
    & \geq \frac{s}{8} \cdot \Bigg( \mathbb{E}_{\lambda \sim \Lambda}[(2\lambda-1)(2 g(\lambda)-1)] - c  \Big(\kappa^{-1} \left(2\mathbb{E}_{\lambda \sim \Lambda}\left[g(\lambda)^2 + (1-g(\lambda))^2 \right]+1\right) + s\left( \mathbb{E}_{\lambda \sim \Lambda}[(2g(\lambda)-1)^2+2]\right)\Big)\Bigg).
\end{align*}

Thus, if $0 \leq c <\frac{\mathbb{E}_{\lambda \sim \Lambda}\left[ (2\lambda-1)(2g(\lambda)-1)\right]}{ \kappa^{-1} \left(2\mathbb{E}_{\lambda \sim \Lambda}\left[g(\lambda)^2 + (1-g(\lambda))^2 \right]+1\right) + s\left( \mathbb{E}_{\lambda \sim \Lambda}[(2g(\lambda)-1)^2+2]\right)}$, then $\mathbb{E}[\lossmix(c \Sigma^{-1}\mu)]$ is decreasing as a function of $c$ and we conclude $c_{\mathrm{mix},n,\kappa }\geq \frac{\mathbb{E}_{\lambda \sim \Lambda}\left[ (2\lambda-1)(2g(\lambda)-1)\right]}{ \kappa^{-1} \left(2\mathbb{E}_{\lambda \sim \Lambda}\left[g(\lambda)^2 + (1-g(\lambda))^2 \right]+1\right) + s\left( \mathbb{E}_{\lambda \sim \Lambda}[(2g(\lambda)-1)^2+2]\right)}$ which implies $\lVert \wmix \rVert = \Omega(1)$, and one can note that this lower bound is independent of $n$.

In order to get the upper bound, we use the following inequality: For each $z \in \mathbb{R}$, we have
\begin{equation*}
    l(z) + l(-z) = \log(1+e^{-z}) + \log(1+e^z) = \log(2+e^z + e^{-z}) > |z|.
\end{equation*}
For each $c \geq 0$, we have
\begin{align*}
    \mathbb{E}_\kappa[\lossmix(c \mSigma^{-1}\vmu)] &= \frac{1}{n^2}  \sum_{i,j=1}^n \mathbb{E}_\kappa\left[\tilde{y}_{i,j} l( c \vmu^\top \mSigma^{-1} \tilde{\vx}_{i,j}) + (1-\tilde{y}_{i,j})l(- c \vmu^\top \mSigma^{-1} \tilde{\vx}_{i,j})\right]\\
    &> \frac{1}{n^2} \sum_{\overset{i,j=1}{i \neq j}}^n \mathbb{E}_\kappa\left[ \tilde{y}_{i,j} l( c \vmu^\top \mSigma^{-1} \tilde{\vx}_{i,j}) + (1-\tilde{y}_{i,j})l(- c \vmu^\top \mSigma^{-1} \tilde{\vx}_{i,j})\right]\\
    &\geq \frac{1}{n^2} \sum_{\overset{i,j=1}{i \neq j}}^n  \mathbb{E}_\kappa\left[\min \{\tilde{y}_{i,j},1-\tilde{y}_{i,j}\} ( l( c \vmu^\top \mSigma^{-1} \tilde{\vx}_{i,j}) + l(- c \vmu^\top \mSigma^{-1} \tilde{\vx}_{i,j}))\right]\\
    &\geq \frac{1}{n^2}  \sum_{\overset{i,j=1}{i \neq j}}^n \mathbb{E}_\kappa\left[ \min \{ \tilde{y}_{i,j},1-\tilde{y}_{i,j}\} | c \vmu^\top \mSigma^{-1} \tilde{\vx}_{i,j}|\right]\\
    &= \frac{1}{n^2}  \sum_{\overset{i,j=1}{i \neq j}}^n \mathbb{E}_\kappa\left[ |\min \{ \tilde{y}_{i,j},1-\tilde{y}_{i,j}\} c \vmu^\top \mSigma^{-1} \tilde{\vx}_{i,j}|\right]\\
    &= \frac{1}{n^2}  \sum_{\overset{i,j=1}{i \neq j}}^n \mathbb{E}_\kappa\left[ |\min \{ \tilde{y}_{i,j},1-\tilde{y}_{i,j}\}(2g(\tilde{y}_{i,j})-1) s|\right] c\\
    &= \frac{n-1}{n}\mathbb{E}_{\lambda \sim \Lambda}[\min \{\lambda, 1-\lambda \}|2g(\lambda)-1|]s c\\
    &\geq \frac{1}{4}\mathbb{E}_{\lambda \sim \Lambda}[\min \{\lambda, 1-\lambda \}|2g(\lambda)-1|] s c.
\end{align*}
From our assumption $\mathbb{P}_{\lambda \sim \Lambda}\left [\lambda \notin \{ 0,1 \} \land g(\lambda) \neq \frac{1}{2}\right]>0$, we have $\mathbb{E}_{\lambda \sim \Lambda}[\min(\lambda, 1-\lambda)]|2g(\lambda)-1| > 0$. Thus, if $c \geq \frac{4\log 2}{\mathbb{E}_{\lambda \sim \Lambda}[\min(\lambda, 1-\lambda)|2g(\lambda)-1|]s}$, then $\mathbb{E}_\lambda[\lossmix (c \Sigma^{-1}\mu)] > \log2 = \mathbb{E}_\lambda[\lossmix (0)]$, so $c\mSigma^{-1}\vmu$ cannot be a minimizer. Hence, $c_{\mathrm{mix},n,\kappa}^* <\frac{4\log 2}{\mathbb{E}_{\lambda \sim \Lambda}[\min(\lambda, 1-\lambda)|2g(\lambda)-1|]s} 
 = \mathcal O(1)$, which is also independent of $n$. Combining the lower and upper bounds, we have our conclusion that $c_{\mathrm{mix},n,\kappa}^* = \Theta(1)$.
\hfill $\square$
\subsection{Proof of Lemma~\ref{lemma:minimizer_dependent}}\label{proof:minimizer_dependent}
The challenging part is that we cannot directly apply Lemma~\ref{lemma:independent_concentration} to show pointwise convergence of $\hat{F}_N(\vtheta)$ to $F_N(\vtheta)$ on $\mathcal{C}$ since $\veta_{i,j}$'s are not independent and not identically distributed. We can overcome this problem by using Lemma~\ref{lemma:dependent_concentration} which is the alternative version of Lemma~\ref{lemma:independent_concentration}. We prove Lemma~\ref{lemma:dependent_concentration} in Appendix~\ref{proof:dependent_concentration}.

By Lemma~\ref{lemma:dependent_concentration}, there exists a universal constant $C_1'>0$ such that  for any $\vtheta \in \mathcal{C}$, we have
\begin{equation*}
\mathbb{P}\left[ \left| \hat{F}_N(\vtheta) - F_N(\vtheta)\right|> \frac{\alpha \epsilon^2 }{8}\right] \leq 2 \exp \left( - \frac{ \alpha^2 \epsilon^4}{C_1' M}N\right).
\end{equation*}
Notice that from the given condition and Jensen's inequality, for each $i,j \in [N]$, 
\begin{equation*}
    \mathbb{E}_{\veta \sim \mathcal{P}_{i,j}}[g(\veta)] \leq \log \mathbb{E}_{\veta \sim \mathcal{P}_{i,j}}\left[e^{g(\veta)}\right] < \log L <L,
\end{equation*}
and from triangular inequality, we have
\begin{align}
\label{eq:lemmaminimizer_dependent-1}
 &\quad \mathbb{P}\left[ \sup_{\vtheta \in \mathcal{C}}\left\lVert \nabla_\vtheta \left( \hat{F}_N(\vtheta) - F_N(\vtheta) \right)\right\rVert > 3 L \right] \nonumber \\
 &\leq \mathbb{P}\left[ \frac{1}{N^2} \sum_{i,j=1}^N \sup_{\vtheta \in \mathcal{C}}\left\lVert \nabla_\vtheta \left( f(\vtheta, \veta_{i,j}) - \mathbb{E}[f(\vtheta, \veta_{i,j})] \right)\right\rVert > 3 L\right]\nonumber \\
 &\leq \mathbb{P} \left[ \frac{1}{N^2} \sum_{i,j=1}^N \left( g(\veta_{i,j}) + \mathbb{E}[g(\veta_{i,j})] \right)> 3 L \right] \nonumber \\
 &\leq \mathbb{P} \left[ \frac{1}{N^2} \sum_{i,j=1}^N \left( g(\veta_{i,j}) - \mathbb{E}[g(\veta_{i,j})] \right)> L \right].
\end{align}
Since we have $\mathbb{E}_{\veta \sim \mathcal{P}_{i,j}}\left [e^{ |g(\veta) - \mathbb{E}_{\veta\sim \mathcal{
P}_{i,j}}[g(\veta)]|} \right] \leq \mathbb{E}_{\veta \sim \mathcal{P}_{i,j}}\left [e^{ g(\veta)} \right] \cdot e^{\mathbb{E}_{\veta\sim \mathcal{
P}_{i,j}}[g(\veta)] }<L^2$ by our assumption, we can apply Lemma~\ref{lemma:dependent_concentration} to the RHS of \Eqref{eq:lemmaminimizer_dependent-1}. Therefore, we have
\begin{equation}
\label{eq:lemmaminimizer_dependent-3}
 \mathbb{P}\left[ \sup_{\vtheta \in \mathcal{C}}\left\lVert \nabla_\vtheta \left( \hat{F}_N(\vtheta) - F_N(\vtheta) \right)\right\rVert > 3 L \right] 
 \leq \exp \left(-\frac{1}{C_1''}N\right),
\end{equation}
where $C_1''>0$ is a universal constant. Without loss of generality, we can choose $C_1'' = C_1'$ that works for both \Eqref{eq:lemmaminimizer_dependent-1} and \Eqref{eq:lemmaminimizer_dependent-3}.

We choose $\bar{\vtheta}_1, \dots, \bar{\vtheta}_m \in \mathcal{C}$ with $m \leq \max \left\{1, \left(\frac{C_2'k^{1/2} D L}{\alpha \epsilon^2 }\right)^k \right\}$ where $C_2'>0$ is a universal constant and satisfies following: 
For any $\vtheta \in \mathcal{C}$, there exists $i_\vtheta \in [m]$ such that $\left \lVert \vtheta - \bar{\vtheta}_{i_\vtheta} \right \rVert< \frac{\alpha \epsilon^2}{24L}$. In other words, $\bar{\vtheta}_1, \dots, \bar{\vtheta}_m$ form an $\frac{\alpha \epsilon^2}{24L}$-cover of $\mathcal {C}$.

Suppose $| \hat{F}_N(\vtheta) - F_N(\vtheta) | < \frac{\alpha \epsilon^2}{8}$  for each $k \in [m]$ and $\sup_{\theta \in \mathcal{C}}\left\lVert \nabla_\vtheta \left( \hat{F}_N(\vtheta) - F_N(\vtheta) \right)\right\rVert < 3L$ which implies $\hat{F}_N(\vtheta)-F_N(\vtheta)$ is $3L$-Lipschitz. Then, for any $\vtheta \in \mathcal{C}$, we have
\begin{align*}
   &\quad \left | \hat{F}_N(\vtheta) - F_N(\vtheta) \right| \\
   &\leq \left | \hat{F}_N(\bar{\vtheta}_{i_\vtheta}) - F_N(\bar{\vtheta}_{i_\vtheta}) \right| +  \left |\left( \hat{F}_N(\vtheta) - F_N(\vtheta) \right) -  \left  (\hat{F}_N(\bar{\vtheta}_{i_\vtheta}) - F_N(\bar{\vtheta}_{i_\vtheta}) \right) \right|\\
    &\leq \frac{\alpha \epsilon^2}{8} + 3L \lVert \vtheta - \bar{\vtheta}_{i_\vtheta} \rVert < \frac{\alpha \epsilon^2}{8} + \frac{\alpha \epsilon^2}{8} = \frac{\alpha \epsilon^2}{4}.
\end{align*}
By applying union bound, we conclude
\begin{align*}
    &\quad \mathbb{P}\left[ \sup_{\vtheta \in \mathcal{C}}|\hat{F}_N(\vtheta) - F_N(\vtheta)| >  \frac{\alpha \epsilon^2}{4} \right] \\
    &\leq 2\max \left \{1, \left(\frac{C_2'k^{1/2} DL}{\alpha \epsilon^2 }\right)^k \right \} \exp\left(- \frac{ \alpha^2 \epsilon^4}{C_1' M}N \right) + \exp\left(-\frac{1}{C_1'}N \right)\\ &\leq 3 \max \left \{1, \left(\frac{C_2'k^{1/2} DL}{\alpha \epsilon^2 }\right)^k \right \} \exp \left( - \min \left\{ \frac{\alpha^2 \epsilon^4}{C_1' M}, \frac{1}{C_1'}\right \} N \right)\\
    &=3 \max \left \{1, \left(\frac{C_2'k^{1/2} D L}{\alpha \epsilon^2 }\right)^k \right \} \exp \left( -  \frac{\alpha^2 \epsilon^4}{C_1' M} N \right),
\end{align*}
where the last equality equality is due to $ \alpha \epsilon^2 \leq 1$ and $M\geq 1$ which are implied by given conditions.

Suppose $ \sup_{\vtheta \in \mathcal{C}}|\hat{F}_N(\vtheta) - F_N(\vtheta)| <  \frac{\alpha \epsilon^2}{4}$ and $\left \lVert \hat{\vtheta}_N^* - \vtheta_N^* \right \rVert > \epsilon$. Then, from the strong convexity of $F_N(\vtheta)$, we have
\begin{align*}
    &\quad \hat{F}_N(\hat{\vtheta}_N^*) - \hat{F}_N(\vtheta_N^*)\\
    &= \left( \hat{F}_N(\hat{\vtheta}_N^*) - F_N(\hat{\vtheta}_N^*) \right) + \left( F_N(\vtheta_N^*) - \hat{F}_N(\vtheta_N^*)\right) + \left(  F_N(\hat{\vtheta}_N^*) - F_N(\vtheta_N^*) )\right)\\
    &\geq -\frac{\alpha \epsilon^2 }{2} + \frac{\alpha}{2} \left \lVert \hat{\vtheta}_N^* - \vtheta_N^* \right \rVert^2 >0.
\end{align*}
This is a contradiction to the fact that $\hat{\vtheta}_N$ is a minimizer of $\hat{F}_N(\vtheta)$. Hence, we have
\begin{equation*}
    \mathbb{P} \left[ \left \lVert \hat{\vtheta}_N^* - \vtheta_N^* \right \rVert >\epsilon \right] \leq 3 \max \left \{1, \left(\frac{C_2'k^{1/2} L D}{\alpha \epsilon^2 }\right)^k \right\} \exp \left( - \frac{\alpha^2 \epsilon^4}{C_1' M}N \right),
\end{equation*}
and equivalently, if 
\begin{equation*}
    N \geq \frac{C_1' M}{ \alpha^2 \epsilon^4} \log\left(\frac{3}{\delta} \max \left\{1, \left(\frac{C_2'k^{1/2} L D}{\alpha \epsilon^2 }\right)^k \right\} \right),
\end{equation*}
then $\left \lVert \hat{\vtheta}_N^* - \vtheta_N^* \right \rVert <\epsilon$ with probability at least $1-\delta$.
\hfill $\square$

\section{Proofs for Section \ref{section:mask}}
Before moving on to proof of the main results of Section~\ref{section:mask}, We would like to represent $\mathbb{E}_\kappa [ \lossmask ( \vw )]$ in a simpler form. Note that $(\tilde{\vx}_{i,i}^\mathrm{mask},\tilde{y}_{i,i}^\mathrm{mask}) = (\vx_i, y_i)$ for each $i \in [n]$. For $i,j \in [n]$ with $i \neq j$, conditioning on $(\mM_{i,j}, \lambda_{i,j}) \sim \mathcal{M}$, we have
\begin{align*}
    \begin{cases}
    \tilde{y}_{i,j} = 1,   &\tilde{x}_{i,j}\sim N(\vmu, \mSigma \odot \left(\mM_{i,j}\mM_{i,j}^\top + (\vone-\mM_{i,j})(\vone-\mM_{i,j})^\top )\right) \hfill\text{with probability }\frac{1}{4}\\
    \tilde{y}_{i,j} = \lambda_{i,j}, &\tilde{x}_{i,j} \sim N((2\mM_{i,j}-1)\odot\vmu, \mSigma \odot \left(\mM_{i,j}\mM_{i,j}^\top + (\vone-\mM_{i,j})(\vone-\mM_{i,j})^\top )\right)\hfill\text{with probability } \frac{1}{4}\\
    \tilde{y}_{i,j} = 1-\lambda_{i,j},  &\tilde{x}_{i,j}\sim N(-(2\mM_{i,j}-1)\vmu, \mSigma \odot \left(\mM_{i,j}\mM_{i,j}^\top + (\vone-\mM_{i,j})(\vone-\mM_{i,j})^\top )\right)\hfill\text{with probability }\frac{1}{4}\\
    \tilde{y}_{i,j} = 0,  &\tilde{x}_{i,j}\sim N(-\vmu, \mSigma \odot \left(\mM_{i,j}\mM_{i,j}^\top + (\vone-\mM_{i,j})(\vone-\mM_{i,j})^\top )\right) \hfill\text{with probability } \frac{1}{4}
    \end{cases}.
\end{align*}
Let $ \support(\mM) = \left \{ \mM^{(1)}, \dots, \mM^{(p)} \right \} $ where $(\mM, \lambda ) \sim \mathcal{M} $. For each $k \in [p]$, define 
\begin{equation*}
\vmu^{(k)} = \vmu \odot (2\mM^{(k)}-1), \mSigma^{(k)} = \mSigma \odot \left(\mM^{(k)}{\mM^{(k)}}^\top + (\vone-\mM^{(k)})(\vone-\mM^{(k)})^\top\right),
\end{equation*}
and
\begin{equation*}
a_k = \frac{1}{2}\mathbb{E}_{(\mM,\lambda) \sim \mathcal{M}}\left[\lambda\vone_{\mM = \mM^{(k)}}\right], b_k = \frac{1}{2}\mathbb{E}_{(\mM,\lambda) \sim \mathcal{M}}\left[(1-\lambda)\vone_{\mM = \mM^{(k)}}\right],c_k = \frac{1}{2} \mathbb{P}_{(\mM, \lambda) \sim \mathcal{M}}\left[\mM = \mM^{(k)} \right].
\end{equation*}
By Assumption~\ref{Assumption:mask}, $a_1, \dots, a_p, b_1, \dots, b_p, c_1, \dots, c_p>0$. One can check that $a_k + b_k = c_k$ for each $k \in [p]$ and $\sum_{k=1}^p (a_k + b_k + c_k) = 1$.  Then, for each $\kappa\in (0,\infty)$, we can rewrite $\mathbb{E}_\kappa [\lossmask (\vw)]$ as the following form:
\begin{align}\label{eqn:mask}
    &\quad \mathbb{E}_\kappa [\lossmask(\vw)]\nonumber \\
    &= \frac{1}{n^2} \left( \sum_{i=1}^n \mathbb{E}_\kappa \left[y_il(\vw^\top \vx_i) + (1-y_i) l(-\vw^\top \vx_i)\right]  + \sum_{\overset{i,j \in [n]}{i \neq j}} \mathbb{E}_\kappa \left[\tilde{y}_{i,j}^\mathrm{mask} l(\vw^\top \tilde{\vx}_{i,j}^\mathrm{mask}) + (1-\tilde{y}_{i,j}^\mathrm{mask}) l(-\vw^\top \tilde{\vx}_{i,j}^\mathrm{mask})\right]\right)\nonumber \\ 
    &=\frac{1}{n}\mathbb{E}_{\rvx \sim N(\vmu, \kappa^{-1} \mSigma)} \left[l(\vw^\top \rvx)\right] \nonumber \\
    &\quad + \frac{n-1}{n} \left( \sum_{k=1}^p \mathbb{E}_{\rvx^{(k)} \sim N(\vmu^{(k)}, \kappa^{-1}\mSigma^{(k)})}\left [a_k l\left (\vw^\top \rvx^{(k)}\right) +  b_k l \left(-\vw^\top \rvx^{(k)} \right) \right] + c_k \mathbb{E}_{\rvz^{(k)} \sim N(\vmu, \kappa^{-1} \mSigma^{(k)})}\left[l\left( \vw^\top \rvz^{(k)}\right)\right]\right),
\end{align}
and
\begin{equation}\label{eqn:mask_infty}
    \mathbb{E}_\infty[\lossmask (\vw)] = \frac{1}{n} l(\vw^\top \vmu) + \frac{n-1}{n}\sum_{k=1}^p \left(a_k l\left (\vw^\top \vmu^{(k)}\right) + b_k l\left (- \vw^\top \vmu^{(k)}\right) + c_k l\left (\vw^\top \vmu\right) \right).
\end{equation}
In addition, notice that for each $k \in [p]$, we have
\begin{align}\label{eqn:mask_dist}
\begin{cases}
\tilde{\vx}_{i,j}^\mathrm{mask} \sim N \left (\vmu, \kappa^{-1} \mSigma^{(k)} \right) &\text{with probability $\frac{c_k}{2}$}\\
\tilde{\vx}_{i,j}^\mathrm{mask} \sim N \left (-\vmu, \kappa^{-1} \mSigma^{(k)} \right) &\text{with probability $\frac{c_k}{2}$}\\
\tilde{\vx}_{i,j}^\mathrm{mask} \sim N \left (\vmu^{(k)}, \kappa^{-1} \mSigma^{(k)} \right) &\text{with probability $\frac{c_k}{2}$}\\
\tilde{\vx}_{i,j}^\mathrm{mask} \sim N \left (-\vmu^{(k)}, \kappa^{-1} \mSigma^{(k)} \right) &\text{with probability $\frac{c_k}{2}$}
\end{cases}.
\end{align}

\subsection{Proof of Theorem~\ref{thm:mask_minimizer}}\label{proof:mask}
We first prove the uniqueness of a minimizer $\wmask$ of $\mathbb{E}_\kappa [\lossmask(\vw)]$ for all $\kappa \in (0, \infty]$ and show that they are bounded. Next, we prove that $\wmask$ converges to $\wmask[\infty]$ as $\kappa \rightarrow \infty$.
\paragraph{Step 1: $\mathbb{E}_\kappa [\lossmask(\vw)]$ has a Unique Minimizer $\wmask$ and $\{ \wmask : \kappa \in (0, \infty]\} $ is Bounded}\quad

We prove the strict convexity of $\mathbb{E}_\kappa [\lossmask (\vw)]$ for each $\kappa \in (0, \infty]$. Consider the case $\kappa \neq \infty$ first. For $t \in [0,1]$, and $\vw_1, \vw_2 \in \mathbb{R}^d$ with $\vw_1 \neq \vw_2$,
\begin{equation*}
    t l\left(\vw_1^\top \vx \right) + (1-t) l\left(\vw_2^\top \vx\right) >  l\left((t\vw_1 + (1-t)\vw_2)^\top \vx\right),
\end{equation*}
and
\begin{equation*}
    t l\left(-\vw_1^\top \vx\right) + (1-t) l\left(-\vw_2^\top \vx\right) >  l\left(-(t\vw_1 + (1-t)\vw_2)^\top \vx\right),
\end{equation*}
for any $\vx\in \mathbb{R}^d$ except a Lebesgue measure zero set (hyperplane orthogonal to $\vw_1 - \vw_2$).

By \Eqref{eqn:mask} and taking expectation, we have for any $t \in [0,1]$, and $\vw_1, \vw_2 \in \mathbb{R}^d$,
\begin{equation*}
    t \mathbb{E}_\kappa[\lossmask(\vw_1)] + (1-t) \mathbb{E}_\kappa[\lossmask(\vw_2)] >  \mathbb{E}_\kappa[\lossmask(t\vw_1+(1-t)\vw_2)].
\end{equation*}
For the case $\kappa = \infty$, by Assumption~\ref{Assumption:mask}, there exist $i_1, \dots, i_d \in [m]$ such that $\left \{ \vmu^{(i_1)}, \dots, \vmu^{(i_d)} \right \}$ spans $\mathbb{R}^d$ and for any $t\in [0,1]$ and $\vw_1,\vw_2 \in \mathbb{R}^d$ with $\vw_1 \neq \vw_2$, at least one of $k \in [d]$ satisfies 
\begin{equation*}
    t l\left(\vw_1^\top \vmu^{(i_k)}\right) + (1-t) l\left(\vw_2^\top \vmu^{(i_k)}\right) >  l\left((t\vw_1 + (1-t)\vw_2)^\top \vmu^{(i_k)}\right),
\end{equation*}
and
\begin{equation*}
    t l\left(- \vw_1^\top \vmu^{(i_k)}\right) + (1-t) l\left(-\vw_2^\top \vmu^{(i_k)}\right) >  l\left(-(t\vw_1 + (1-t)\vw_2)^\top \vmu^{(i_k)}\right).
\end{equation*}
From \Eqref{eqn:mask_infty}, we can conclude the strict convexity of $\mathbb{E}_\kappa [\lossmask (\vw)]$. 

In order to complete this step, we prove the existence of a ball containing all minimizers of $\mathbb{E}_\kappa[\lossmask(\vw)]$. For the case $\kappa \neq \infty$, from \Eqref{eqn:mask}, we have
\begin{equation*}
    \mathbb{E}_\kappa[\lossmask(\vw)] \geq \frac{1}{2}\sum_{k=1}^d \min\{a_{i_k}, b_{i_k}\} \mathbb{E}_{\rvx^{(k)} \sim N(\vmu^{(i_k)},\kappa^{-1} \mSigma^{(i_k)})} \left[l\left(\vw^\top \rvx^{(k)}\right) + l\left(-\vw^\top \rvx^{(k)}\right)\right].
\end{equation*}
Since $l(z) + l(-z) \geq |z|$ for each $z \in \mathbb{R}$,
\begin{align*}
    \mathbb{E}_\kappa[\lossmask(\vw)] &\geq \frac{1}{2}\sum_{k=1}^d \min\{a_{i_k},b_{i_k}\} \mathbb{E}_{\rvx^{(k)} \sim N(\vmu^{(i_k)}, \kappa^{-1} \mSigma^{(i_k)})}\left[\left|\vw^\top \rvx^{(k)}\right|\right] \\
    &\geq \frac{1}{2}\sum_{k=1}^d \min\{a_{i_k}, b_{i_k}\} \left|\mathbb{E}_{\rvx^{(k)} \sim N(\vmu^{(i_k)}, \kappa^{-1} \mSigma^{(i_k)})} \left[\vw^\top \rvx^{(k)}\right]\right|\\
    &= \frac{1}{2}\sum_{k=1}^d \min\{a_{i_k}, b_{i_k}\} \left|\vw^\top \vmu^{(i_k)}\right|.
\end{align*}
Also, for the case $\kappa = \infty$, from \Eqref{eqn:mask_infty} and using similar argument, we have
\begin{align*}
    \mathbb{E}_\infty[\lossmask(\vw)] 
    &\geq \frac{1}{2} \sum_{k=1}^d \min\{a_{i_k}, b_{i_k}\} \left(l\left(\vw^\top \vmu^{(k)}\right) + l\left(-\vw^\top \vmu^{(k)}\right)\right)\\
    &\geq \frac{1}{2}\sum_{k=1}^d \min\{a_{i_k},b_{i_k}\} \left|\vw^\top \vmu^{(i_k)}\right|.
\end{align*}
Since $\left\{\mu^{(i_1)}, \dots, \mu^{(i_d)}\right\}$ spans $\mathbb{R}^d$, for any unit vector $\vu \in \mathbb{R}^d$, $\sum_{k=1}^d \min\{a_{i_k},  b_{i_k}\} \left|\vu^\top \vmu^{(i_k)}\right| >0$ and $\vu \mapsto \sum_{k=1}^d \min\{a_{i_k},  b_{i_k}\} \left|\vu^\top \vmu^{(i_k)}\right|$ has the minimum value $m>0$ since $\{ \vu \in \mathbb{R}^d : \lVert \vu \rVert =1 \}$ is compact and the mapping is continuous. If $\lVert \vw \rVert > \frac{2\log 2}{m}$, then we have
\begin{align*}
    \mathbb{E}_\kappa[\lossmask(\vw)]  &\geq \frac{1}{2} \sum_{k=1}^d \min\{a_{i_k},  b_{i_k}\} \left|\vw^\top \vmu^{(i_k)}\right|\\ 
    &= \lVert \vw \rVert \sum_{k=1}^d \min\{a_{i_k},  b_{i_k}\} \left|\left(\frac{\vw}{\lVert \vw \rVert}\right)^\top \vmu^{(i_k)}\right|\\ 
    &\geq \frac{1}{2}\lVert \vw \rVert m \geq \log 2\\
    &= \mathbb{E}_\kappa\left[\lossmask(\vzero)\right].
\end{align*}
Hence, for any $\kappa \in (0,\infty]$, the minimizer of $\mathbb{E}_\kappa[\lossmask(\vw)]$ is contained in the ball centered at origin with radius $R:=\frac{2 \log 2}{m}$. Together with the strict convexity of $\mathbb{E}_{\kappa}[\lossmask (\vw)]$, we have our conclusion. \hfill $\square$

\paragraph{Step 2: $\wmask$  Converges to $\wmask[\infty]$ as $\kappa \rightarrow \infty$} \quad

For each $\vw \in \mathbb{R}^d$ with $\lVert \vw \rVert \leq R$ and unit vector $\vv \in \mathbb{R}^d$,  \Eqref{eqn:mask_infty} implies
\begin{align*}
 &\quad \vv^\top \nabla_\vw^2 \mathbb{E}_\infty[\lossmask (\vw) ]\vv \\
 &= \frac{1}{n} l''(\vw^\top \vmu) (\vv^\top \vmu)^2 + \frac{n-1}{n} \sum_{k=1}^p \left[ \left(a_k l''\left(\vw^\top \vmu^{(k)}\right) + b_k l''\left(-\vw^\top \vmu^{(k)}\right)\right) \left(\vv^\top \vmu^{(k)}\right)^2 + c_k l''\left(\vw^\top \vmu \right)(\vv^\top \vmu)^2 \right]\\
 &\geq \frac{1}{2} \sum_{k=1}^p (a_k + b_k) l''\left (R\lVert \vmu^{(k)}\rVert \right) \left(\vv^\top \vmu^{(k)} \right)^2.
\end{align*}
By Assumption~\ref{Assumption:mask}, at least one of $k\in[p]$ satisfies $\vv^\top \vmu^{(k)} \neq 0$ and thus, $\frac{1}{2} \sum_{k=1}^p (a_k + b_k) l''\left ( R\lVert \vmu^{(k)}\rVert \right) \left(\vv^\top \vmu^{(k)} \right)^2 >0$. Since $\vv \mapsto \frac{1}{2} \sum_{k=1}^p (a_k + b_k) l''\left ( R\lVert \vmu^{(k)}\rVert \right) \left(\vv^\top \vmu^{(k)} \right)^2$ is continuous and $\{ \vv \in \mathbb{R}^d:\lVert \vv\rVert=1 \}$ is compact, it has minimum $\alpha>0$ on this set. Hence, $\mathbb{E}_\infty \left[ \lossmask (\vw)\right]$ is $\alpha$-strongly convex on $\left \{\vw\in \mathbb{R}^d : \lVert \vw \rVert \leq R \right \}$ and since $\wmask \in \left \{\vw\in \mathbb{R}^d : \lVert \vw \rVert \leq R \right \}$ for each $\kappa \in (0,\infty)$, we have
\begin{align*}
&\quad \frac{\alpha}{2} \lVert \wmask - \wmask[\infty]\rVert^2\\
&\leq \mathbb{E}_\infty[\lossmask (\wmask)] - \mathbb{E}_\infty[\lossmask (\wmask[\infty])]\\
&\leq \mathbb{E}_\infty[\lossmask (\wmask)] - \mathbb{E}_\kappa[\lossmask (\wmask)]\\
&\quad + \mathbb{E}_\kappa[\lossmask (\wmask)] - \mathbb{E}_\kappa[\lossmask (\wmask[\infty])]\\
&\quad +\mathbb{E}_\kappa[\lossmask (\wmask[\infty])] - \mathbb{E}_\infty[\lossmask (\wmask[\infty])]\\
&\leq \mathbb{E}_\infty[\lossmask (\wmask)] - \mathbb{E}_\kappa[\lossmask (\wmask)]+\mathbb{E}_\kappa[\lossmask (\wmask[\infty])] - \mathbb{E}_\infty[\lossmask (\wmask[\infty])].
\end{align*}
The last inequality holds since $\wmask$ is a minimizer of $\mathbb{E}_\kappa [\lossmask (\vw)]$.
For any $\vw \in \mathbb{R}^d$ with $\lVert \vw \rVert \leq R$, by Lemma~\ref{lemma:ineq5}, we have
\begin{align*}
&\quad\mathbb{E}_\kappa [\lossmask (\vw)] - \mathbb{E}_\infty [\lossmask(\vw)]\\
&= \frac{1}{n} \mathbb{E}_{X \sim N(\vw^\top \vmu, \kappa^{-1} \vw^\top \mSigma \vw)}[l(X) - l(\vw^\top \vmu)] \\
&\quad + \frac{n-1}{n} \Bigg( \sum_{k=1}^p \mathbb{E}_{X^{(k)} \sim N(\vw^\top \vmu^{(k)}, \kappa^{-1} \vw^\top \mSigma^{(k)} \vw)}\left[a_k \left(l\left(X^{(k)}\right)-l\left(\vw^\top \vmu^{(k)}\right) \right) +b_k \left(l\left(-X^{(k)}\right) - l\left(- \vw^\top \vmu^{(k)}\right)\right) \right]\\
& \qquad \qquad \quad+ c_k \mathbb{E}_{X^{(k)}\sim N(\vw^\top \vmu, \kappa^{-1} \vw^\top \mSigma^{(k)}\vw )}\left[l\left(X^{(k)}\right) - l(\vw^\top \vmu)\right] \Bigg)\\
&\leq \kappa^{-1/2} \left( \frac{1}{n} (\vw^\top \mSigma \vw)^{1/2} + \frac{n-1}{n} \sum_{k=1}^p (a_k+ b_k +c_k) \left(\vw^\top \mSigma^{(k)} \vw\right)^{1/2} \right) \\
&\leq  \kappa^{-1/2} \lVert \vw \rVert \left( \frac{\lVert \mSigma \rVert^{1/2}}{n} + \frac{n-1}{n}  \sum_{k=1}^p (a_k + b_k + c_k) \left \lVert \mSigma ^{(k)}\right \rVert^{1/2}\right)\\
&\leq R \kappa^{-1/2} \left( \frac{\lVert \mSigma \rVert^{1/2}}{n} +  \frac{n-1}{n}  \sum_{k=1}^p (a_k + b_k + c_k) \left \lVert \mSigma ^{(k)}\right \rVert^{1/2}\right),\\
&\leq R \kappa^{-1/2} \left( \lVert \mSigma \rVert^{1/2} + \sum_{k=1}^p (a_k + b_k + c_k) \left \lVert \mSigma ^{(k)}\right \rVert^{1/2}\right).
\end{align*}
Therefore, 
\begin{equation}\label{eqn:mask_uniform_converge}
    \frac{\alpha}{2 }\lVert \wmask - \wmask[\infty]\rVert^2 \leq 2 R \kappa^{-1/2} \left( \lVert \mSigma \rVert + \sum_{k=1}^p (a_k + b_k + c_k) \left \lVert \mSigma ^{(k)}\right \rVert^{1/2}\right),
\end{equation}
and we conclude $\lim_{\kappa \rightarrow \infty} \wmask = \wmask[\infty]$. \hfill $\square$
\subsection{Proof of Theorem~\ref{thm:mask_convergence}}\label{proof:mask_convergence}
Since our sample complexity bound contains $\Omega (1)$ which can hide $d$, we may assume $n \geq d$ and $n \geq 2$. 
Let $R$ be the upper bound on $\{ \lVert \wmask \rVert : \kappa \in (0,\infty] \}$ which we defined in Appendix~\ref{proof:mask}. Next, define a compact set $\mathcal{C}:= \{\vw \in \mathbb{R}^d \mid \lVert \vw \rVert \leq 2R\}$, which trivially contains $\wmask$ for all $\kappa \in (0, \infty]$. For any $\vw\in \mathbb{R}^d$ and nonzero $\vv \in \mathbb{R}^d$, we have
\begin{align*}
    \vv \nabla^2_\vw \lossmask (\vw) \vv &= \frac{1}{n^2} \sum_{i,j=1}^n (\tilde{y}_{i,j}^\mathrm{mask} l''\left(w^\top \tilde{\vx}_{i,j}^\mathrm{mask}\right)\left(\vv^\top \tilde{\vx}_{i,j}^\mathrm{mask}\right)^2 +  \left(1-\tilde{y}_{i,j}^\mathrm{mask}\right)l''\left(-\vw^\top \tilde{\vx}_{i,j}^\mathrm{mask}\right)\left(\vv^\top \tilde{\vx}_{i,j}^\mathrm{mask}\right)^2)\\
    &\geq\frac{1}{n^2} \sum_{i=1}^n (y_i l''(\vw^\top \vx_i)(\vv^\top \vx_i)^2 +  (1-y_i)l''(-\vw^\top \vx_i)(\vv^\top \vx_i)^2)>0,
\end{align*}
almost surely, since $\{\vx_i\}_{i \in [n]}$ spans $\mathbb{R}^d$ almost surely. Therefore, $\lossmask (\vw)$ is strictly convex almost surely and we conclude that $\lossmask(\vw)$ has a unique minimizer $\hat{\vw}_{\mathrm{mask}, S}^*$ on $\mathcal{C}$ almost surely. Also, if $\hat{\vw}_{\mathrm{mask},S}^*$ belong to interior of $\mathcal{C}$, it is minimizer of $\lossmask(\vw)$ on $\mathbb{R}^d$. We will prove high probability convergence of $\hat{\vw}_{\mathrm{mask}, S}^*$ to $\wmask$ using Lemma~\ref{lemma:minimizer_dependent} and convert $\ell_2$ convergence into directional convergence.  For simplicity, we define
\begin{equation*}
f_{i,j}(\vw) : = \tilde{y}_{i,j}^\mathrm{mask}l(\vw^\top \tilde{\vx}_{i,j}^\mathrm{mask}) + (1-\tilde{y}_{i,j}^\mathrm{mask}) l(\vw^\top \tilde{\vx}_{i,j}^\mathrm{mask}),
\end{equation*}
for each $i,j \in [n]$.
We start with the following claim which is useful for estimating quantities described in assumptions of Lemma~\ref{lemma:minimizer_dependent} for our setting.
\begin{claim}\label{claim:mask}
For any  $t>0$, we have
\begin{equation*}
\mathbb{E}_\kappa \left[ e^{t \left \lVert \tilde{\vx}_{i,j}^\mathrm{mask}\right \rVert} \right] \leq \max \left \{  2^{d/2} + e^{4\kappa^{-1}t^2 \lVert \mSigma \rVert} ,  \sum_{k=1}^p 2c_k \left(2^{d/2} + e^{4\kappa^{-1} t^2 \left \lVert \mSigma^{(k)} \right \rVert} \right)  \right \}e^{t\lVert \vmu \rVert} ,
\end{equation*}
for all $i,j \in [n]$.
\end{claim}
\begin{proof}[Proof of Claim~\ref{claim:mask}]
We first consider the case $i = j$. By invoking triangular inequality and Lemma~\ref{lemma:ineq4}, we have
\begin{align*}
\mathbb{E}_\kappa \left[ e^{t \left \lVert \tilde{\vx}_{i,i}^\mathrm{mask}\right \rVert} \right] 
&= \mathbb{E}_\kappa \left[ e^{t \left \lVert \vx\right \rVert} \right] = \mathbb{E}_{\rvx \sim N(\vmu, \kappa^{-1} \mSigma) } \left[ e^{t \left \lVert \rvx\right \rVert} \right]\\
&= \mathbb{E}_{\rvx \sim N(\vmu, \kappa^{-1} \mSigma) } \left[ e^{t \left \lVert \rvx\right \rVert} \right] \leq \mathbb{E}_{\rvx \sim N(\vmu, \kappa^{-1} \mSigma) } \left[ e^{t \left \lVert \rvx - \vmu \right \rVert} \right] e^{t\lVert \vmu \rVert}\\
&= \mathbb{E}_{\rvz \sim N(\vzero, \kappa^{-1}t^2 \mSigma )} \left[e^{\lVert \rvz \rVert} \right] e^{t\lVert \vmu \rVert} \\
&= \left( 2^{d/2} + e^{4\kappa^{-1}t^2 \lVert \mSigma \rVert}\right)e^{t\lVert \vmu \rVert}.
\end{align*}
for each $i \in [n]$. Next, we handle the case $i \neq j$. From \Eqref{eqn:mask_dist}, for any $i,j\in[n]$ with $i \neq j$, we have
\begin{align*}
&\quad \mathbb{E}_\kappa \left[ e^{t \left \lVert \tilde{\vx}_{i,j}^\mathrm{mask}\right \rVert} \right]\\
&= \sum_{k=1}^p c_k \left( \mathbb{E}_{\rvx \sim N \left (\vmu, \kappa^{-1}\mSigma^{(k)} \right)}\left [e^{t\lVert \rvx \rVert}\right] + \mathbb{E}_{\rvx \sim N\left( \vmu^{(k)} , \kappa^{-1} \mSigma^{(k)}\right)}\left [e^{t\lVert \rvx \rVert}\right]\right)\\
&\leq \sum_{k=1}^p c_k \left( \mathbb{E}_{\rvx \sim N \left (\vmu, \kappa^{-1}\mSigma^{(k)} \right)}\left [e^{t\lVert \rvx -\vmu \rVert}\right]e^{t\lVert \vmu \rVert} + \mathbb{E}_{\rvx \sim N\left( \vmu^{(k)} , \kappa^{-1} \mSigma^{(k)}\right)}\left [e^{t\lVert \rvx -\vmu^{(k)}\rVert}\right]e^{t\lVert \vmu^{(k)}\rVert} \right)\\ 
&= \sum_{k=1}^p 2c_k \mathbb{E}_{\rvz \sim N\left (\vzero, \kappa^{-1}t^2 \mSigma^{(k)}\right)} \left [e^{\lVert \rvz \rVert} \right]e^{t \lVert \vmu \rVert}.
\end{align*}
By applying Lemma~\ref{lemma:ineq4},
\begin{equation*}
\mathbb{E}_{\rvz \sim N \left (\vzero, \kappa^{-1}t^2 \mSigma^{(k)}\right) }\left [e^{\lVert \rvz \rVert} \right] \leq 2^{d/2} + e^{4\kappa^{-1}t^2 \left \lVert \mSigma^{(k)} \right\rVert },
\end{equation*}
and we have our conclusion.
\end{proof}

\paragraph{Step 1: Estimate Upper Bound of $\mathbb{E}_\kappa \left[ e^{\left|f_{i,j}(\vw) - \mathbb{E}_\kappa[f_{i,j}(\vw)]\right|}\right]$ on $\mathcal{C}$ }\quad

For any $\vw \in \mathcal{C}$ and $i,j \in [n]$, we have
\begin{align*}
    |f_{i,j}(\vw)| &= |\tilde{y}_{i,j}^\mathrm{mask}l(\vw^\top \tilde{\vx}_{i,j}^\mathrm{mask}) + (1-\tilde{y}_{i,j}^\mathrm{mask})l(-\vw^\top \tilde{\vx}_{i,j}^\mathrm{mask})|\\
    &\leq l(-|\vw^\top \tilde{\vx}_{i,j}^\mathrm{mask}|)\\
    &\leq l(-2R \lVert \tilde{\vx}_{i,j}^\mathrm{mask} \rVert)\\
    &= \log \left( 1+e^{2R\left\lVert \tilde{x}_{i,j}^\mathrm{mask} \right\rVert} \right).
\end{align*}
By applying Claim~\ref{claim:mask} for $t = 2R$, there exists $M_\kappa'$ such that $M_\kappa' = \Theta(1)$ and
\begin{equation*}
    \mathbb{E}_\kappa \left[ e^{|f_{i,j}(\vw)|}\right] \leq \mathbb{E}_\kappa \left[ 1+ e^{2R \lVert \tilde{\vx}_{i,j}^\mathrm{mask}\rVert} \right]< M_\kappa'.
\end{equation*}
From triangular inequality and Jensen's inequality, we have
\begin{align*}
\mathbb{E}_\kappa \left[ e^{\left|f_{i,j}(\vw) - \mathbb{E}_\kappa[f_{i,j}(\vw)]\right|}\right] 
&\leq \mathbb{E}_\kappa \left[ e^{\left|f_{i,j}(\vw)\right| + \left|\mathbb{E}_\kappa[f_{i,j}(\vw)]\right|}\right]\\
&\leq \mathbb{E}_\kappa \left[ e^{\left|f_{i,j}(\vw)\right| }\right]^2\\
&\leq {M_\kappa'}^2.
\end{align*}
Letting $M_\kappa = {M'_\kappa}^2$, it follows that $M_\kappa = \Theta(1)$ and $\mathbb{E}_\kappa \left[ e^{\left|f_{i,j}(\vw) - \mathbb{E}_\kappa[f_{i,j}(\vw)]\right|}\right] < M_\kappa$ for all $\vw \in \mathcal{C}$.

\paragraph{Step 2: Estimate Upper Bound of $\lVert \nabla_\vw f_{i,j}(\vw)\rVert$ and $\lVert \nabla_\vw \mathbb{E}_\kappa[f_{i,j}(\vw)]\rVert$} \quad

For each $\vw \in \mathcal{C}$ and $i,j \in [n]$, 
\begin{align*}
\lVert \nabla_\vw f_{i,j}(\vw) \rVert 
&= \left \lVert \nabla_\vw (\tilde{y}_{i,j}^\mathrm{mask}l(\vw^\top \tilde{\vx}_{i,j}^\mathrm{mask}) + (1-\tilde{y}_{i,j}^\mathrm{mask}) l(-\vw^\top \tilde{\vx}_{i,j}^\mathrm{mask})) \right \rVert \\
&= \left \lVert \tilde{y}_{i,j}^\mathrm{mask}l'(\vw^\top \tilde{\vx}_{i,j}^\mathrm{mask})\tilde{\vx}_{i,j}^\mathrm{mask} - (1-\tilde{y}_{i,j}^\mathrm{mask}) l'(-\vw^\top \tilde{\vx}_{i,j}^\mathrm{mask})\tilde{\vx}_{i,j}^\mathrm{mask} \right \rVert\\
&\leq  \left \lVert \tilde{\vx}_{i,j}^\mathrm{mask} \right \rVert.
\end{align*}
In addition, by Lemma~\ref{lemma:gradient&hessian},
\begin{align*}
\lVert \nabla_\vw \mathbb{E}[f_{i,j}(\vw)] \rVert 
&= \left \lVert \nabla_\vw \mathbb{E}\left[(\tilde{y}_{i,j}^\mathrm{mask}l(\vw^\top \tilde{\vx}_{i,j}^\mathrm{mask}) + (1-\tilde{y}_{i,j}^\mathrm{mask}) l(-\vw^\top \tilde{\vx}_{i,j}^\mathrm{mask}))\right] \right \rVert \\
&= \left \lVert \mathbb{E}\left[\nabla_\vw(\tilde{y}_{i,j}^\mathrm{mask}l(\vw^\top \tilde{\vx}_{i,j}^\mathrm{mask}) + (1-\tilde{y}_{i,j}^\mathrm{mask}) l(-\vw^\top \tilde{\vx}_{i,j}^\mathrm{mask})) \right] \right \rVert \\
&= \left \lVert \mathbb{E}\left[ \tilde{y}_{i,j}^\mathrm{mask}l'(\vw^\top \tilde{\vx}_{i,j}^\mathrm{mask})\tilde{\vx}_{i,j}^\mathrm{mask} - (1-\tilde{y}_{i,j}^\mathrm{mask}) l'(-\vw^\top \tilde{\vx}_{i,j}^\mathrm{mask})\tilde{\vx}_{i,j}^\mathrm{mask} \right]\right \rVert\\
&\leq  \mathbb{E}\left[ \left  \lVert \tilde{\vx}_{i,j}^\mathrm{mask} \right \rVert \right].
\end{align*}
Also, by applying Claim~\ref{claim:mask} with $t=1$, there exists $L_\kappa$ such that $\mathbb{E}_\kappa \left[e^{ \left  \lVert \tilde{\vx}_{i,j}^\mathrm{mask} \right \rVert} \right]<L_\kappa$ and $L_\kappa = \Theta(1)$.

\paragraph{Step 3: Estimate Strong Convexity Constant of $\mathbb{E}_\kappa[\lossmask (\vw)]$ on $\mathcal{C}$} \mbox{}

By Lemma~\ref{lemma:gradient&hessian}, Lemma~\ref{lemma:ineq1} and \Eqref{eqn:mask},  we have 
\begin{align*}
    \vv \nabla^2_\vw \mathbb{E}_\kappa\left[\lossmask (\vw)\right] \vv &\geq \frac{n-1}{n} \sum_{k=1}^p a_k \mathbb{E}_{\rvx^{(k)} \sim N\left (\vmu^{(k)}, \kappa^{-1} \mSigma^{(k)}\right)}\left[l''\left(\vw^\top \rvx^{(k)}\right) \left(\vv^\top \rvx^{(k)}\right)^2 \right]\\
    &\geq \frac{1}{8} \min_{k\in [p]}\{a_k\} \sum_{k=1}^p \mathbb{E}_{\rvx^{(k)} \sim N\left (\vmu^{(k)}, \kappa^{-1}\mSigma^{(k)}\right)}\left[e^{-\left(\vw^\top \rvx^{(k)}\right)^2/2} \left(\vv^\top \rvx^{(k)}\right)^2 \right],
\end{align*}
for any $\vw\in \mathcal{C}$ and unit vector $\vv \in \mathbb{R}^d$.

For each $k \in [p]$,
\begin{align*}
    &\quad \vv^\top \mSigma^{(k)} \vv\\
    &= \left(\mM^{(k)} \odot \vv\right)^\top \mSigma \left(\mM^{(k)} \odot \vv\right) + \left(\left(\vone-\mM^{(k)}\right)\odot \vv\right)^\top \mSigma \left(\left(\vone-\mM^{(k)}\right) \odot \vv\right) \\
    &\geq \left\lVert \mSigma ^{-1} \right\rVert^{-1} \left \lVert \mM^{(k)} \odot \vv  \right\rVert ^2 + \left\lVert \mSigma ^{-1} \right\rVert^{-1} \left \lVert \left(\vone-\mM^{(k)} \right) \odot \vv \right \rVert ^2\\
    &= \left \lVert \mSigma^{-1} \right\rVert^{-1} \lVert \vv \rVert^2.
\end{align*}
Hence, $\mSigma^{(k)}$ is positive definite for all $k \in [p]$ and by Lemma~\ref{lemma:ineq3},
\begin{align*}
    &\quad \mathbb{E}_{\rvx^{(k)} \sim N\left (\vmu^{(k)}, \kappa^{-1}\mSigma^{(k)}\right)}\left[e^{-\left(\vw^\top \rvx^{(k)}\right)^2/2} \left(\vv^\top \rvx^{(k)}\right)^2 \right]\\
    &\geq \left(\frac{1}{2} \left(\vv^\top \vmu^{(k)}\right)^2 - \kappa^{-2} \left \lVert \mSigma^{(k)} \right \rVert^2 \lVert \vw \rVert^4 \left \lVert \vmu^{(k)} \right \rVert^2 \right)\\
    &\quad \cdot \left(\kappa^{-1} \left \lVert \mSigma^{(k)} \right \rVert\left(\kappa \left\lVert \left(\mSigma^{(k)}\right)^{-1} \right\rVert+ \lVert \vw \rVert^2 \right)\right)^{-d/2} \exp \left( -\lVert \vw \rVert^2 \left\lVert \left(\mSigma^{(k)}\right)^{-1}\right\rVert \left\lVert \mSigma^{(k)}\right\rVert \left \lVert \vmu^{(k)} \right \rVert^2 \right)\\
    &\geq \left(\frac{1}{2} \left(\vv^\top \vmu^{(k)}\right)^2 - 16\kappa^{-2} \left \lVert \mSigma^{(k)} \right \rVert^2 R^4 \left \lVert \vmu^{(k)} \right \rVert^2 \right)\\
    &\quad \cdot \left(\kappa^{-1} \left \lVert \mSigma^{(k)} \right \rVert\left(\kappa \left\lVert \left(\mSigma^{(k)}\right)^{-1} \right\rVert+ 4R^2 \right)\right)^{-d/2} \exp \left( -4R^2 \left\lVert \left(\mSigma^{(k)}\right)^{-1}\right\rVert \left\lVert \mSigma^{(k)}\right\rVert \left \lVert \vmu^{(k)} \right \rVert^2 \right)\\
    &= \left(\vv^\top \vmu^{(k)}\right)^2 \Theta(1).
\end{align*}
By Assumption~\ref{Assumption:mask}, $\sum_{k=1}^p \left(\vv^\top \vmu^{(k)}\right)^2 >0$ for each unit vector $\vv \in \mathbb{R}^d$ and since $\vv \mapsto \sum_{k=1}^p \left(\vv^\top \vmu^{(k)}\right)^2$ is continuous, we conclude that $\mathbb{E}_\kappa [\lossmask (\vw)]$ is $\alpha_k$-strongly convex on $\mathcal{C}$ where $\alpha_\kappa = \Theta(1)$. In addition, we can choose $\alpha_\kappa$ small enough so that for sufficiently large $\kappa$, $\lVert \wmask \rVert \leq R < \alpha_\kappa^{-1/2}$ since $\lVert \wmask \rVert = \Theta(1)$. This choice of $\alpha_\kappa$ makes it possible to apply Lemma~\ref{lemma:minimizer_dependent}.

\paragraph{Step 4: Lower Bounds on $\lVert \wmask \rVert$ for Sufficiently Large $n$ and $\kappa$} \quad

We need lower bounds on $\lVert \wmask \rVert$ that are independent of $n$ when we apply Lemma~\ref{lemma:minimizer_dependent} in our final step. However, finding such lower bounds is challenging since we do not know the exact direction of $\wmask$ unlike $\w$ and $\wmix$. In addition, the fact that $\lossmask(\vw)$ is dependent on $n$ also makes it hard. 
Instead, we will focus on sufficiently large $n$ and $\kappa$ and look for lower bounds independent of $n$, which is sufficient for our analysis.

We introduce a function $\mathcal{L}_\infty^\mathrm{mask} :  \R^d \to \R$ which corresponds to the limit case of $\mathbb{E}_\infty [\lossmask (\vw)]$ as $n \rightarrow \infty$ and defined as follows (i.e., the limit when both $n$ and $\kappa$ approach $\infty$):
\begin{equation}\label{eqn:mask_infty_limit}
    \mathcal{L}_\infty^\mathrm{mask}(\vw) := \lim_{n \rightarrow \infty} \mathbb{E}_{\infty}[\lossmask(\vw)]= \sum_{k=1}^p \left(a_k l\left (\vw^\top \vmu^{(k)}\right) + b_k l\left (- \vw^\top \vmu^{(k)}\right) + c_k l\left (\vw^\top \vmu\right) \right).
\end{equation}
Analyzing a minimizer of $\mathcal{L}_\infty^\mathrm{mask}(\vw)$ is helpful because it is independent of $n$ and approximates  $\lossmask(\vw)$ for large enough $n$ and $\kappa$.

Recall that we choose $i_1, \dots, i_d \in [m]$ such that $\left \{ \vmu^{(i_1)}, \dots, \vmu^{(i_d)} \right \}$ spans $\mathbb{R}^d$ in Appendix~\ref{proof:mask}. For any $t\in [0,1]$ and $\vw_1,\vw_2 \in \mathbb{R}^d$ with $\vw_1 \neq \vw_2$, at least one of $k \in [d]$ satisfies 
\begin{equation*}
    t l\left(\vw_1^\top \vmu^{(i_k)}\right) + (1-t) l\left(\vw_2^\top \vmu^{(i_k)}\right) >  l\left((t\vw_1 + (1-t)\vw_2)^\top \vmu^{(i_k)}\right),
\end{equation*}
and
\begin{equation*}
    t l\left(- \vw_1^\top \vmu^{(i_k)}\right) + (1-t) l\left(-\vw_2^\top \vmu^{(i_k)}\right) >  l\left(-(t\vw_1 + (1-t)\vw_2)^\top \vmu^{(i_k)}\right).
\end{equation*}
We can conclude the strict convexity of $\mathcal{L}_\infty^\mathrm{mask}(\vw)$. Also, since $l(z) + l(-z) \geq |z|$ for each $z \in \R$, we have
\begin{align*}
    \mathcal{L}_\infty^\mathrm{mask}(\vw)
    &\geq \frac{1}{2} \sum_{k=1}^d \min\{a_{i_k}, b_{i_k}\} \left(l\left(\vw^\top \vmu^{(k)}\right) + l\left(-\vw^\top \vmu^{(k)}\right)\right)\\
    &\geq \frac{1}{2}\sum_{k=1}^d \min\{a_{i_k},b_{i_k}\} \left|\vw^\top \vmu^{(i_k)}\right|.
\end{align*}
In Appendix~\ref{proof:mask}, we have shown that for any unit vector $\vu \in \mathbb{R}^d$, $\sum_{k=1}^d \min\{a_{i_k},  b_{i_k}\} \left|\vu^\top \vmu^{(i_k)}\right| >0$ and $\vu \mapsto \sum_{k=1}^d \min\{a_{i_k},  b_{i_k}\} \left|\vu^\top \vmu^{(i_k)}\right|$ has the minimum value $m>0$. If $\lVert \vw \rVert > R$, where we previously defined $R:=\frac{2 \log 2}{m}$, then we have
\begin{align*}
    \mathcal{L}_\infty^\mathrm{mask}(\vw)  &\geq \frac{1}{2} \sum_{k=1}^d \min\{a_{i_k},  b_{i_k}\} \left|\vw^\top \vmu^{(i_k)}\right|\\ 
    &= \lVert \vw \rVert \sum_{k=1}^d \min\{a_{i_k},  b_{i_k}\} \left|\left(\frac{\vw}{\lVert \vw \rVert}\right)^\top \vmu^{(i_k)}\right|\\ 
    &\geq \frac{1}{2}\lVert \vw \rVert m \geq \log 2\\
    &= \mathcal{L}_\infty^\mathrm{mask}(\vzero).
\end{align*}
Hence, a minimizer of $\mathcal{L}_\infty^\mathrm{mask}(\vw)$ contained in the ball centered origin with radius $R$. Together with the strict convexity, we can conclude that $\mathcal{L}_\infty^\mathrm{mask}(\vw)$ has the unique minimizer $\vw_{\mathrm{mask}}^*$ and it satisfies $\lVert \vw_{\mathrm{mask}}^* \rVert \leq R$.

In addition, we would like to prove that $\vw_{\mathrm{mask}}^*$ is nonzero. This will make our lower bounds on $\lVert \wmask \rVert$ positive. Since $\vmu$ is nonzero, without loss of generality, we assume 1st coordinate of $\vmu$, namely $\mu_1$, is nonzero. We consider a weight $r\ve_1$, where $\ve_1$ is 1st standard basis and $r>0$ which will be chosen later. We have
\begin{align*}
    &\quad \mathcal{L}_\infty^\mathrm{mask}(r\ve_1) - \mathcal{L}_\infty^\mathrm{mask}(\vzero)\\
    &= \sum_{k=1}^p \left[ a_k \left( l\left( r \ve_1^\top \vmu^{(k)} \right) - l(0) \right) + b_k \left(  l\left( -r \ve_1^\top \vmu^{(k)} \right) - l(0)\right) + c_k \left(  l\left( r \ve_1^\top \vmu \right) - l(0)\right)\right]\\
    &= \sum_{k \in \mathcal{I} }\Big[ a_k \big( l\left( r \mu_1 \right) - l(0) \big) + b_k \big(  l\left( -r \mu_1 \right) - l(0)\big) + c_k \big(  l\left( r \mu_1 \right) - l(0)\big)\Big],
\end{align*}
where $\mathcal{I} \subset [p]$ is the index set satisfying 1st coordinate of $\mM^{(k)}$ is 1 for each $k \in \mathcal{I}$. From our definition of $a_k,b_k,c_k$'s, $a_k + b_k = c_k$, thus we have
\begin{align*}
    &\quad \sum_{k \in \mathcal{I} }\Big[ a_k \big( l\left( r \mu_1 \right) - l(0) \big) + b_k \big(  l\left( -r \mu_1 \right) - l(0)\big) + c_k \big(  l\left( r \mu_1 \right) - l(0)\big)\Big]\\
    &= \sum_{k \in \mathcal{I} }\Big[ a_k \big( l\left( r \mu_1 \right) -l\left( -r \mu_1 \right) \big)  + c_k \big(  l\left( r \mu_1 \right) - l(0)\big)\Big]\\
    &= \sum_{k \in \mathcal{I} }\Big[- a_k r\mu_1  + c_k \big(  l\left( r \mu_1 \right)+l\left(- r \mu_1 \right) - 2l(0)\big)\Big]\\
    &= \sum_{k \in \mathcal{I} } c_kr\mu_1\left (  \frac{l\left( r \mu_1 \right)+l\left(- r \mu_1 \right) - 2l(0)}{r\mu_1} - \frac{a_k}{c_k}\right).
\end{align*}
Since $\lim_{z \rightarrow 0} \frac{l(z)+l(-z) - 2l(0)}{z} = 0$, we can choose $r>0$ small enough so that $\frac{l\left( r \mu_1 \right)+l\left(- r \mu_1 \right) - 2l(0)}{r\mu_1} < \frac{a_k}{c_k}$ for all $k \in \mathcal{I}$. Then, we obtain $\mathcal{L}_\infty^\mathrm{mask}(r\ve_1) - \mathcal{L}_\infty^\mathrm{mask}(\vzero)<0$ and thus $\vw_{\mathrm{mask}}^*$ is nonzero.

Next, we would like to characterize the strong convexity of  $\mathcal{L}_\infty^\mathrm{mask}(\vw)$.
For each $\vw \in \mathbb{R}^d$ with $\lVert \vw \rVert \leq R$ and unit vector $\vv \in \mathbb{R}^d$,  \Eqref{eqn:mask_infty_limit} implies
\begin{align*}
 &\quad \vv^\top \nabla_\vw^2 \mathcal{L}_\infty^\mathrm{mask}(\vw)\vv \\
 &=  \sum_{k=1}^p \left[ \left(a_k l''\left(\vw^\top \vmu^{(k)}\right) + b_k l''\left(-\vw^\top \vmu^{(k)}\right)\right) \left(\vv^\top \vmu^{(k)}\right)^2 + c_k l''\left(\vw^\top \vmu \right)(\vv^\top \vmu)^2 \right]\\
 &\geq \frac{1}{2} \sum_{k=1}^p (a_k + b_k) l''\left (R\lVert \vmu^{(k)}\rVert \right) \left(\vv^\top \vmu^{(k)} \right)^2.
\end{align*}
Recall that we have shown that $\vv \mapsto \frac{1}{2} \sum_{k=1}^p (a_k + b_k) l''\left ( R\lVert \vmu^{(k)}\rVert \right) \left(\vv^\top \vmu^{(k)} \right)^2$ has minimum $\alpha>0$ on the unit sphere $\{ \vv \in \mathbb{R}^d:\lVert \vv\rVert=1 \}$ in Appendix~\ref{proof:mask}. Hence, $\mathcal{L}_\infty^\mathrm{mask}(\vw)$ is $\alpha$-strongly convex on $\left \{\vw\in \mathbb{R}^d : \lVert \vw \rVert \leq R \right \}$. Since $ \wmask[\infty]$ and $\vw_{\mathrm{mask}}^*$ are contained in $ \left \{\vw\in \mathbb{R}^d : \lVert \vw \rVert \leq R \right \}$, we have
\begin{align*}
    &\quad \frac{\alpha}{2} \lVert \wmask[\infty] - \vw_{\mathrm{mask}}^* \rVert^2\\
    &\leq \mathcal{L}_\infty^{\mathrm{mask}}(\wmask[\infty]) - \mathcal{L}_\infty^{\mathrm{mask}}(\vw_{\mathrm{mask}}^*)\\
    &= \Big( \mathcal{L}_\infty^{\mathrm{mask}}\big(\wmask[\infty]\big)-\mathbb{E}_\infty \left [\lossmask \big (\wmask[\infty] \big) \right ]\Big) \\
    &\quad + \Big(\mathbb{E}_\infty \left [\lossmask \big (\wmask[\infty] \big) \right ]-\mathbb{E}_\infty \left [\lossmask (\vw_{\mathrm{mask}}^*  ) \right ] \Big)\\
    &\quad +\Big( \mathbb{E}_\infty \left [\lossmask\big(\vw_{\mathrm{mask}}^*\big) \right ] - \mathcal{L}_\infty^{\mathrm{mask}}\big(\vw_{\mathrm{mask}}^*\big) \Big)\\
    &\leq \Big( \mathcal{L}_\infty^{\mathrm{mask}}\big(\wmask[\infty]\big)-\mathbb{E}_\infty[\lossmask\big(\wmask[\infty]\big)]\Big)
    +\Big( \mathbb{E}_\infty[\lossmask\big(\vw_{\mathrm{mask}}^*\big)] - \mathcal{L}_\infty^{\mathrm{mask}}\big(\vw_{\mathrm{mask}}^*\big) \Big).
\end{align*}
For each $\vw \in \R^d$ with $\lVert \vw \rVert \leq R$, 
\begin{align*}
    &\quad \left|\mathcal{L}_\infty^{\mathrm{mask}}(\vw)-\mathbb{E}_\infty[\lossmask(\vw)] \right| \\
    &\leq \frac{1}{n} \left|l(\vw^\top \vmu) \right| + \frac{1}{n} \left | \sum_{k=1}^p \left(a_k l\left( \vw^\top \vmu^{(k)}\right) + b_k l \left( - \vw^\top \vmu^{(k)}  \right) + c_k l\left( \vw^\top \vmu \right)  \right) \right|\\
    &\leq \frac{1}{n}\left( l(-R \lvert \vmu \rVert) + \sum_{k=1}^p \left(a_k l\left(-R\left \lVert \vmu^{(k)} \right \rVert \right) + b_k l\left(-R\left \lVert \vmu^{(k)}\right \rVert\right) + c_k l\left( -R \lVert \vmu \rVert\right)  \right)\right)\\
    &= \frac{l(-R\lVert \vmu \rVert)}{n} \left(1 + \sum_{k=1}^p (a_k + b_k + c_k) \right)\\
    &= \frac{2l(-R\lVert \vmu \rVert) }{n}.
\end{align*}
The last two inequalities are due to $\left \lVert \vmu^{(k)}\right \rVert = \lVert \vmu \rVert$ and definition of $a_k,b_k,c_k$'s. Hence, we have
\begin{equation}\label{eqn:norm_converge}
    \frac{\alpha}{2}\lVert \wmask[\infty] - \vw_{\mathrm{mask}}^* \rVert^2 \leq \frac{4l(-R\lVert \vmu \rVert) }{n}.
\end{equation}
From triangular inequality, \Eqref{eqn:mask_uniform_converge}, and \Eqref{eqn:norm_converge}, we have
\begin{align*}
    \lVert \vw_{\mathrm{mask}}^*\rVert &\leq \lVert \wmask \rVert + \lVert \wmask[\infty] - \wmask\rVert + \lVert \vw_{\mathrm{mask}}^*-\wmask[\infty]\rVert\\
    &\leq  \lVert \wmask \rVert + \left(\frac{8l(-R\lVert \vmu \rVert)}{\alpha n}\right)^{1/2} + \left[\frac{4R}{\alpha \kappa^{1/2}} \left( \lVert \mSigma \rVert^{1/2} + \sum_{k=1}^p (a_k+b_k+c_k) \left \lVert \mSigma^{(k)} \right\rVert^{1/2} \right)\right]^{1/2}.
\end{align*}
Thus, if 
\begin{equation} \label{eqn:large_kappa}
n \geq \frac{128l(-R\lVert \vmu \rVert)}{\alpha\lVert  \vw_{\mathrm{mask}}^* \rVert^2} ,\kappa \geq \left(\frac{64R \left( \lVert \mSigma \rVert^{1/2} + \sum_{k=1}^p (a_k+b_k+c_k) \left \lVert \mSigma^{(k)} \right\rVert^{1/2} \right)}{\alpha \lVert  \vw_{\mathrm{mask}}^* \rVert^2}\right)^2,
\end{equation}
then we have $\lVert \wmask \rVert \geq \frac{\lVert \vw_{\mathrm{mask}}^* \rVert}{2}$.
Notice that the lower bounds in \Eqref{eqn:large_kappa} are numerical constants independent of $\kappa$.

\paragraph{Step 5: Sample Complexity for Directional Convergence} \quad

Assume $n$ and $\kappa$ is large enough so that satisfies \Eqref{eqn:large_kappa} and $\frac{\lVert \vw_{\mathrm{mask}}^*\rVert }{2} \epsilon \leq \lVert \wmask \rVert \epsilon < \lVert \wmask \rVert < \alpha_\kappa^{-1/2}$ We also assume the unique existence of $\hat{\vw}_{\mathrm{mask},S}^*$ which occurs almost surely. By Lemma~\ref{lemma:minimizer_dependent}, if
\begin{equation*}
    n \geq  \frac{16 C_1'M_\kappa}{\alpha_\kappa^2 \lVert \vw_{\mathrm{mask}}^* \rVert^4 \epsilon^4} \log \left( \frac{3}{\delta} \max \left\{ 1, \left( \frac{8C_2'd^{1/2}RL_\kappa}{\alpha_\kappa \lVert \vw_{\mathrm{mask}} ^* \rVert \epsilon^2} \right)^d  \right\} \right) = \frac{\Theta(1)}{\epsilon^4} \left (1+ \log \frac{1}{\epsilon} + \log \frac{1}{\delta} \right ),
\end{equation*}
then we have $\lVert\wmask - \hat{\vw}_{\mathrm{mask},S}^* \rVert \leq  \frac{\lVert \vw_{\mathrm{mask}}^*\rVert }{2}  \epsilon \leq \lVert \wmask \rVert \epsilon$ with probability at least $1-\delta$.
Furthermore, if $\lVert \wmask - \hat{\vw}_{\mathrm{mask},S}^* \rVert \leq  \lVert \wmask \rVert \epsilon$, then $\hat{\vw}_{\mathrm{mask},S}^*$ belongs to interior of $\mathcal{C}$ therefore, $\hat{\vw}_{\mathrm{mask},S}^*$ is a minimizer of $\lossmask(\vw)$ over the entire $\mathbb{R}^d$. Also, we have
\begin{align*}
    \cosim ( \hat{\vw}_{\mathrm{mask},S}^*,\wmask)&=\bigg (1- \sin^2\Big(\angle\big(\hat{\vw}_{\mathrm{mask},S}^*, \wmask\big)\Big)\bigg)^{1/2}\\
    &\geq 1- \sin\Big(\angle\big(\hat{\vw}_{\mathrm{mask},S}^*, \wmask\big)\Big)\\
    &\geq 1- \epsilon.
\end{align*}
Hence, we conclude that if $n = \frac{\Omega(1)}{\epsilon^4} \left (1+ \log \frac{1}{\epsilon} + \log \frac{1}{\delta} \right )$, then with probability at least $1-\delta$,  the masking based Mixup loss $\lossmask(\vw)$ has a unique minimizer $\hat{\vw}_{\mathrm{mask},S}^*$ and $\cosim ( \hat{\vw}_{\mathrm{mask},S}^*, \wmask)\geq 1- \epsilon$. \hfill $\square$

\section{Technical Lemmas}\label{proof:lemmas}
In this section, we introduce several technical lemmas that previously appeared. Before moving on, we define some additional notation which will be used in this section. 
\paragraph{Notation.}
For each (random) vector $\vu \in \mathbb{R}^k$, we use $(\vu)_i$ to refer to the $i$th component of $\vu$ and for each matrix $\mM \in \mathbb{R}^{k \times k}$, we use $(\mM)_i$ to represent $i$th diagonal entry of $\mM$ for $i \in [k]$. In addition, we use $\lVert \vu \rVert_\mM$ for $\left(\vu^\top\mM \vu\right)^{1/2}$.
\subsection{The Bayes Optimal Classifier for $\mathcal{D}_\kappa$}\label{proof:Bayes_optimal}
To get the Bayes optimal classifier for $\mathcal{D}_\kappa$ for $\kappa \in (0,\infty)$, we have to solve the following problem:
\begin{equation}\label{eqn:bayes_optimal}
    \min_{\vw \in \mathbb{R}^d}  \mathbb{P}_{(\vx,y) \sim \mathcal{D}_\kappa }\left[(2y-1)\vw^\top \vx >0 \right].  
\end{equation}
From our definition of $\mathcal{D}_\kappa$ we have
\begin{align*}
\mathbb{P}_{(\vx,y) \sim \mathcal{D}_\kappa }\left[(2y-1)\vw^\top \vx >0 \right] &= \mathbb{P}_{\rvx \sim N(\vmu, \kappa^{-1} \mSigma) }\left[\vw^\top \rvx>0 \right]\\
&= \mathbb{P}_{Z \sim N(0,1)} \left[\kappa^{-1/2}\lVert \vw \rVert
_\mSigma Z + \vw^\top \vmu >0 \right]\\
&= \mathbb{P} \left[ Z >  - \frac{\kappa^{1/2} \vw^\top \vmu }{\lVert \vw \rVert_\mSigma}\right].
\end{align*}
Without loss of generality, consider the solution of Problem (\ref{eqn:bayes_optimal}) with $\vw^\top \vmu = 1$. We can change Problem (\ref{eqn:bayes_optimal}) to the problem $\min_{\vw^\top \vmu = 1} \vw^\top \mSigma \vw$ and by Lemma~\ref{lemma:opt}, we have our conclusion. \hfill $\square$

\subsection{Interchanging Differentiation and Expectation}\label{appendix:leibniz}
We will introduce technical results related to interchanging differentiation and expectation. The following Lemma \ref{lemma:leibniz} is a slight variant of Leibniz's rule.
\begin{lemma}\label{lemma:leibniz}
Let $f : U \times \mathbb{R}^k \rightarrow \mathbb{R}$ be a function where $U$ is an open subset of $\mathbb{R}$. Suppose a probability distribution $\mathcal{P}$ on $\mathbb{R}^k$  satisfies the following conditions:
\begin{enumerate}
\item $\mathbb{E}_{\veta \sim \mathcal{P}}[f(\theta, \veta)] < \infty$ for all $\theta \in \mathbb{R}$.
\item For any $\theta \in U$, $\frac{\partial}{\partial \theta}f(\theta,\veta)$ exists for every $\veta \in \mathbb{R}$.
\item There is $g: \mathbb{R}^k \rightarrow \mathbb{R}$ such that $|\frac{\partial}{\partial \theta}f(\theta,\veta)| \leq g(\veta)$ for each $\theta \in U$ and $\veta \in \mathbb{R}^k$. In addition, $ \mathbb{E}_{\veta \sim \mathcal{P}}[g(\veta)] < \infty$.
\end{enumerate}
Then, $\frac{d}{d\theta} \mathbb{E}_{\veta \sim \mathcal{P}}[f(\theta, \veta)] = \mathbb{E}_{\veta \sim \mathcal{P}}\left[ \frac{\partial}{\partial \theta}f(\theta,\veta)\right]$ for all $\theta \in U$.
\end{lemma}

\begin{proof}[Proof of Lemma~\ref{lemma:leibniz}]
Fix any $\theta \in U$ and let $\{h_m\}_{m \in \mathbb{N}}$ be any sequence of nonzero real numbers such that $h_m \rightarrow 0$ as $m \to \infty$ and $\theta + h_m \in U$.
Define $f_{\theta,m} : \mathbb{R} \rightarrow \mathbb{R}$ as $f_{\theta,m}(\veta) = \frac{1}{h_m}\left(f\left( \theta + h_m,\veta \right) -f\left( \theta,\veta \right) \right)$. Then, $f_{\theta,m}(\veta) \rightarrow \frac{\partial}{\partial \theta} f(\theta,\veta)$ as $m \to \infty$. 
Therefore, for large enough $m, \left | \frac{\partial}{\partial \theta} f(\theta,\veta) - f_{\theta,m}(\veta)\right|<1$ holds, which implies $|f_{\theta,m}(\veta)| \leq g(\veta)+1$ and $\mathbb{E}_{\veta \sim \mathcal{P}}[g(\veta) +1] = \mathbb{E}_{\veta \sim \mathcal{P}}[g(\veta)] +1 < \infty$. 
Then, by dominated convergence theorem,
\begin{align*}
\lim_{m \rightarrow \infty} \frac{\mathbb{E}_{\veta \sim \mathcal{P}}[f(\theta+h_m,\veta)] - \mathbb{E}_{\veta \sim \mathcal{P}} [f(\theta,\veta)]}{h_m} 
&= \lim_{m \rightarrow \infty} \frac{\mathbb{E}_{\veta \sim \mathcal{P}} [f(\theta+h_m,\veta) - f(\theta,\veta)]}{h_m}\\
&= \lim_{m \rightarrow \infty} \mathbb{E}_{\veta \sim \mathcal{P}} [f_{\theta,m}(\veta)]\\
&= \mathbb{E}_{\veta \sim \mathcal{P}} \left[\frac{\partial}{\partial \theta} f(\theta,\veta)\right].
\end{align*}
Since our choice of $\{h_m\}_{m \in \mathbb{N}}$ is arbitrary, $\frac{d}{d\theta} \mathbb{E}_{\veta \sim \mathcal{P}}[f(\theta, \veta)] = \mathbb{E}_{\veta \sim \mathcal{P}} \left[\frac{\partial}{\partial \theta} f(\theta,\veta)\right]$.
\end{proof} 

By applying Lemma~\ref{lemma:leibniz}, we can obtain the following lemma which makes us possible to investigate stationary points and strong convexity constants of expected losses in the proof of our main theorems.

\begin{lemma}\label{lemma:gradient&hessian}
For any vector $\vu \in \mathbb{R}^k$, positive definite matrix $\mM \in \mathbb{R}^{k \times k}$, functions $f :[0,1] \rightarrow [-1,1], g: [0,1] \rightarrow [0,1]$ and probability distribution $\mathcal{P}$ with support contained in $[0,1]$, we have
\begin{equation*}
    \nabla_\vw \mathbb{E}_{\eta \sim \mathcal{P}} \left [\mathbb{E}_{\rvx \sim N(f(\eta)\vu, g(\eta)\mM)}[l(\vw^\top \rvx)] \right ]
    = \mathbb{E}_{\eta \sim \mathcal{P}} \left [\mathbb{E}_{\rvx \sim N(f(\eta)\vu, g(\eta)\mM)}[l'(\vw^\top \rvx)\rvx \right ],
\end{equation*}
and
\begin{equation*}
    \nabla_\vw^2 \mathbb{E}_{\eta \sim \mathcal{P}} \left [\mathbb{E}_{\rvx \sim N(f(\eta)\vu, g(\eta)\mM)}[l(\vw^\top \rvx)] \right ]
    = \mathbb{E}_{\eta \sim \mathcal{P}} \left [\mathbb{E}_{\rvx \sim N(f(\eta)\vu, g(\eta)\mM)}[l''(\vw^\top \rvx) \rvx \rvx^\top] \right ].
\end{equation*}
\end{lemma}
\begin{proof}[Proof of Lemma~\ref{lemma:gradient&hessian}]
    For each $i \in [k]$ and $\vx \in \mathbb{R}^k$, we have
    \begin{equation*}
        \left| \frac{\partial}{\partial (\vw)_i} l(\vw^\top \vx) \right| = |l'(\vw^\top \vx) (\vx)_i| = \left | \frac{(\vx)_i}{1+e^{\vw^\top \vx}} \right| \leq |(\vx)_i|.
    \end{equation*}
    Since $\mathbb{E}_{\rvx \sim N(f(\eta)\vu, g(\eta)\mM)}[|(\rvx)_i|]<\infty$ for each $\eta \in [0,1]$, by Lemma~\ref{lemma:leibniz},
    \begin{equation*}
        \frac{\partial}{\partial (\vw)_i} \mathbb{E}_{\rvx \sim N(f(\eta)\vu, g(\eta)\mM)}[l(\vw^\top \rvx)] 
        = \mathbb{E}_{\rvx \sim N(f(\eta)\vu, g(\eta)\mM)}[l'(\vw^\top \rvx)(\rvx)_i].
    \end{equation*}
    Also, 
    \begin{align*}
        \mathbb{E}_{\rvx \sim N(f(\eta)\vu, g(\eta)\mM)}[|(\rvx)_i|] &\leq \mathbb{E}_{\rvx \sim N(f(\eta)\vu, g(\eta)\mM)}[(\rvx)_i^2]^{1/2} \\
        &= \left(g(\eta) (\mM)_i + f(\eta)^2 (\vu)_i^2 \right)^{1/2}\\
        &\leq \left((\mM)_i + (\vu)_i^2 \right)^{1/2}.
    \end{align*}
    Hence, $\mathbb{E}_{\eta \sim \mathcal{P}}\left[\mathbb{E}_{\rvx \sim N(f(\eta)\vu, g(\eta)\mM)}[|(\rvx)_i|] \right]< \infty$ because the inner expectation is uniformly bounded for all $\eta \in [0,1]$. Applying Lemma~\ref{lemma:leibniz} again, we have
    \begin{equation*}
        \frac{\partial}{\partial (\vw)_i}\mathbb{E}_{\eta \sim \mathcal{P}}\left [\mathbb{E}_{\rvx \sim N(f(\eta)\vu, g(\eta) \mM)}[l(\vw^\top \rvx)]\right ] 
        = \mathbb{E}_{\eta \sim \mathcal{P}}\left [\mathbb{E}_{\rvx \sim N(f(\eta)\vmu, g(\eta) \kappa^{-1} \mSigma)}[l'(\vw^\top \rvx)(\rvx)_i] \right ],
    \end{equation*}
    and we conclude
    \begin{equation*}
        \nabla_\vw \mathbb{E}_{\eta \sim \mathcal{P}} \left [\mathbb{E}_{\rvx \sim N(f(\eta)\vu, g(\eta)\kappa^{-1}\mM)}[l(\vw^\top \rvx)] \right ]
        = \mathbb{E}_{\eta \sim \mathcal{P}} \left [\mathbb{E}_{\rvx \sim N(f(\eta)\vu, g(\eta)\kappa^{-1}\mM)}[l'(\vw^\top \rvx)\rvx] \right ].
    \end{equation*}
    For each $i,j \in [k]$ and $\vx \in \mathbb{R}^k$,
    \begin{equation*}
        \left | \frac{\partial^2}{\partial (\vw)_j \partial (\vw)_i}l(\vw^\top \vx) \right|  =  \left |l''(\vw^\top \vx) (\vx)_i (\vx)_j \right|\leq \frac{|(\vx)_i(\vx)_j|}{4}.
    \end{equation*}
    Since $\mathbb{E}_{\rvx \sim N(f(\eta)\vu, g(\eta)\mM)}[|(\rvx)_i(\rvx)_j|]<\infty$ for each $\eta \in [0,1]$, by Lemma~\ref{lemma:leibniz}, we have
    \begin{align*}
         \frac{\partial^2}{\partial (\vw)_j \partial (\vw)_i}\mathbb{E}_{\rvx \sim N(f(\eta)\vu, g(\eta)\mM)}[l(\vw^\top \rvx)] &= \frac{\partial}{\partial (\vw)_j}\mathbb{E}_{\rvx \sim N(f(\eta)\vu, g(\eta)\mM)}\left [\frac{\partial}{\partial (\vw)_i}l(\vw^\top \rvx)\right] \\
         &= \mathbb{E}_{\rvx \sim N(f(\eta)\vu, g(\eta)\mM)}\left [\frac{\partial^2}{\partial (\vw)_j \partial (\vw)_i}l(\vw^\top \rvx) \right ]\\
         &= \mathbb{E}_{\rvx \sim N(f(\eta)\vu, g(\eta)\mM)}[l''(\vw^\top \rvx)(\rvx)_i (\rvx)_j].
    \end{align*}
    Also,  $\left| \frac{\partial^2}{\partial (\vw)_j \partial (\vw)_i}\mathbb{E}_{\rvx \sim N(f(\eta)\vu, g(\eta)\mM)}[l(\vw^\top \rvx)] \right| \leq  \frac{1}{4}\mathbb{E}_{\rvx \sim N(f(\eta)\vu, g(\eta)\mM)}[|(\rvx)_i(\rvx)_j|]$ and
    \begin{align*}
        \frac{1}{4}\mathbb{E}_{\eta \sim \mathcal{P}} \left [\mathbb{E}_{\rvx \sim N(f(\eta)\vu, g(\eta)\mM)}[|(\rvx)_i(\rvx)_j|] \right ]
        &\leq \frac{1}{8} \mathbb{E}_{\eta \sim \mathcal{P}} \left [\mathbb{E}_{\rvx \sim N(f(\eta)\vu, g(\eta)\mM)}[(\rvx)_i^2 + (\rvx)_j^2] \right ]\\
        &= \frac{1}{8} \mathbb{E}_{\eta \sim \mathcal{P}} \left [\mathbb{E}_{\rvx \sim N(f(\eta)\vu, g(\eta)\mM)}[g(\eta) ((\mM)_i + (\mM)_j) + f(\eta)^2 ((\vu)_i^2 + (\vu)_j^2)] \right ]\\
        &\leq \frac{1}{8} ( \Tr (\mM) + \lVert \vu \rVert^2) <\infty.
    \end{align*}
By applying Lemma~\ref{lemma:leibniz} again, for each $i,j \in [k]$, 
\begin{equation*}
    \frac{\partial^2}{\partial (\vw)_j \partial (\vw)_i}\mathbb{E}_{\eta \sim \mathcal{P}}\left [\mathbb{E}_{\rvx \sim N(f(\eta)\vu, g(\eta)\mM)}[l(\vw^\top \rvx)] \right ] = \mathbb{E}_{\eta \sim \mathcal{P}}\left[\mathbb{E}_{\rvx \sim N(f(\eta)\vu, g(\eta)\mM)}\left[\frac{\partial^2}{\partial (\vw)_j \partial (\vw)_i}l(\vw^\top \rvx)\right]\right],
\end{equation*}
and we conclude 
\begin{equation*}
    \nabla_\vw^2 \mathbb{E}_{\eta \sim \mathcal{P}} \left [\mathbb{E}_{\rvx \sim N(f(\eta)\vu, g(\eta)\mM)}[l(\vw^\top \rvx)] \right ]
    = \mathbb{E}_{\eta \sim \mathcal{P}} \left [\mathbb{E}_{\rvx \sim N(f(\eta)\vu, g(\eta)\mM)}[l''(\vw^\top \rvx)\rvx \rvx^\top ] \right ].
\end{equation*}
\end{proof}

\subsection{Inequalities}
Let us introduce some inequalities which are used in the proof of main theorems.

The following lemma provides us computable lower bound on the strong convexity constant.
\begin{lemma} \label{lemma:ineq1}
For any $z \in \mathbb{R}$,
\begin{equation*}
    l''(z) = \frac{e^z}{(e^z+1)^2} \geq \frac{1}{4} e^{-z^2/2}.
\end{equation*}
\end{lemma}
\begin{proof}[Proof of Lemma~\ref{lemma:ineq1}]
Define a function $f: \mathbb{R} \rightarrow \mathbb{R}$ as $f(z) = \frac{1}{2} z^2 - 2\log(1+e^z)$ for each $z \in \mathbb{R}$. Then, we have
\begin{equation*}
f'(z) = z - \frac{2 e^z}{1+e^z}, f''(z) = 1- \frac{2e^z}{(e^z+1)^2} > 0,
\end{equation*}
for each $z \in \mathbb{R}$. Hence, $f$ is a convex function and $z \mapsto -z - 2 \log 2$ is a tangent line of the graph of $f(z)$ at $z=0$. Therefore, for each $z \in \mathbb{R}$, we have
\begin{equation*}
     f(z) = \frac{1}{2}z^2 - 2 \log (1+e^z) \geq -z - 2\log 2,
\end{equation*}
which implies 
\begin{equation*}
    \frac{1}{2}z^2 \geq -z +2\log(1+e^z)-2\log 2 = -\log \left(  \frac{4e^z}{(e^z+1)^2}\right) .
\end{equation*}
Thus, we conclude $l''(z) = \frac{e^z}{(e^z+1)^2} \geq \frac{1}{4} e^{-z^2/2}$ for each $z \in \mathbb{R}$.
\end{proof}

When we get lower bounds on $\lVert \w \rVert$ and $\lVert \wmask \rVert$, the following lemma is useful.
\begin{lemma} \label{lemma:ineq2}
For any $z \in \mathbb{R}$,
\begin{equation*}
    \frac{z}{1+e^z} \geq \frac{1}{2} z - \frac{1}{4}z^2.
\end{equation*}
\end{lemma}
\begin{proof}[Proof of Lemma~\ref{lemma:ineq2}]
Define a function $f: \mathbb{R} \rightarrow \mathbb{R}$ as $f(z) = \frac{1}{1+e^z}$ for each $z \in \mathbb{R}$. We have
\begin{equation*}
f'(z) = -\frac{e^z}{(1+e^z)^2}, f''(z) = \frac{e^z (e^z-1)}{(1+e^z)^3}.
\end{equation*}
Therefore, $f$ is convex on $[0, \infty)$ and concave on $(-\infty, 0]$. Since $z \mapsto -\frac{1}{4}z + \frac{1}{2}$ is tangent line of $f(z)$ at $z=0$, we have
\begin{equation*}
f(z) = \frac{1}{1+e^z} \geq -\frac{1}{4}z + \frac{1}{2},
\end{equation*}
for any $z \geq 0$ and
\begin{equation*}
f(z) = \frac{1}{1+e^z} \leq -\frac{1}{4}z + \frac{1}{2},
\end{equation*}
for any $z \leq 0$. By multiplying $z$ to the inequalities above, we have our conclusion.
\end{proof}
Together with Lemma~\ref{lemma:ineq1}, the following lemma provide us lower bounds on strong convexity constant of expected loss $\mathbb{E}_\kappa [\loss (\vw)]$, $\mathbb{E}_\kappa [\lossmix (\vw)]$ and $\mathbb{E}_\kappa [\lossmask (\vw)]$.  
\begin{lemma} \label{lemma:ineq3}
Let $\vu \in \mathbb{R}^k$ be a vector and $\mM \in \mathbb{R}^{k\times k}$ be a positive definite matrix. For any vector $\vw \in \mathbb{R}^k$ and unit vector $\vv \in \mathbb{R}^k$, we have
\begin{align*}
\mathbb{E}_{\rvx \sim N(\vu,\mM)}\left [(\vv^\top \rvx)^2e^{- (\vw^\top \rvx)^2/2} \right]
\geq &\max \left \{ (\lVert \mM^{-1}\rVert +\lVert \vw \rVert^2)^{-1} ,\frac{1}{2} (\vv^\top \vu)^2 - \lVert \mM\rVert^2 \lVert \vw\rVert^4 \lVert \vu \rVert^2  \right \}\\
&\quad \cdot \left (\lVert \mM \rVert (\lVert \mM^{-1}\rVert + \lVert \vw\rVert^2 ) \right )^{-k/2} \exp \left( -\lVert \vw \rVert^2 \lVert \mM^{-1} \rVert \lVert \mM \rVert \lVert \vu\rVert^2\right).
\end{align*}
\end{lemma}
\begin{proof}[Proof of Lemma~\ref{lemma:ineq3}]
By changing expectation to integral form, we have
\begin{equation*}
\mathbb{E}_{\rvx \sim N(\vu,\mM)} \left[ (\vv^\top \rvx)^2 e^{-(\vw^\top \rvx)^2/2}\right]= \int_{\mathbb{R}^k} (2 \pi)^{-k/2} \underbrace{\det(\mM)^{-1/2}}_{\mathrm{(a)}} (\vv^\top \vx)^2 \underbrace{\exp \left(-\frac{1}{2} \left(\lVert \vx-\vu\rVert_{\mM^{-1}}^2 + (\vw^\top \vx)^2 \right)\right)}_{\mathrm{(b)}} d\vx.
\end{equation*}
We can rewrite the term (a) as
\begin{align*}
    \det (\mM)^{-1/2} &= \det(\mM(\mM^{-1} +\vw\vw^\top)(\mM^{-1} +\vw\vw^\top)^{-1})^{-1/2}\\
    &= \det ((\mI+\mM\vw\vw^\top)(\mM^{-1} +\vw\vw^\top)^{-1})^{-1/2} \\
    &= \det(\mI+\mM\vw\vw^\top)^{-1/2} \det (\mM^{-1} + \vw\vw^\top )^{1/2},
\end{align*}
and the term (b) as
\begin{align*}
    &\quad \exp \left(-\frac{1}{2} \left(\lVert \vx-\vu\rVert_{\mM^{-1}}^2 + (\vw^\top \vx)^2 \right)\right)\\
    &= \exp \left(-\frac{1}{2} \left( \lVert \vx \rVert_{\mM^{-1} + \vw\vw^\top} - 2\vu^\top \mM^{-1}\vx + \lVert \vu \rVert_{\mM^{-1}} \right) \right)\\
    &=  \exp \left(-\frac{1}{2} \left( \lVert \vx \rVert_{\mM^{-1} + \vw\vw^\top}^2 - 2\vu^\top \mM^{-1} (\mM^{-1} + \vw\vw^\top)^{-1} (\mM^{-1} + \vw\vw^\top) \vx + \lVert \vu \rVert_{\mM^{-1}}^2 \right) \right)\\
    &= \exp \left(-\frac{1}{2} \left( \lVert \vx \rVert_{\mM^{-1} + \vw\vw^\top}^2 - 2\vu^\top  (\mI + \vw\vw^\top \mM)^{-1} (\mM^{-1} + \vw\vw^\top) \vx + \lVert \vu \rVert_{\mM^{-1}}^2 \right) \right)\\
    &= \exp \left( -\frac{1}{2} \left ( \left \lVert \vx- \left (\mI+\mM\vw\vw^\top \right)^{-1} \vu \right \rVert_{\mM^{-1}+ \vw\vw^\top}^2 
    + \lVert \vu \rVert_{\mM^{-1}}^2 - \left \lVert \left(\mI+\mM\vw\vw^\top \right)^{-1}\vu  \right \rVert_{\mM^{-1} + \vw\vw^\top}^2 \right)\right)\\
    &= \exp \left(-\frac{1}{2} \left \lVert \vx- \left (\mI+\mM\vw\vw^\top \right )^{-1}u \right \rVert_{\mM^{-1} + \vw \vw^\top}^2 \right) \cdot \exp \bigg(-\frac{1}{2}  \underbrace{\left(\lVert \vu \rVert_{\mM^{-1}}^2 - \left \lVert  (\mI+ \mM \vw\vw^\top )^{-1} \vu\right \rVert_{\mM^{-1} + \vw\vw^\top}^2 \right)}_{\mathrm{(c)}}\bigg).
\end{align*}
Also, the term (c) can be simplified as
\begin{align*}
&\quad \lVert \vu \rVert_{\mM^{-1}}^2 - \lVert (\mI+\mM\vw\vw^\top)^{-1}\vu \rVert_{\mM^{-1} + \vw\vw^\top}^2\\
&= \vu^\top \mM^{-1} \vu - ((\mI+\mM\vw\vw^\top)^{-1}\vu)^\top (\mM^{-1} + \vw\vw^\top )((\mI+\mM\vw\vw^\top)^{-1}\vu)\\
&= \vu^\top \mM^{-1} \vu - (\vu^\top (\mI+\vw\vw^\top \mM)^{-1}
) (\mM^{-1} (\mI + \mM \vw\vw^\top ))((\mI+\mM\vw\vw^\top)^{-1}\vu)\\
&= \vu^\top \mM^{-1} \vu - \vu^\top (\mI+\vw\vw^\top \mM)^{-1} \mM^{-1}\vu\\
&= \vu^\top (\mI+\vw\vw^\top \mM) (\mI+\vw\vw^\top \mM)^{-1} \mM^{-1} \vu - \vu^\top (\mI+ \vw\vw^\top \mM)^{-1}\mM^{-1}\vu\\
&= \vu^\top (\vw\vw^\top \mM)(\mI+\vw\vw^\top \mM)^{-1} \mM^{-1}\vu\\
&=  \vu^\top (\vw\vw^\top\mM)(\mM+\mM \vw\vw^\top \mM)^{-1}\vu\\
&= \vu^\top (\vw\vw^\top \mM ) ((\mI + \mM\vw\vw^\top) \mM)^{-1} \vu\\
&= \vu^\top \vw\vw^\top (\mI+ \mM\vw\vw^\top)^{-1} \vu.
\end{align*}
Therefore, we have 
\begin{align*}
&\quad \mathbb{E}_{\rvx \sim N(\vu,\mM)} \left[ (\vv^\top \rvx)^2 e^{-(\vw^\top \rvx)^2/2}\right]\\
&= \underbrace{ \int_{\mathbb{R}^d} (2\pi)^{-d/2} \det(\mM^{-1} + \vw \vw^\top)^{1/2} (\vv^\top \vx)^2 \exp\left( -\frac{1}{2} \left(\lVert \vx-(\mI+\mM\vw\vw^\top)^{-1}\vu\right)\rVert _{\mM^{-1}+ \vw\vw^\top}^2 \right) d\vx}_{\mathrm{(d)}}\\
& \quad \cdot \det (\mI+\mM\vw\vw^\top)^{-1/2} \exp \left( -\frac{1}{2} \vu^\top \vw\vw^\top  (\mI+\mM\vw\vw^\top)^{-1} \vu \right).
\end{align*}
By changing integral into expectation, we have
\begin{align*}
\mathrm{(d)} &= \mathbb{E}_{\rvx \sim N\left( (\mI+\mM\vw\vw^\top)^{-1}\vu, (\mM^{-1}+ \vw\vw^\top)^{-1} \right)}\left[\left(\vv^\top \rvx\right)^2\right]\\
&= \mathbb{E}_{\rz \sim N\left( \vv^\top(\mI+\mM\vw\vw^\top)^{-1}\vu, \vv^\top(\mM^{-1}+ \vw\vw^\top)^{-1}\vv \right)}\left[\rz^2\right].
\end{align*}
From $\mathbb{E}_{\rz \sim N(m, \sigma^2)}[\rz^2] = m^2+\sigma^2$ for each $m \in \mathbb{R}, \sigma >0$, note that
\begin{align*}
&\quad \mathbb{E}_{\rz \sim N\left( \vv^\top(\mI+\mM\vw\vw^\top)^{-1}\vu, \vv^\top(\mM^{-1}+ \vw \vw^\top)^{-1}\vv \right)}[\rz^2]\\
&\geq \vv^\top(\mM^{-1}+ \vw\vw^\top)^{-1}\vv \geq \lVert \mM^{-1} + \vw \vw^\top\rVert^{-1}\\
&\geq \left( \lVert \mM^{-1} \rVert + \lVert \vw\rVert^2 \right)^{-1},
\end{align*}
and
\begin{align*}
&\mathbb{E}_{\rz \sim N\left( \vv^\top(\mI+\mM\vw\vw^\top)^{-1}\vu, \vv^\top(\mM^{-1}+ \vw \vw^\top)^{-1} \vv \right)}[\rz^2]\\
&\geq ( \vv^\top(\mI+\mM\vw\vw^\top)^{-1}\vu)^2 = (\vv^\top \vu - \vv^\top \mM\vw\vw^\top (\mI + \mM \vw \vw^\top)^{-1}\vu)^2\\ 
&\geq \frac{1}{2} (\vv^\top \vu)^2 - (\vv^\top \mM \vw\vw^\top (\mI + \mM \vw \vw^\top)^{-1} \vu)^2 \\
&\geq \frac{1}{2}  (\vv^\top \vu)^2 - \lVert \mM \vw \vw^\top \rVert^2 \lVert(\mI + \mM \vw \vw^\top)^{-1} \rVert \lVert \vu \rVert^2\\
&\geq \frac{1}{2} (\vv^\top \vu)^2 - \lVert \mM\rVert^2 \lVert \vw\rVert^4 \lVert \vu \rVert^2.
\end{align*}
Also,
\begin{equation*}
\det(\mI+\mM \vw \vw^\top) = \det(\mM) \det(\mM^{-1} + \vw \vw^\top) \leq (\lVert \mM\rVert\lVert \mM^{-1} + \vw\vw^\top\rVert)^k \leq (\lVert \mM\rVert (\lVert \mM^{-1}\rVert + \lVert \vw\rVert^2))^k,
\end{equation*}
and
\begin{align*}
    \vu^\top \vw\vw^\top (\mI+\mM\vw\vw^\top )^{-1} \vu  &= \vu^\top \vw\vw^\top  (\mM^{-1}+\vw\vw^\top)^{-1} \mM^{-1} \vu\\
    &\leq \lVert \vw\vw^\top  (\mM^{-1}+\vw\vw^\top)^{-1}\mM^{-1} \rVert \lVert \vu \rVert^2\\
    &\leq \lVert \vw\vw^\top \rVert \lVert (\mM^{-1}+\vw\vw^\top)^{-1} \rVert \lVert \mM^{-1} \rVert  \lVert \vu \rVert^2\\
    &\leq \lVert \vw \rVert^2  \lVert \mM \rVert \lVert \mM^{-1} \rVert\lVert \vu\rVert^2.
\end{align*}
Hence, we have our conclusion.
\end{proof}

The following lemma makes us obtain $M$ value in Lemma~\ref{lemma:minimizer_independent} and Lemma~\ref{lemma:minimizer_dependent} when we prove the sample complexity results.

\begin{lemma} \label{lemma:ineq4}
Let $\mM \in \mathbb{R}^{k\times k}$ be a positive definite matrix. Then, we have
\begin{equation*}
    \mathbb{E}_{\rvz \sim N(\vzero,\mM)} \left[ e^{\lVert \rvz \rVert}\right] \leq e^{4 \lVert \mM\rVert} + 2^{k/2}.
\end{equation*}
\end{lemma}
\begin{proof}[Proof of Lemma~\ref{lemma:ineq4}]
We have
\begin{align*}
\mathbb{E}_{\rvz \sim N(\vzero,\mM)}\left[ e^{\lVert \rvz \rVert}\right] &= \mathbb{E}_{\rvz \sim N(\vzero,\mM)}\left[ e^{\lVert \rvz \rVert} \vone_{\lVert \rvz \rVert \leq 4\lVert \mM \rVert}\right] +  \mathbb{E}_{\rvz \sim N(\vzero,\mM)}\left[ e^{\lVert \rvz \rVert} \vone_{\lVert \rvz \rVert > 4 \lVert \mM \rVert}\right]\\
&\leq e^{4 \lVert \mM \rVert} + \mathbb{E}_{Z \sim N(\vzero,\mM)}\left[ e^{\frac{1}{4} \lVert \mM \rVert^{-1} \lVert \rvz \rVert^2}\right],
\end{align*}
and
\begin{align*}
\mathbb{E}_{\rvz \sim N(\vzero,\mM)}\left[ e^{\frac{1}{4} \lVert \mM\rVert^{-1} \lVert \rvz \rVert^2}\right] &= \int_{\mathbb{R}^k}  (2 \pi)^{-k/2} \det (\mM)^{-1/2} \exp\left(-\frac{1}{2} \vx^\top \left(\mM^{-1}-\frac{1}{2}\lVert \mM\rVert^{-1} \mI\right) \vx\right)d\vx\\
&= \det\left(\mM^{-1}-\frac{1}{2}\lVert \mM \rVert^{-1}\mI\right)^{-1/2}\det (\mM)^{-1/2} = \det\left(\mI - \frac{1}{2} \lVert M \rVert^{-1} \right)^{-1/2} \\
& \leq \left \lVert \left(\mI - \frac{1}{2}\lVert \mM\rVert^{-1}\mM \right)^{-1} \right \rVert^{k/2} \leq 2^{k/2} .
\end{align*}
Hence, we have our conclusion.
\end{proof}

Lastly, we introduce the lemma used in showing uniform convergence of $\mathbb{E}_\kappa[\lossmask(\vw)]$ to $\mathbb{E}_\infty[\lossmask(\vw)]$ as $\kappa \rightarrow \infty$.
\begin{lemma}\label{lemma:ineq5}
For each $m \in \mathbb{R}$ and $\sigma >0$, 
\begin{equation*}
0 \leq \mathbb{E}_{X \sim N(m, \sigma^2)}[l(X)] - l(m) \leq \sigma.
\end{equation*}
\end{lemma}
\begin{proof}[Proof of Lemma~\ref{lemma:ineq5}]
Since $l(\cdot)$ is convex, by Jensen's inequality, we have
\begin{equation*}
    \mathbb{E}_{X \sim N(m, \sigma^2)}[l(X)] \geq l(\mathbb{E}_{X \sim N(m, \sigma^2)}[X])  = l(m).
\end{equation*}
Also,  we have
\begin{equation*}
    l(m) - \mathbb{E}_{X \sim N(m, \sigma^2)}[l(X)] = \mathbb{E}_{X \sim N(m, \sigma^2)}[l(m)-l(X)] \geq \mathbb{E}_{X \sim N(m, \sigma^2)}[l'(X) (m-X)] \geq -\mathbb{E}[|X-m|],
\end{equation*}
where the last inequality used $|l'(z)|\leq 1$ for all $z \in \mathbb{R}$.
By Cauchy-Schwartz inequality, $\mathbb{E}_{X \sim N(m, \sigma^2)}[|X-m|] \leq \mathbb{E}_{X \sim N(m, \sigma^2)}[(X-m)^2]^{1/2} = \sigma$. Thus, we have our conclusion.
\end{proof}

\subsection{Concentration Bounds}
We introduce concentration bounds for i.i.d. random variables which we use in the proof of Lemma~\ref{lemma:minimizer_independent}.
\begin{lemma}\label{lemma:independent_concentration}
Let $X_1, \dots, X_N \overset{\mathrm{i.i.d.}}{\sim} \mathcal{P}$ where $\mathcal{P}$ is a probability distribution on $\mathbb{R}$. Suppose $\mathbb{E}_{X \sim \mathcal{P}}\left[e^{|X - \mathbb{E}_{X \sim \mathcal{P}}[X]|}\right]\leq M$ for some constant $M>0$. Then, for any $0<\epsilon<1$, 
\begin{equation*}
    \mathbb{P} \left [ \frac{1}{N} \sum_{i=1}^N X_i - \mathbb{E}_{X\sim \mathcal{P}}[X] > \epsilon \right] 
    \leq \exp\left ( -\frac{C \epsilon^2 N}{M} \right ),
\end{equation*}
and
\begin{equation*}
    \mathbb{P} \left [ \frac{1}{N} \sum_{i=1}^N X_i - \mathbb{E}_{X\sim \mathcal{P}}[X] < -\epsilon \right] 
    \leq \exp\left ( -\frac{C \epsilon^2 N}{M} \right ),
\end{equation*}
where $C >0$ is a universal constant.
\end{lemma}
\begin{proof}[Proof of Lemma~\ref{lemma:independent_concentration}]
From our definition of $M$, we have $M \geq 1$. Choose $t = \frac{\epsilon}{16M}$, then we have $0 < t <\frac{1}{2}$ since $0 < \epsilon < 1$ and $M \geq 1$. From Chernoff bound, we have
\begin{equation*}
\mathbb{P}\left[ \frac{1}{N}\sum_{i=1}^N X_i - \mathbb{E}_{X \sim \mathcal{P}}[X] > \epsilon \right] \leq \mathbb{E}\left[e^{t\left(\sum_{i=1}^N X_i - \mathbb{E}_{X\sim \mathcal{P}}[X] \right)}\right]e^{-t N \epsilon} ,
\end{equation*}
and
\begin{equation*}
\mathbb{P}\left[ \frac{1}{N}\sum_{i=1}^N X_i - \mathbb{E}_{X \sim \mathcal{P}}[X] < -\epsilon \right] \leq \mathbb{E}\left[e^{t\left(\sum_{i=1}^N \mathbb{E}_{X\sim \mathcal{P}}[X] - X_i  \right)}\right]e^{-t N \epsilon}.
\end{equation*}
Since $X_1, \dots, X_N \overset{\mathrm{i.i.d.}}{\sim} \mathcal{P}$, we have
\begin{equation*}
\mathbb{E}\left[e^{t\left(\sum_{i=1}^N X_i - \mathbb{E}_{X\sim \mathcal{P}}[X] \right)}\right]e^{-t N \epsilon} = \mathbb{E}\left[e^{t\sum_{i=1}^N\left( X_i - \mathbb{E}_{X\sim \mathcal{P}}[X] -\epsilon \right)}\right] = \mathbb{E}_{X \sim \mathcal{P}}\left[e^{t\left( X - \mathbb{E}_{X\sim \mathcal{P}}[X] -\epsilon \right)}\right]^N,
\end{equation*}
and
\begin{equation*}
\mathbb{E}\left[e^{t\left(\sum_{i=1}^N \mathbb{E}_{X\sim \mathcal{P}}[X] - X_i  \right)}\right]e^{-t N \epsilon} = \mathbb{E}\left[e^{t\sum_{i=1}^N\left(\mathbb{E}_{X\sim \mathcal{P}}[X] - X_i -\epsilon \right)}\right] = \mathbb{E}_{X \sim \mathcal{P}}\left[e^{t\left( \mathbb{E}_{X\sim \mathcal{P}}[X]-X -\epsilon \right)}\right]^N.
\end{equation*}
For each $x \in \mathbb{R}$, $e^x \leq 1+x + \frac{1}{2}x^2 e^{|x|}$. Thus,
\begin{align*}
 &\quad \mathbb{E}_{X \sim \mathcal{P}}\left[e^{t\left( X - \mathbb{E}_{X\sim \mathcal{P}}[X] -\epsilon \right)}\right]\\
 &\leq \mathbb{E}_{X \sim \mathcal{P}} \left[ 1 + t\left(X - \mathbb{E}_{X \sim \mathcal{P}}[X]-\epsilon \right) + \frac{t^2}{2} \left(X - \mathbb{E}_{X \sim \mathcal{P}}[X]-\epsilon \right)^2 e^{\left|t\left(X - \mathbb{E}_{X \sim \mathcal{P}}[X]-\epsilon \right) \right|} \right]\\
 &=1 - \epsilon t + \frac{t^2}{2} \mathbb{E}_{X \sim \mathcal{P}} \left [\left(X - \mathbb{E}_{X \sim \mathcal{P}}[X]-\epsilon \right)^2 e^{\left|t \left(X - \mathbb{E}_{X \sim \mathcal{P}}[X]-\epsilon \right) \right|}  \right],
\end{align*}
and
\begin{align*}
 &\quad \mathbb{E}_{X \sim \mathcal{P}}\left[e^{t\left( \mathbb{E}_{X\sim \mathcal{P}}[X]-X -\epsilon \right)}\right]\\
 &\leq \mathbb{E}_{X \sim \mathcal{P}} \left[ 1 + t\left(\mathbb{E}_{X \sim \mathcal{P}}[X]-X-\epsilon \right) + \frac{t^2}{2} \left(\mathbb{E}_{X \sim \mathcal{P}}[X]-X-\epsilon \right)^2 e^{\left|t\left(\mathbb{E}_{X \sim \mathcal{P}}[X]- X-\epsilon \right) \right|} \right]\\
 &=1 - \epsilon t + \frac{t^2}{2} \mathbb{E}_{X \sim \mathcal{P}} \left [\left(\mathbb{E}_{X \sim \mathcal{P}}[X]-X-\epsilon \right)^2 e^{\left|t \left( \mathbb{E}_{X \sim \mathcal{P}}[X]-X-\epsilon \right) \right|}  \right].
\end{align*}
Also, since $x^2 \leq 4 e^{ |x|/2}$ for each $x \in \mathbb{R}$, we have
\begin{align*}
&\quad \mathbb{E}_{X \sim \mathcal{P}} \left [\left(X - \mathbb{E}_{X \sim \mathcal{P}}[X]-\epsilon \right)^2 e^{\left|t\left(X - \mathbb{E}_{X \sim \mathcal{P}}[X]-\epsilon \right) \right|}  \right] \\
&\leq 4 \mathbb{E}_{X \sim \mathcal{P}} \left[ e^{\left(t+\frac{1}{2}\right) \left|X - \mathbb{E}_{X \sim \mathcal{P}}[X]-\epsilon \right|}  \right]
\leq 4 \mathbb{E}_{X \sim \mathcal{P}} \left[ e^{ \left|X- \mathbb{E}_{X \sim \mathcal{P}}[X] \right| +\epsilon }  \right] \\
&\leq 16M,
\end{align*}
and
\begin{align*}
&\quad \mathbb{E}_{X \sim \mathcal{P}} \left [\left(\mathbb{E}_{X \sim \mathcal{P}}[X]-X-\epsilon \right)^2 e^{\left|t\left(\mathbb{E}_{X \sim \mathcal{P}}[X]-X-\epsilon \right) \right|}  \right] \\
&\leq 4 \mathbb{E}_{X \sim \mathcal{P}} \left[ e^{\left(t+\frac{1}{2}\right) \left|\mathbb{E}_{X \sim \mathcal{P}}[X]-X-\epsilon \right|}  \right]
\leq 4 \mathbb{E}_{X \sim \mathcal{P}} \left[ e^{ \left|X- \mathbb{E}_{X \sim \mathcal{P}}[X] \right| +\epsilon }  \right] \\
&\leq 16M.
\end{align*}
Therefore, by substituting $t = \frac{\epsilon}{16M}$ we get
\begin{equation*}
\mathbb{E}_{X \sim \mathcal{P}}\left[e^{t\left( X - \mathbb{E}_{X\sim \mathcal{P}}[X] -\epsilon \right)}\right]  \leq 1- \epsilon t + 8Mt^2 =  1- \frac{\epsilon^2}{32 M} \leq \exp \left (-\frac{\epsilon^2}{32M} \right),
\end{equation*}
and
\begin{equation*}
\mathbb{E}_{X \sim \mathcal{P}}\left[e^{t\left( \mathbb{E}_{X\sim \mathcal{P}}[X]-X -\epsilon \right)}\right]  \leq 1- \epsilon t + 8Mt^2 =  1- \frac{\epsilon^2}{32 M} \leq \exp \left (-\frac{\epsilon^2}{32M} \right).
\end{equation*}
We conclude
\begin{equation*}
    \mathbb{P} \left [ \frac{1}{N} \sum_{i=1}^N X_i - \mathbb{E}_{X\sim \mathcal{P}}[X] > \epsilon \right] \leq \exp\left ( -\frac{C \epsilon^2 N}{M} \right ),
\end{equation*}
and
\begin{equation*}
    \mathbb{P} \left [ \frac{1}{N} \sum_{i=1}^N X_i - \mathbb{E}_{X\sim \mathcal{P}}[X] <- \epsilon \right] \leq \exp\left ( -\frac{C \epsilon^2 N}{M} \right ).
\end{equation*}
where $C = \frac{1}{32}$.
\end{proof}

We extend Lemma~\ref{lemma:independent_concentration} to non i.i.d. setting random variables with special types of dependency using the following two technical lemmas.
\begin{lemma}\label{lemma:partition}
There are disjoint sets $P_1, \dots, P_m \subset [N] \times [N]$ such that $[N] \times [N]  = \bigcup_{i=1}^m P_i \cup \{(1,1), \dots, (N,N) \}$ and 
\begin{align*}
    m = 
    \begin{cases}
    2N &\text{if $N$ is odd},\\
    2(N-1) & \text{if $N$ is even},
    \end{cases}
    \quad
    |P_k| = 
    \begin{cases}
    \frac{N-1}{2} & \text{if $N$ is odd},\\
    \frac{N}{2} & \text{if $N$ is even},
    \end{cases}
    \text{ for all } k \in [m].
\end{align*}
That is, $P_1, \dots, P_m$ and $\{(1,1), \dots, (N,N)\}$ together form a partition of $[N] \times [N]$.
Furthermore, for each $k \in [m]$ and for any distinct $(i_1, j_1), (i_2, j_2) \in P_k$, 
$\{i_1\} \cup \{j_2 \}$ and $\{i_2\} \cup \{j_2\}$ are disjoint.
\end{lemma}
\begin{proof}[Proof of Lemma~\ref{lemma:partition}]\mbox{}\\
\textbf{Case 1: $N$ is odd}\\
For each $k \in [N]$, define
\begin{align*}
    P_{2k-1} &= \{(i,j) \mid i+j \equiv k ~(\bmod ~N), i<j\},\\ 
    P_{2k} &= \{(i,j) \mid i+j \equiv k ~(\bmod ~N), i>j \}.
\end{align*}
It can be easily checked that these $P_1, \dots, P_{2N}$ are what we desired.\\
\textbf{Case 2: $N$ is even}\\
For each $k\in [N-1]$, there is unique $i_k\in [N-1]$ such that $2i_k \equiv k ~(\bmod ~(N-1))$. For each $k \in [N-1]$, define
\begin{align*}
    P_{2k-1} &= \{(i,j) \mid i+j \equiv k ~(\bmod ~(N-1)), i<j\} \cup \{(i_k,N)\},\\
    P_{2k} &= \{(i,j) \mid i+j \equiv k ~(\bmod ~(N-1)), i>j \}\cup \{(N,i_k)\},
\end{align*}
and it can be easily checked that these $P_1, \dots, P_{2(N-1)}$ are what we desired. 
\end{proof}
This is a generalized version of Cauchy-Schwartz inequality.
\begin{lemma}\label{lemma:cauchy}
Suppose $X_1, \dots, X_k$ are nonnegative random variables. Then, 
\begin{equation*}
    \mathbb{E}\left [\prod_{i=1}^k X_i \right ] \leq \left( \prod_{i=1}^k \mathbb{E}\left[X_i^k\right]\right)^{\frac{1}{k}}.
\end{equation*}
\end{lemma}
\begin{proof}[Proof of Lemma~\ref{lemma:cauchy}]
We prove this by using induction on $k$. Note that the case $k=1$ is trivial and $k=2$ is Cauchy-Schwartz inequality. Suppose Lemma~\ref{lemma:cauchy} holds for $k=m$. Let $X_1, \dots, X_{m+1}$ be nonnegative random variables. By H\"older inequality, 
\begin{equation*}
    \mathbb{E}[X_1X_2 \cdots X_{m+1}] \leq \mathbb{E}\left[(X_1 \cdots X_m)^{\frac{m+1}{m}}\right]^\frac{m}{m+1}\mathbb{E}\left[X_{m+1}^{m+1}\right]^{\frac{1}{m+1}}.
\end{equation*}
From the induction hypothesis, we have
\begin{equation*}
    \mathbb{E} \left [(X_1 \cdots X_m)^{\frac{m+1}{m}} \right ] \leq \left( \prod_{i=1}^m \mathbb{E}\left [\left (X_i^\frac{m+1}{m}\right)^m \right] \right)^\frac{1}{m}= \left( \prod_{i=1}^m \mathbb{E}\left [X_i^{m+1} \right] \right)^\frac{1}{m}.
\end{equation*}
Therefore, we have 
\begin{equation*}
    \mathbb{E}[X_1X_2 \cdots X_{m+1}] \leq \left( \prod_{i=1}^{m+1} \mathbb{E}\left [X_i^{m+1} \right] \right)^\frac{1}{m+1}.
\end{equation*}
By the principle of mathematical induction, our conclusion holds for all $k \in \mathbb{N}$.
\end{proof}

Using the two lemmas above, we prove the following lemma which is used in the proof of Lemma~\ref{lemma:minimizer_dependent}.
\begin{lemma}\label{lemma:dependent_concentration}
Let $\{X_{i,j}\}_{i,j\in [N]}$ be real-valued random variables satisfy followings.
\begin{itemize}
\item $\mathbb{E}\left[ e^{|X_{i,j} - \mathbb{E}[X_{i,j}]|} \right] \leq M$ for some $M>0$.
\item If $\{ i_1\} \cup \{ j_1\}$ and $\{i_2\} \cup \{j_1\}$ are disjoint for $i_1, i_2, j_1, j_2 \in [n]$, then $X_{i_1,j_1}$ and $X_{i_2,j_2}$ are independent.
\end{itemize}
Then, for any $0<\epsilon<1$,
\begin{equation*}
    \mathbb{P} \left[ \frac{1}{N^2}\sum_{i,j=1}^N (X_{i,j} - \mathbb{E}[X_{i,j}])  >\epsilon \right] 
    \leq \exp\left (- \frac{C'\epsilon^2 N}{M} \right ),
\end{equation*}
and
\begin{equation*}
    \mathbb{P} \left[ \frac{1}{N^2}\sum_{i,j=1}^N (X_{i,j} -  \mathbb{E}[X_{i,j}]) < -\epsilon \right] 
    \leq \exp \left ( -\frac{C'\epsilon^2 N}{M} \right ),
\end{equation*}
where $C'>0$ is a universal constant.
\end{lemma}
\begin{proof}[Proof of Lemma~\ref{lemma:dependent_concentration}]\label{proof:dependent_concentration}
From Chernoff bound, we have
\begin{equation*}
\mathbb{P} \left[\frac{1}{N^2} \sum_{i,j=1}^N (X_{i,j}-\mathbb{E}[X_{i,j}]) > \epsilon \right] \leq  \mathbb{E}\left [e^{t \sum_{i,j=1}^N (X_{i,j} - \mathbb{E}[X_{i,j}] - \epsilon)} \right ],
\end{equation*}
and
\begin{equation*}
\mathbb{P} \left [\frac{1}{N^2} \sum_{i,j=1}^N (X_{i,j}-\mathbb{E}[X_{i,j}]) < -\epsilon \right] \leq  \mathbb{E}\left [e^{t \sum_{i,j=1}^N ( \mathbb{E}[X_{i,j}] - X_{i,j} - \epsilon)} \right ],
\end{equation*}
for any $t>0$.
We would like to get an upper bound on $\mathbb{E} \left [e^{t \sum_{i,j=1}^N (X_{i,j} - \mathbb{E}[X_{i,j}] - \epsilon)} \right ]$ and $\mathbb{E} \left [e^{t \sum_{i,j=1}^N ( \mathbb{E}[X_{i,j}]-X_{i,j} - \epsilon)} \right ]$, but the problem is that $\{X_{i,j}\}_{i,j \in [N]}$ are not independent of one another. We overcome this obstacle by applying the two lemmas above.

For our $N$, consider $m$ and $P_1, \dots, P_m \subset [N] \times [N]$ that we can obtain from Lemma~\ref{lemma:partition}. 
From our definition of $M$, we have $M \geq 1$ and then let $t = \frac{\epsilon}{16(m+1)M}$. Since we have $0 < \epsilon < 1$ and $M\geq 1$, we can check that $0 < (m+1)t < \frac{1}{2}$.

By Lemma~\ref{lemma:cauchy}, we have
\begin{align*}
    &\quad \mathbb{E} \left[ e^{t \sum_{i,j=1}^N (X_{i,j} - \mathbb{E}[X_{i,j}]-\epsilon)} \right] \\
    &= \mathbb{E}\left[e^{t \sum_{i=1}^N (X_{i,i} - \mathbb{E}[X_{i,i}]-\epsilon)}\prod_{k=1}^m e^{t \sum_{(i,j)\in P_k} (X_{i,j} - \mathbb{E}[X_{i,j}]-\epsilon)}  \right]\\
    &\leq \left( \mathbb{E}\left[e^{(m+1)t \sum_{i=1}^N (X_{i,i} - \mathbb{E}[X_{i,i}]-\epsilon)} \right] \prod_{k=1}^m \mathbb{E} \left[e^{(m+1)t \sum_{(i,j)\in P_k} (X_{i,j} - \mathbb{E}[X_{i,j}]-\epsilon))}  \right]\right)^{\frac{1}{m+1}}\\
    &= \left( \prod_{i=1}^N \mathbb{E}\left[e^{(m+1)t(X_{i,i} - \mathbb{E}[X_{i,i}]-\epsilon)} \right] \prod_{k=1}^m \left( \prod_{(i,j) \in P_k} \mathbb{E} \left[e^{(m+1)t (X_{i,j} - \mathbb{E}[X_{i,j}]-\epsilon)}  \right]\right)\right)^{\frac{1}{m+1}}\\
    &= \left(\prod_{i,j=1}^N \mathbb{E}\left[e^{(m+1)t (X_{i,j} - \mathbb{E}[X_{i,j}]-\epsilon)} \right]\right)^{\frac{1}{m+1}},
\end{align*}
and 
\begin{align*}
    &\quad \mathbb{E} \left[ e^{t \sum_{i,j=1}^N ( \mathbb{E}[X_{i,j}]-X_{i,j}-\epsilon)} \right] \\
    &= \mathbb{E}\left[e^{t \sum_{i=1}^N ( \mathbb{E}[X_{i,i}]-X_{i,i}-\epsilon)}\prod_{k=1}^m e^{t \sum_{(i,j)\in P_k} ( \mathbb{E}[X_{i,j}]-X_{i,j}-\epsilon)}  \right]\\
    &\leq \left( \mathbb{E}\left[e^{(m+1)t \sum_{i=1}^N ( \mathbb{E}[X_{i,i}]-X_{i,i}-\epsilon)} \right] \prod_{k=1}^m \mathbb{E} \left[e^{(m+1)t \sum_{(i,j)\in P_k} ( \mathbb{E}[X_{i,j}]-X_{i,j}-\epsilon))}  \right]\right)^{\frac{1}{m+1}}\\
    &= \left( \prod_{i=1}^N \mathbb{E}\left[e^{(m+1)t( \mathbb{E}[X_{i,i}]-X_{i,i}-\epsilon)} \right] \prod_{k=1}^m \left( \prod_{(i,j) \in P_k} \mathbb{E} \left[e^{(m+1)t ( \mathbb{E}[X_{i,j}]-X_{i,j}-\epsilon)}  \right]\right)\right)^{\frac{1}{m+1}}\\
    &= \left(\prod_{i,j=1}^N \mathbb{E}\left[e^{(m+1)t ( \mathbb{E}[X_{i,j}]-X_{i,j}-\epsilon)} \right]\right)^{\frac{1}{m+1}}.
\end{align*}
For each $x \in \mathbb{R}, e^x \leq 1+x+\frac{1}{2}x^2 e^{|x|}$and $x^2 \leq 4 e^{|x|/2}$. Thus, for each $(i,j) \in [N] \times [N]$, we have
\begin{align*}
    &\quad \mathbb{E}\left[e^{(m+1)t (X_{i,j} - \mathbb{E}[X_{i,j}]-\epsilon)} \right]\\
    &\leq 1 -\epsilon (m+1) t+\frac{1}{2}(m+1)^2t^2\mathbb{E}\left[(X_{i,j} - \mathbb{E}[X_{i,j}]-\epsilon)^2 e^{|(m+1)t (X_{i,j} - \mathbb{E}[X_{i,j}]-\epsilon)|}\right]\\
    &\leq 1 -\epsilon(m+1)t+2(m+1)^2t^2\mathbb{E}\left[e^{((m+1)t+\frac{1}{2}) |X_{i,j} - \mathbb{E}[X_{i,j}]-\epsilon|}\right]\\
    &\leq 1-\epsilon (m+1)t + 8M(m+1)^2t^2 = 1-\frac{\epsilon^2}{32M}\\
    &\leq \exp \left(-\frac{ \epsilon^2}{32M} \right).
\end{align*}
and
\begin{align*}
    &\quad \mathbb{E}\left[e^{(m+1)t (\mathbb{E}[X_{i,j}]- X_{i,j}-\epsilon)} \right]\\
    &\leq 1 -\epsilon (m+1) t+\frac{1}{2}(m+1)^2t^2\mathbb{E}\left[(\mathbb{E}[X_{i,j}]- X_{i,j}-\epsilon)^2 e^{|(m+1)t (\mathbb{E}[X_{i,j}]- X_{i,j}-\epsilon)|}\right]\\
    &\leq 1 -\epsilon(m+1)t+2(m+1)^2t^2\mathbb{E}\left[e^{((m+1)t+\frac{1}{2}) |\mathbb{E}[X_{i,j}]- X_{i,j}-\epsilon|}\right]\\
    &\leq 1-\epsilon (m+1)t + 8M(m+1)^2t^2 = 1-\frac{\epsilon^2}{32M} \\
    &\leq \exp \left(-\frac{ \epsilon^2}{32M} \right).
\end{align*}
Thus, we obtain 
\begin{equation*}
\mathbb{E}\left [e^{t \sum_{i,j=1}^N (X_{i,j} - \mathbb{E}[X_{i,j}] - \epsilon)}\right ] \leq \exp \left(- \frac{ \epsilon^2N^2}{32(m+1)M} \right) \leq \exp \left(- \frac{ C'\epsilon^2N}{M} \right),
\end{equation*}
and 
\begin{equation*}
\mathbb{E}\left [e^{t \sum_{i,j=1}^N ( \mathbb{E}[X_{i,j}] -X_{i,j} - \epsilon)}\right ] \leq \exp \left(- \frac{ \epsilon^2N^2}{32(m+1)M} \right) \leq \exp \left(- \frac{ C'\epsilon^2N}{M} \right),
\end{equation*}
where $C'>0$ is a universal constant, and we have our conclusion.
\end{proof}

\section{Detailed Experimental Settings and Additional Results of Section~\ref{exp:gaussian}}\label{additional_reults}

\subsection{Detailed Settings }\label{setting:gaussian}
In Section 6.1, we intentionally selected values for $\vmu$ and $\mSigma$ such that $\vmu$ is not an eigenvector of $\mSigma$. This was done to ensure that $\vmu$ is distinct in a direction from $\mSigma^{-1} \vmu$. Our selected value for $\vmu$ and $\mSigma$ is as follows, but we note that any other general choices would work.
\begin{equation*} \vmu = \begin{pmatrix} -0.1067\\ 0.2572\\ -0.2392\\ 0.4135\\ -0.2179\\ -0.3995\\ -0.1437\\ 0.5950\\ 0.1786\\ -0.2839
\end{pmatrix},
\end{equation*}
\begin{align*}
\mSigma =
\begin{pmatrix} 
0.4481& 0.0904& -0.0128& -0.0245& 0.1082& -0.2444& 0.1817& 0.0881& 0.0308& 0.0450\\ 
0.0904& 0.4727& 0.0578& -0.1620& 0.0481& -0.0629& 0.0509& -0.1300& -0.1013& -0.1706\\
-0.0128& 0.0578& 0.2477& -0.0728& -0.0490& 0.1214& 0.0189& 0.0159& 0.0064& 0.1649\\
-0.0245& -0.1620& -0.0728& 0.4457& 0.0462& -0.1026& 0.1188& -0.0066& -0.0757& 0.1065\\
0.1082& 0.0481& -0.0490& 0.0462& 0.2892& 0.0268& 0.1117& -0.1799& 0.0617& 0.1787\\
-0.2444& -0.0629& 0.1214& -0.1026& 0.0268& 0.4248& -0.0868& 0.0565& 0.0482& 0.2182\\
0.1817& 0.0509& 0.0189& 0.1188& 0.1117& -0.0868& 0.3638& -0.0980& -0.0279& 0.1658\\
0.0881& -0.1300& 0.0159& -0.0066& -0.1799& 0.0565& -0.0980& 0.4999& 0.0010& -0.0318\\
0.0308& -0.1013& 0.0064& -0.0757& 0.0617& 0.0482& -0.0279& 0.0010& 0.1550& 0.1723\\
0.0450& -0.1706& 0.1649& 0.1065& 0.1787& 0.2182& 0.1658& -0.0318& 0.1723& 0.6230
\end{pmatrix}
\end{align*}

\subsection{Addtional Results of Section~\ref{exp:gaussian}}\label{result:gaussian}
We provide additional experimental results of Section~\ref{exp:gaussian}. We follow the same setting described in Section~\ref{exp:gaussian} and Appendix~\ref{setting:gaussian} without fixing initial weights for various choices on the number of samples $n = 50,100,200,500,1000,2000$ and the separability constant $\kappa = 0.1,0.2,0.5,1.0,2.0,5.0$. We plot the average cosine similarity between the Bayes optimal direction and learned weights in Figure~\ref{exp:d=10} and one may check that the experiments align with our theoretical findings.

In addition, we provide results on $d=20$ in Figure~\ref{exp:d=20} in order to demonstrate a dependency of a sample complexity on a data dimension $d$.  For the results with $d=20$, we use additional values on the number of samples $n=5000,10000$ and we choose $\vmu \in \R^{20}$ and $\mSigma \in \R^{20 \times 20}$ as
\begin{equation*}
    \vmu = \left(\vmu_0^\top, \vmu_0^\top\right)^\top,  \mSigma = \begin{pmatrix}
\mSigma_0 & \vzero \\
\vzero &\mSigma_0 
\end{pmatrix}, 
\end{equation*}
where $\vmu_0 \in \R^{10}$ and $\mSigma_0 \in \R^{10 \times 10}$ is choice of $\vmu$ and $\mSigma$ described in Appendix~\ref{setting:gaussian}. This choice makes it easy to compare two cases $d=10$ and $d=20$. Comparison between Figure~\ref{exp:d=10} and Figure~\ref{exp:d=20} shows that sample complexities for finding optimal decision boundary significantly increase in dimension $d$.

\begin{figure}[h]
\centering
    \begin{subfigure}[$d=10$]
    {
    \centering
    \includegraphics[width = \textwidth]{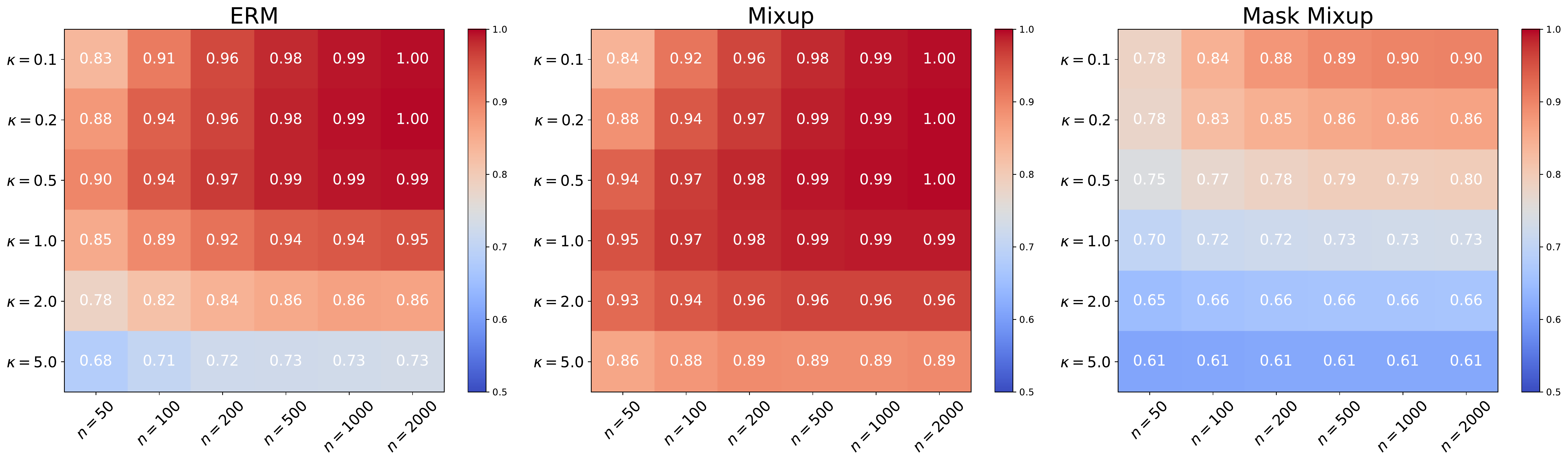}
    \label{exp:d=10}
    }
    \end{subfigure}
    \begin{subfigure}[$d=20$]
    {
    \centering
    \includegraphics[width = \textwidth]{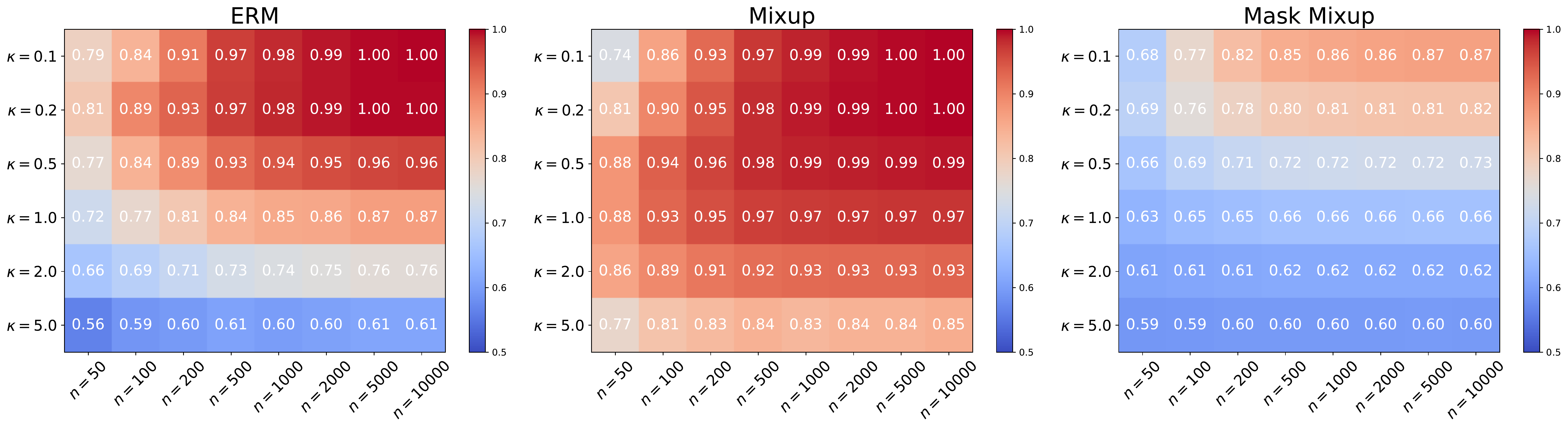}
    \label{exp:d=20}
    }
    \end{subfigure}
    \caption{The average cosine similarity between the Bayes optimal direction and learned weights}
\end{figure}


\end{document}